%% file: main.tex
\documentclass[twoside]{article}

\usepackage[accepted]{utils/aistats2026}
\usepackage[utf8]{inputenc} 
\usepackage[T1]{fontenc}    
\usepackage[colorlinks=true, linkcolor=blue, citecolor=black,urlcolor=black]{hyperref}        
\usepackage{url}            
\usepackage{booktabs}       
\usepackage{amsfonts}       
\usepackage{nicefrac}       
\usepackage{microtype}      
\usepackage[table]{xcolor}  
\usepackage{nccmath}
\usepackage{relsize}
\usepackage{colortbl}

\usepackage{algorithm}
\usepackage{algorithmic}
\usepackage{enumitem}

\usepackage[round]{natbib}

\usepackage{amsmath,amsthm,amssymb,amsfonts}
\usepackage{mathabx,mathtools}
\usepackage{bm}
\usepackage{bbm}
\usepackage{soul}
\usepackage[mathscr]{euscript}

\let\underbrace\LaTeXunderbrace

\usepackage[capitalize,noabbrev]{cleveref}

\usepackage[caption=false]{subfig}
\usepackage{wrapfig}
\usepackage{multirow}

\usepackage[skins,theorems]{tcolorbox} 
\tcbset{highlight math style={enhanced,colframe=red,colback=white,arc=0pt,boxrule=1pt}} 

\usepackage{centernot}

\theoremstyle{plain}
\newtheorem{theorem}{Theorem}[section]

\newtheorem{lemma}[theorem]{Lemma}

\theoremstyle{definition}
\newtheorem{definition}[theorem]{Definition}

\theoremstyle{remark}
\newtheorem{remark}[theorem]{Remark}
\usepackage{thm-restate}

\input{utils/defs}

\begin{document}

%

%

\twocolumn[

\aistatstitle{Best Policy Learning from Trajectory Preference Feedback}

\aistatsauthor{ Akhil Agnihotri \And Rahul Jain \And  Deepak Ramachandran \And Zheng Wen }

\aistatsaddress{ \hspace{2cm} University of Southern California \And \hspace{4.25cm} Google DeepMind \And \hspace{2.25cm} OpenAI } ]

\begin{abstract}
Reinforcement Learning from Human Feedback (RLHF) has emerged as a powerful approach for aligning generative models, but its reliance on learned reward models makes it vulnerable to mis-specification and reward hacking. Preference-based Reinforcement Learning (PbRL) offers a more robust alternative by directly leveraging noisy binary comparisons over trajectories. We study the best policy identification problem in PbRL, motivated by post-training optimization of generative models, for example, during multi-turn interactions. Learning in this setting combines an offline preference dataset—potentially biased or out-of-distribution and collected from a rater of subpar `competence'—with online pure exploration, making systematic online learning essential. To this end, we propose Posterior Sampling for Preference Learning ($\mathsf{PSPL}$), a novel algorithm inspired by Top-Two Thompson Sampling that maintains posteriors over the reward model and dynamics. We provide the first Bayesian simple regret guarantees for PbRL and introduce an efficient approximation that outperforms existing baselines on simulation and image generation benchmarks.
\end{abstract}

\input{introduction}
\input{preliminaries}
\input{algorithm}

\input{analysis}

\input{approximation}
\input{experiments}
\input{conclusion}

\clearpage
\bibliography{references}
\bibliographystyle{plainnat}

\section*{Checklist}



\begin{enumerate}

  \item For all models and algorithms presented, check if you include:
  \begin{enumerate}
    \item A clear description of the mathematical setting, assumptions, algorithm, and/or model. Yes, please see Section \ref{sec:preliminaries_pspl}.
    \item An analysis of the properties and complexity (time, space, sample size) of any algorithm. Yes, please see Section \ref{sec:analysis}.
    \item (Optional) Anonymized source code, with specification of all dependencies, including external libraries. No, code will be released later.
  \end{enumerate}

  \item For any theoretical claim, check if you include:
  \begin{enumerate}
    \item Statements of the full set of assumptions of all theoretical results. Yes, please see Sections \ref{sec:algorithm} and \ref{sec:analysis}, and Appendix \ref{sec:appendix}.
    \item Complete proofs of all theoretical results. Yes, please see Appendix \ref{sec:appendix}.
    \item Clear explanations of any assumptions. Yes, please see Appendix \ref{sec:appendix}.
  \end{enumerate}

  \item For all figures and tables that present empirical results, check if you include:
  \begin{enumerate}
    \item The code, data, and instructions needed to reproduce the main experimental results (either in the supplemental material or as a URL). Yes, code will be released later. Please see Appendix \ref{sec:appendix} for instructions and experimental setup.
    \item All the training details (e.g., data splits, hyperparameters, how they were chosen). Yes, please see Appendix \ref{sec:appendix}.
    \item A clear definition of the specific measure or statistics and error bars (e.g., with respect to the random seed after running experiments multiple times). Yes, please see Section \ref{sec:experiments} and Appendix \ref{sec:appendix}.
    \item A description of the computing infrastructure used. (e.g., type of GPUs, internal cluster, or cloud provider). Yes, please see Appendix \ref{sec:appendix}.
  \end{enumerate}

  \item If you are using existing assets (e.g., code, data, models) or curating/releasing new assets, check if you include:
  \begin{enumerate}
    \item Citations of the creator If your work uses existing assets. Yes, please see Appendix \ref{sec:appendix}.
    \item The license information of the assets, if applicable. Yes, please see Appendix \ref{sec:appendix}.
    \item New assets either in the supplemental material or as a URL, if applicable. Not Applicable.
    \item Information about consent from data providers/curators. Yes, please see Appendix \ref{sec:appendix}.
    \item Discussion of sensible content if applicable, e.g., personally identifiable information or offensive content. Not Applicable.
  \end{enumerate}

  \item If you used crowdsourcing or conducted research with human subjects, check if you include:
  \begin{enumerate}
    \item The full text of instructions given to participants and screenshots. Not Applicable.
    \item Descriptions of potential participant risks, with links to Institutional Review Board (IRB) approvals if applicable. Not Applicable.
    \item The estimated hourly wage paid to participants and the total amount spent on participant compensation. Not Applicable.
  \end{enumerate}

\end{enumerate}

\clearpage
\appendix
\thispagestyle{empty}

\onecolumn
\aistatstitle{Supplementary Material}

\input{appendix}

\end{document}

%% file: utils/defs.tex
\usepackage{bigints}

\newcommand{\norm}[3]{{\|#1\|}_{#2}^{#3}}

\renewcommand{\iota}{\textit{i}}
\renewcommand{\implies}{\; \Rightarrow \;}

\newcommand{\forAll}{{\;\; \forall \;}}

\newcommand{\UD}{\Ucal_{\Dcal_{0}}}

\DeclareMathOperator*{\argmin}{\arg\!\min}
\DeclareMathOperator*{\argmax}{\arg\!\max}

\newcommand{\given}{\, | \,}
\newcommand{\Given}{\; \big| \;}

\renewcommand{\bar}[1]{\overline{#1}} 
\renewcommand{\tilde}[1]{\widetilde{#1}} 
\renewcommand{\hat}[1]{\widehat{#1}} 
\newcommand{\E}[1]{\underset{\begin{subarray}{c} #1 \end{subarray}}{\Ebb}}

\newcommand{\PSPL}{\mathsf{PSPL}}

\newcommand{\SR}{\mathscr{SR}}

\definecolor{usc}{RGB}{153, 27, 30}
\definecolor{jpm}{RGB}{0, 53, 148}
\definecolor{eqcol}{RGB}{206, 240, 216}
\definecolor{backg_blue}{RGB}{225,236,244}
\definecolor{tagtxt_blue}{RGB}{88,115,159}
\definecolor{backg_red}{RGB}{250, 222, 222}
\definecolor{tagtxt_red}{RGB}{204, 71, 71}
\definecolor{darkgreen}{RGB}{21, 145, 1}
\definecolor{darkorange}{RGB}{224, 116, 0}
\definecolor{orange}{RGB}{230, 125, 14}
\definecolor{teal}{RGB}{0, 128, 128}
\definecolor{neon}{RGB}{25, 210, 227}
\definecolor{maroon}{RGB}{143, 4, 4}
\definecolor{purple}{RGB}{119, 26, 240}
\definecolor{Pink}{RGB}{212, 38, 235}

\definecolor{Green}{rgb}{0.13, 0.65, 0.3}
\definecolor{Greenish}{RGB}{49, 158, 140}
\definecolor{TS}{RGB}{0, 125, 140}
\definecolor{LinTS}{RGB}{6, 75, 204}
\definecolor{naiveTS}{RGB}{201, 127, 0}
\definecolor{warmTS}{RGB}{128, 0, 32}
\definecolor{cellg}{RGB}{223, 247, 220}
\definecolor{Pink}{RGB}{212, 38, 235}
\definecolor{peach}{RGB}{250, 187, 177}
\definecolor{brightblue}{RGB}{66, 135, 245}
\definecolor{brightred}{RGB}{255, 102, 102}
\definecolor{violet}{RGB}{153, 153, 255}
\definecolor{dullpink}{RGB}{204, 171, 208}
\definecolor{lightorange}{RGB}{255, 191, 128}
\definecolor{lemon}{RGB}{255, 228, 113}
\definecolor{skyblue}{RGB}{141, 215, 247}
\definecolor{crimson}{rgb}{0.79, 0.0, 0.09}
\definecolor{strange}{RGB}{247, 201, 178}

\newcommand{\Ebb}{{\mathbb E}}

\newcommand{\Ibb}{{\mathbb I}}

\newcommand{\Nbb}{{\mathbb N}}
\newcommand{\Pbb}{{\mathbb P}}

\newcommand{\Rbb}{{\mathbb R}}

\newcommand{\Ibf}{{\mathbf{I}}}

\newcommand{\Acal}{{\mathcal{A}}}

\newcommand{\Dcal}{{\mathcal{D}}}
\newcommand{\Ecal}{{\mathcal{E}}}

\newcommand{\Hcal}{{\mathcal{H}}}

\newcommand{\Lcal}{{\mathcal{L}}}
\newcommand{\Mcal}{{\mathcal{M}}}
\newcommand{\Ncal}{{\mathcal{N}}}
\newcommand{\Ocal}{{\mathcal{O}}}
\newcommand{\Pcal}{{\mathcal{P}}}

\newcommand{\Scal}{{\mathcal{S}}}
\newcommand{\Tcal}{{\mathcal{T}}}
\newcommand{\Ucal}{{\mathcal{U}}}

\makeatletter
\newcommand{\revisionhistory}[1]{%
\@ifundefined{showrevisionhistory}{\relax}{%
{#1}%
}%
}
\makeatother


%% file: introduction.tex
\section{Introduction}
\label{sec:introduction}

\textcolor{black}{RLHF has recently become a cornerstone of aligning large generative models with human intent, enabling advances in natural language processing and other domains. However, the standard RLHF pipeline depends on learning a reward model from human annotations, which introduces two critical limitations: misspecification of the reward function and susceptibility to reward hacking \citep{amodei2016concrete, novoseller2020dueling, tucker2020preference}. These issues stem from the inherent difficulty of reducing complex human objectives to a scalar reward, even when learned from data \citep{sadigh2017active, wirth2017survey}.
}

\textcolor{black}{
PbRL provides an appealing alternative by relying on comparative rather than absolute feedback, thereby offering a more direct and robust signal of human intent \citep{christiano2017deep, saha2023dueling, metcalf2024sample}. In many settings, such preferences are expressed over entire trajectories rather than isolated (final) outcomes, which is especially important in human–robot interactions \citep{wirth2013preference, casper2023open, dai2023safe}, experimentation \citep{novoseller2020dueling}, and recommendation systems \citep{bengs2021preference, kaufmann2023survey}. Similarly, for generative systems, trajectory-level feedback captures the multi-turn nature of interactions more faithfully: when users engage with a large language model (LLM), satisfaction often emerges only after a sequence of exchanges \citep{shani2024multi, zhou2024archer}.
}

\textcolor{black}{
Despite its promise, PbRL in practice often begins with offline preference datasets that are collected off-policy. These datasets are prone to biases and out-of-distribution (OOD) limitations, which can degrade generalization and lead to subpar performance of post-trained models \citep{zhang2024policy, ming2024does}. This raises a fundamental question: given an offline dataset of trajectory comparisons, how should one systematically supplement it with online preference data to maximize learning efficiency under a constrained budget? This question is particularly important for fine-tuning generative models with human feedback, where large-scale offline data are typically available but careful online exploration can provide higher-value information \citep{mao2024context, zhai2024fine}.}

\textcolor{black}{
In this work, we propose a systematic framework for leveraging offline datasets to bootstrap online RL algorithms. We demonstrate that, as expected, incorporating offline data consistently improves online learning performance, as reflected in reduced simple regret. More interestingly, when the agent is further informed about the `competence' of the rater generating the offline feedback—equivalently, the behavioral policy underlying the dataset—the resulting informed agent achieves substantially lower simple regret. Finally, we establish that as the rater competence approaches to that of an expert, higher competence levels yield progressively sharper reductions in simple regret compared to baseline methods.
}

\vspace{-0.15cm}
\textcolor{black}{
\textbf{Relevance to RLHF.} Our problem setting also connects to  LLM alignment, wherein `best policy identification' (BPI) has emerged as a critical objective. Unlike cumulative regret  minimization, which tries to balance exploration and exploitation, BPI ensures that pure exploration leads to the final policy that achieves the highest possible alignment quality \citep{ouyang2022training, chung2024scaling}. Moreover, BPI aligns more naturally with human evaluation processes, which often prioritize assessing final system performance over intermediate learning behaviors. This perspective highlights the importance of optimizing simple regret rather than cumulative regret, which aggregates losses over time \citep{saha2023dueling, ye2024theoretical, zhang2024iterative}. See Figure \ref{fig:intro} for a comparison of $\PSPL$ with Direct Preference Optimization (DPO) \citep{rafailov2024direct} and Identity Preference Optimization (IPO) \citep{ipo}, which shows how on-policy online finetuning helps in BPI.
}

\begin{figure}[ht]
\vspace{-0.5cm}
    \centering
    \subfloat[{\fontfamily{lmss}\selectfont MountainCar}  ]{
        \includegraphics[height=0.2\textwidth, width=0.22\textwidth]{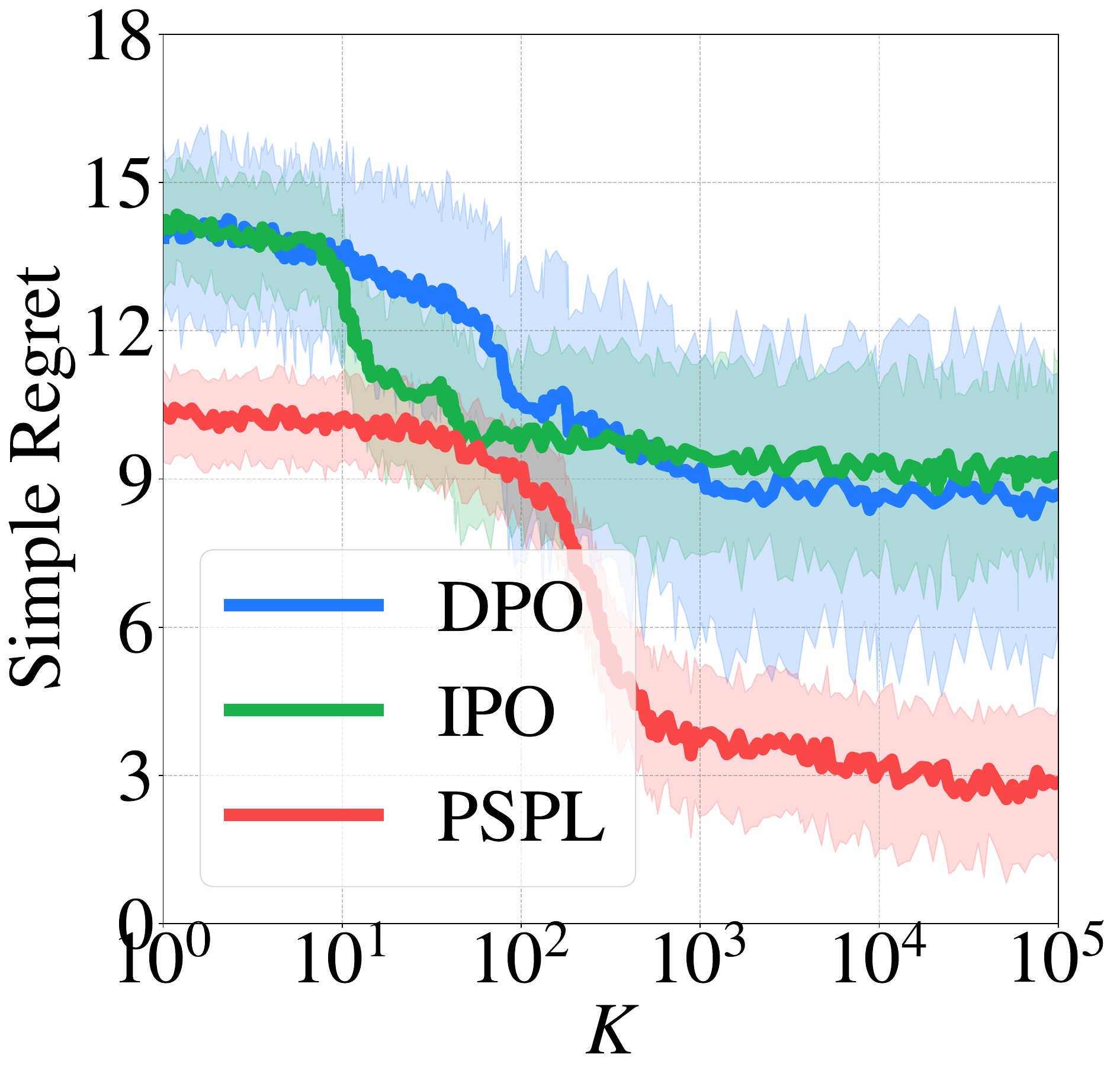}
    } 
    \subfloat[{\fontfamily{lmss}\selectfont RiverSwim}]{
        \includegraphics[height=0.2\textwidth, width=0.22\textwidth]{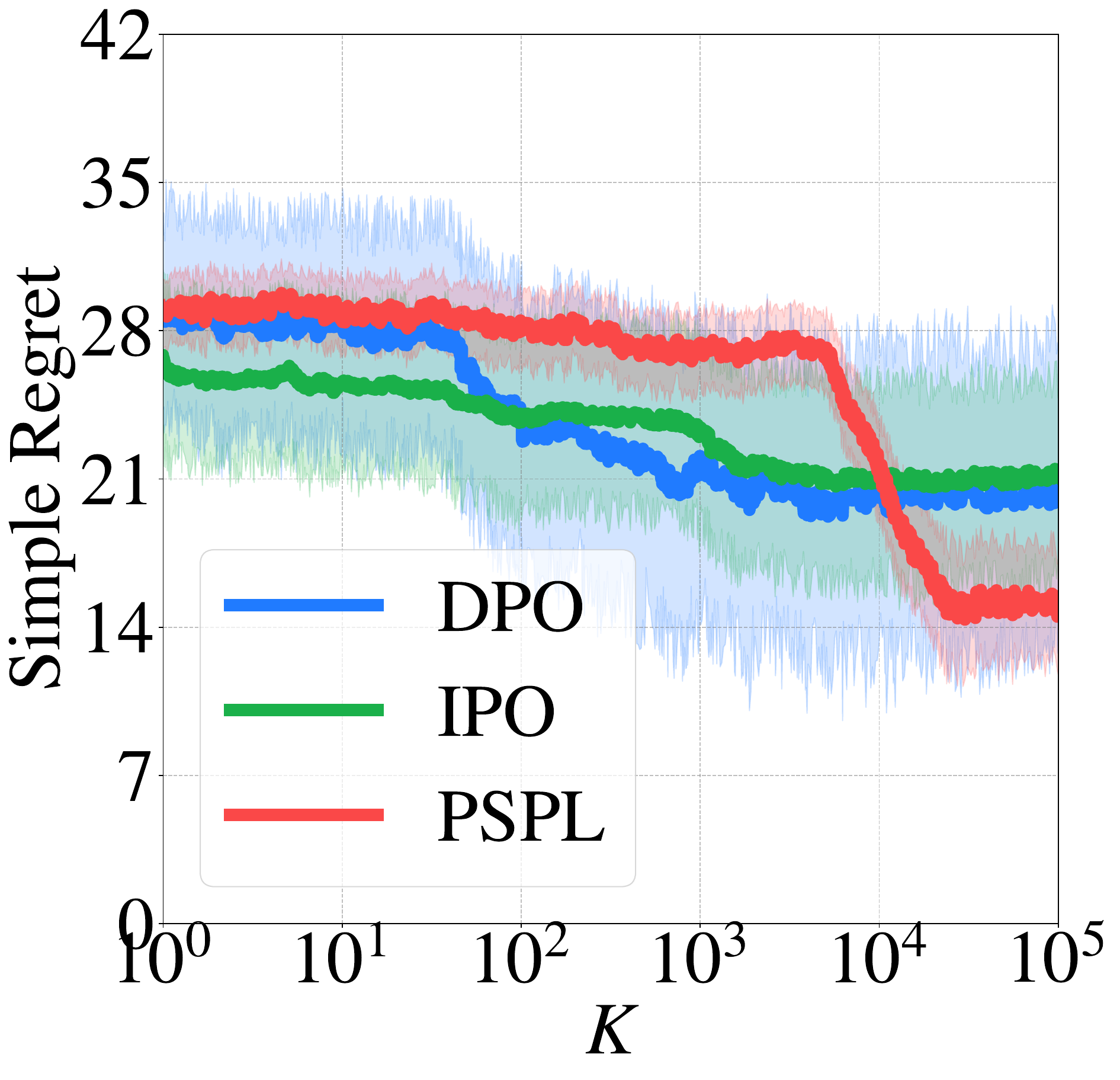}
    } 
\caption{Comparison of $\PSPL$ with current state-of-the-art \emph{offline} finetuning algorithms, DPO and IPO, in two benchmark environments. Online finetuning is necessary for BPI. See Appendix \ref{sec:appendix} for more details.}
\label{fig:intro}
\end{figure}

In this paper, we address the problem of BPI for an unknown episodic MDP where both the transition dynamics and reward functions are unknown. We assume that some offline data is available in the form of preferences over pairs of trajectories. And we can then collect additional data by generating pairs of trajectories and seeking contrastive feedback between the two. At the end, we output the best policy we can learn. 

Our main contributions are: (i) we introduce a formal framework for best policy learning using offline and online trajectory-preference feedback, (ii) we present a top-two Posterior Sampling algorithm $\PSPL$, the first best policy learning algorithm for this setting and present theoretical bounds on its simple regret, and (iii) we introduce a computationally practical version of $\PSPL$ and show that it has excellent empirical performance on benchmark environments as compared to baselines.   

\textbf{Related Work.} There has been a fair amount of research for preference-based bandits (contextual and linear) \citep{singh2002optimizing, bengs2021preference, busa2014survey, saha2021optimal, saha2022versatile, scheid2024optimal, cheung2024leveraging}. However, even contrary to classical RL (cumulative regret or best policy identification) with numerical reward feedback, only a few
works consider incorporating preference feedback in the RL and Bayesian optimization framework \citep{swamy2023inverse, talebi2018variance, zhu2023principled, ye2024theoretical, song2022hybrid, gonzalez2017preferential}. Training RL agents from trajectory-level feedback, available only at the end of each episode, is particularly challenging. \cite{novoseller2020dueling} and \cite{saha2023dueling} analyze finite $K$-episode cumulative regret for this PbRL problem, where two independent trajectories are run, and binary preference feedback is received per episode. While \cite{novoseller2020dueling} assumes a weaker Gaussian process regression model, \cite{saha2023dueling} proposes an optimism-based algorithm that is computationally infeasible due to exponential growth in state space variables. \cite{liu2023efficient} presents a regression-based empirical approach but lacks theoretical guarantees and exploration. \cite{xu2024provably, chen2024order} address offline PbRL with fixed state-action preference datasets, without considering dataset quality. \cite{ye2024theoretical, xiong2023gibbs, li2024policy} take a game-theoretic perspective on offline PbRL in RLHF but assume state-action preferences, avoiding credit assignment from trajectory feedback. Hybrid approaches include \cite{agnihotri2024online}, which studies PbRL in linear bandits, and \cite{hao2023bridging}, which incorporates numerical rewards in the online phase. \textit{To our knowledge, ours is the first work on BPI in unknown MDPs using trajectory-level preference feedback from a subpar expert.}

%% file: preliminaries.tex
\section{Preliminaries}
\label{sec:preliminaries_pspl}

\textcolor{black}{
Consider a $K$-episode, $H$-horizon Markov Decision Process (MDP) setup $\mathcal{M} := (\Pbb_{\eta}, \Scal, \Acal, H, r_{\theta}, \rho)$, where $\Scal$ is a finite state space, $\Acal$ is a finite action space, $\mathbb \Pbb_{\eta}(\cdot \mid s,a)$ are the fixed MDP transition dynamics parameterized by $\eta$ given a state-action pair $(s,a) \in \Scal \times \Acal$, $H \in \Nbb$ is the length of an episode, $r_{\theta}(\cdot)$ is underlying reward model parameterized by $\theta \in \Rbb^{d}$, and $\rho$ denotes the initial distribution over states. We let $S := |\Scal|$ and $A := |\Acal|$ to be the cardinalities of the state and action spaces respectively. In addition, we denote the learner by $\Upsilon$, and let $[N] := \{1, \dots, N\}$ for $N \in \Nbb$.
}

\textcolor{black}{
Under the above MDP setup, a policy $\pi = \{\pi_{h}\}_{h=1}^{H}$ is a sequence of mappings from the state space $\Scal$ to the the probability simplex $\Delta(\Acal)$ over the actions. Specifically, $\pi_h(s, a)$ denotes the probability of selecting action $a$ in state $s$ at step $h$. In addition, $\Pi$ is an arbitrary policy class against which performance of the learner will be measured. Finally, we denote a trajectory by concatenation of all states and actions visited during $H$ steps $\tau:= (s_1, a_1, \cdots , s_{H}, a_H)$. At the start of each episode, we assume $s_1 \sim \rho$. For any given $(\theta, \eta)$ pair of reward-transition parameters, the Bellman update equation for a policy $\pi = \{\pi_h(\cdot ; \theta, \eta) \}_{h=1}^{H}$ and for $(s,a, h) \in \Scal \times \Acal \times [H]$ takes the form:
\begin{equation}
\resizebox{0.82\linewidth}{!}{$
\begin{aligned}
{Q}_{{\theta}, {\eta}, h}^{\pi}(s, a) = r_{{\theta}}(s,a) + \Ebb_{s' \sim \Pbb_{{\eta}}( \cdot \given s,a)} \left[{V}_{{\theta}, {\eta}, {h+1}}^{\pi} (s') \right],
\end{aligned}
$}
\label{eq:approx_bellman}
\end{equation}
with ${V}_{{\theta}, {\eta}, h}^{\pi}(s) = \Ebb_{a \sim \pi_{h}(\cdot | s)} [{Q}_{{\theta}, {\eta}, h}^{\pi}(s, a)]$ for $h \leq H$ and ${V}_{{\theta}, {\eta}, H+1}^{\pi}(s) = 0 \forAll s \in \Scal$.
}

\textbf{Trajectory embedding.} For any trajectory $\tau$ we assume the existence of a trajectory embedding function $\phi : \Gamma  \rightarrow \mathbb{R}^d$, where $\Gamma$ is the set of all possible trajectories of length $H$ and the map $\phi$ is known to the learner. A special case, which we study in this work, is a decomposed embedding, where $\phi(\tau) =  \sum_{h=1}^H  \phi(s_h, a_h)$ and $\phi: \mathcal{S} \times \mathcal{A} \rightarrow \mathbb{R}^d$ is a mapping from state-action pairs to $\mathbb{R}^d$. For all trajectories $\tau \in \Gamma$, we assume that $\| \phi(\tau) \|_{1} \leq B$ for some constant $B > 0$.

\textbf{Preference Modeling.} The learner has access to a rater of arbitrary `competence'.  When presented with two trajectories $\tau_{0}$ and $\tau_{1}$, the rater provides feedback in terms of a Bernoulli random variable $Y \in \{0,1\}$, where $Y=0$  if the rater prefers $\tau_{0}$ to $\tau_{1}$, else $Y=1$ if $\tau_{1}$ is preferred to $\tau_{0}$. Note here that we are working with the setting of sparse feedback which is awarded on a trajectory-level.

We let $\theta \in \Rbb^{d}$ to be an unknown environment reward vector that the rater has \emph{limited} knowledge of. The rater provides preference feedback based on their `competence' and their knowledge of the environment reward vector $\theta$. The `competence' of the rater is characterized by two factors: (i) $\beta \geq 0$ is a measure of the \emph{deliberateness} or \emph{surety} of the rater's decision, and (ii) $\lambda > 0$ controls the degree of \emph{knowledgeability} of the rater about $\theta$ (see Remark \ref{rem:lambda}) . We define $\vartheta \sim$ $N\left(\theta, \mathbf{I}_{d} / \lambda^2\right)$ ($\mathbf{I}_{d}$ is a $d \times d$ identity matrix) as the rater's estimate of the true environment vector $\theta$. Pairwise comparison between two trajectories $\tau_{0}$ and $\tau_{1}$ from the rater is assumed to follow a Bradley-Terry model \citep{bradleyterry1952}, i.e.,

\small
\begin{align}
\begin{split}
\label{eq:pref_logistic_trajectories}
  \Pr ( Y=0 \given \tau_0, \tau_1 \, ; \, \vartheta )  &= \sigma( \beta \langle \phi(\tau_0) - \phi(\tau_1), \vartheta \rangle ), 
 \end{split}
\end{align}
\normalsize
where $\sigma: \Rbb \mapsto [0,1]$ is the logistic link function, i.e., $\sigma(x) = (1+e^{-x})^{-1}$. This naturally leads to the definition of `rater score' of a trajectory $\tau$ as $g_{\beta, \vartheta}(\tau) := \beta \langle \phi(\tau), \vartheta \rangle $, which is of course dependent on the rater's competence.

\begin{remark}
\label{rem:lambda}
    Intuitively, the parameter $\beta \geq 0$ is a measure of the \emph{deliberateness} of the rater's decision: $\beta=0$ means the rater's decisions are uniformly random, whereas as $\beta \to \infty$ means the rater's decisions are greedy with respect to the trajectory scores. Secondly, $\lambda$ is the rater's estimate of the true environment reward model based on its knowledgeability i.e., as $\lambda \to \infty$, $\vartheta \to \theta$. In the context of RLHF alignment, $\lambda$ can be seen as controlling the degree of alignment between a user and the general population from which preferences are aggregated.
\end{remark}

\textbf{Offline Dataset.} There is an initial \emph{offline preference} dataset $\mathcal{D}_0$, which is generated by the rater. This offline dataset of size $N$ is a sequence of tuples of the form 
$
\mathcal{D}_0=\left((\bar{\tau}_n^{(0)}, \bar{\tau}_n^{(1)}, \bar{Y}_n)\right)_{n \in [N]},
$ 
where $\bar{\tau}_n^{(0)}, \bar{\tau}_n^{(1)} \in \Gamma$ are two sampled trajectories, and $\bar{Y}_n \in\{0,1\}$ indicates the rater's preference.

\textbf{Learning Objective.} \textcolor{black}{Consider an MDP $\mathcal{M}$ with unknown transition model $\Pbb_{\eta}(s' \given s,a ) : \Scal \times \Acal \times \Scal \to [0,1]$, parameterized by $\eta$, and an unknown ground-truth reward function $r_{\theta}(s,a) : \Scal \times \Acal \to [0,1]$, parameterized by $\theta$, with $s,s' \in \Scal$ and $a \in \Acal$. One conceivable idea is to assume $r_{\theta}(\tau)$ as the function $r_{\theta}(\tau):= \langle \phi(\tau), \theta \rangle$, where $\theta \in \Rbb^{d}$ is an unknown reward parameter. 
}

The goal of the learner is to identify the optimal policy that maximizes trajectory rewards. Overload notation to denote trajectory rewards by $r_{\theta}(\tau) := \sum_{h=1}^{H} r_{\theta}(s_{h},a_{h})$. Then denote an optimal policy by $\pi^{\star} \in \argmax_{\pi \in \Pi} \E{\tau \sim \pi}[r_{\theta}(\tau)]$. Consider an optimal policy $\pi^{\star}$ and any arbitrary policy $\pi \in \Pi$. Then, with $\tau^{\star} \sim \pi^{\star}$ and $\tau \sim \pi$, the simple Bayesian regret of $\pi$ after $K$ online episodes is defined as:

\vspace{-0.2cm}
\begin{equation}
    \SR_{K}^{\Upsilon}(\pi, \pi^{\star}) := \E{\tau,\tau^{\star}}[ r_{\theta}(\tau^{\star}) - r_{\theta}(\tau) ]
\label{eq:simple_regret_def}
\end{equation}
\vspace{-0.5cm}

In this paper, the objective of the learner $\Upsilon$ then is to design an exploration based online learning algorithm that is \emph{informed} from the offline preference dataset and that which minimizes simple Bayesian regret in Equation \eqref{eq:simple_regret_def}.

\vspace{-0.5cm}
\textcolor{black}{
\begin{remark}
    One motivation for using pure exploration in the online phase of BPI, rather than active learning approaches that target reward model estimation, is its relevance to RLHF. In practice, online active learning with preference data is infeasible \citep{chung2024scaling, achiam2023gpt, anil2023palm}, since preferences are typically collected in fixed offline batches and used to train reward models that then guide online learning (e.g., the AutoRater pipeline \citep{anil2023palm}). Thus, pure exploration better reflects this setting, where offline batched preference data is available and only minimal online interaction is possible to learn the best policy.
\end{remark}
}

%% file: algorithm.tex
\section{The PSPL Algorithm}
\label{sec:algorithm}

For the purpose of pure exploration, the learner $\Upsilon$ is given the opportunity to learn online and generate an \emph{online} dataset of trajectories for which it asks for feedback from the rater. Denote the \emph{online} dataset available to the learner at the beginning of episode $k \in [K]$ as $\mathcal{H}_{k} = \{(\tau_{t}^{(0)}, \tau_{t}^{(1)}, Y_{t})\}_{t=1}^{k-1}$. Then the total available dataset for the learner at the beginning of episode $k$ becomes $\Dcal_{k} = \Dcal_{0} \oplus \Hcal_{k}$, where $\oplus$ denotes concatenation.

Since the parameters $\theta$ and $\eta$ of the reward and transition model are unknown,  
we assume that the prior distribution over the reward vector $\theta$ is a Gaussian
distribution $\nu_{0} \sim \mathcal{N}(\mu_{0}, \Sigma_{0})$ and over the transition model $\eta$ is a Dirichlet distribution $\chi_{0} \sim \mathrm{Dir(\bm{\alpha}_{0}})$ for each state-action pair, where $\bm{\alpha}_{0}$ is a positive real-valued vector of dimension $S$.


Now, note that the offline dataset \emph{informs} the learner $\Upsilon$ of the MDP transition dynamics and the underlying reward model. This is captured in the \emph{informed} prior (before starting the online phase) over the unknown parameters $\eta$ and $\theta$. Denote the probability distributions over the transition parameter and reward parameter by $\chi(\eta)$ and $\nu(\theta)$ respectively. 

\begin{algorithm}[t]
   \caption{Top-two Posterior Sampling for Preference Learning ($\PSPL$)}
\begin{algorithmic}[1]
   \STATE {\bfseries Input:} Initial dataset $\Dcal_{0}$, prior on $\theta$ as $\nu_{0}(\theta)$ and $\eta$ as $\chi_{0}(\eta)$, horizon $H$, episodes $K$.
   \STATE Construct informed prior $\nu_{1}(\theta)$ from Equation \eqref{eq:theta_informed_prior} and $\chi_{1}(\eta)$ from Equation \eqref{eq:eta_informed_prior}.
    \FOR{$k = 1,2, \dots ,K$} 
    \STATE Sample $\hat{\eta}_{k}^{(0)}, \hat{\eta}_{k}^{(1)} \sim \chi_{k}(\eta)$ and $\hat{\theta}_{k}^{(0)}, \hat{\theta}_{k}^{(1)} \sim \nu_{k}(\theta)$.
    \STATE Compute policies $\pi_{k}^{(0)}$ using $(\hat{\eta}_{k}^{(0)}, \hat{\theta}_{k}^{(0)})$ and $\pi_{k}^{(1)}$ using $(\hat{\eta}_{k}^{(1)}, \hat{\theta}_{k}^{(1)})$.
    \STATE Run two trajectories $\tau_{k}^{(0)} \sim \pi_{k}^{(0)}$ and $\tau_{k}^{(1)} \sim \pi_{k}^{(1)}$ for $H$ horizon.
    \STATE Get feedback $Y_{k}$ on $\tau_{k}^{(0)}$ and $\tau_{k}^{(1)}$, and append to dataset as $\Dcal_{k} = \Dcal_{k-1} \oplus \left(\tau_{k}^{(0)}, \tau_{k}^{(1)}, Y_{k} \right)$. 
    \STATE Update posteriors to get $\nu_{k+1}(\theta)$ and $\chi_{k+1}(\eta)$.
    \ENDFOR
    \STATE {\bfseries Output:} \textcolor{black}{Optimal policy $\pi_{K+1}^{\star}$ computed using MAP estimate from $\chi_{K+1}(\eta)$ and $\nu_{K+1}(\theta)$}.
\end{algorithmic}
\label{alg:main_algo_theoretical}
\end{algorithm}

\textbf{Informed prior for \texorpdfstring{$\theta$}{theta}.}
Denoting the \emph{uninformed} prior as $\nu_{0}(\theta)$, we have the \emph{informed} prior $\nu_{1}(\theta)$ as,

\vspace{-0.3cm}
\begin{equation}
\resizebox{0.9\linewidth}{!}{$
\begin{aligned}
     \nu_{1}(\theta) \; &:=  \; \nu(\theta \given \Dcal_{0}) \; \propto \; \Pr(\Dcal_{0} \given \theta) \nu_{0}(\theta) \; \\ 
     & \propto \; \prod_{n=1}^{N} \Pr(\bar{Y}_n \given \bar{\tau}_n^{(0)}, \bar{\tau}_n^{(1)}, \theta) \; \Pr(\bar{\tau}_n^{(0)} \given \theta)   \; \Pr(\bar{\tau}_n^{(1)} \given \theta) \; \nu_{0}(\theta), 
\end{aligned}
$}
\label{eq:theta_informed_prior}
\end{equation}

where the second step follows from Equation \eqref{eq:pref_logistic_trajectories}, and that $\bar{\tau}_n^{(0)}$ and $\bar{\tau}_n^{(1)}$ are assumed independent given $\theta$, \textcolor{black}{as is in the context of RLHF, where outputs (trajectories) are conditionally independent given the prompt}. It is worth emphasizing that the offline dataset carries information about the reward parameter through the terms $\Pr( \cdot \given \theta)$, which incorporates information about the expert’s policy, and thus improves the informativeness of the prior distribution.

\textbf{Informed prior for \texorpdfstring{$\eta$}{eta}.}
Denoting the \emph{uninformed} prior as $\chi_{0}(\eta)$, we have \emph{informed} prior $\chi_{1}(\eta)$ as,

\vspace{-0.3cm}
\begin{equation}
\resizebox{0.77\linewidth}{!}{$
\begin{aligned}
    \chi_{1}(\eta) \; &:=  \; \chi(\eta \given \Dcal_{0}) \;  \propto \; \Pr(\Dcal_{0} \given \eta) \, \chi_{0}(\eta) \; \\ & \propto \; \prod_{n=1}^{N} \prod_{j=0}^{1} \prod_{h=1}^{H-1}  \Pbb_{\eta} \left(\bar{s}_{n,h+1}^{(j)} \given \bar{s}_{n,h}^{(j)}, \bar{a}_{n,h}^{(j)} \right) \; \chi_{0}(\eta),
\end{aligned}
$}
\label{eq:eta_informed_prior}
\end{equation}
\vspace{-0.4cm}

where $\bar{\tau}_{n}^{(j)} := \{\bar{s}_{n,1}^{(j)}, \bar{a}_{n,1}^{(j)}, \dots, \bar{s}_{n,H}^{(j)} \}$ for $j \in \{0,1\}$ is an offline trajectory of length $H$. Note here that the proportional sign hides the dependence on the offline dataset generating policy (since we do not make any assumptions on this behavourial policy) and the fact that $\bar{Y}_{n}$ is conditionally independent of dynamics given $\bar{\tau}_n^{(0)}$, $\bar{\tau}_n^{(1)}$.

After constructing informed priors from the offline preference dataset above, the learner begins the online phase for active data collection using pure exploration. The learner maintains posteriors over the true reward and transition kernels, which inherently permit for exploration. In each episode, using samples from these posteriors, the learner computes two policies using value / policy iteration or linear programming, and rolls out two $H$-horizon trajectories. \textcolor{black}{Posteriors are updated based on the trajectory-level preference feedback, and final policy output is constructed from Maximum-A-Posteriori (MAP) estimate from $\chi_{K+1}(\eta)$ and $\nu_{K+1}(\theta)$, as shown in Algorithm \ref{alg:main_algo_theoretical}}.

\begin{remark}
\label{remark:ps_over_optimism}
Although there exist optimism-based algorithms which construct confidence sets around the reward and transition kernels, it is known that posterior sampling based algorithms provide superior empirical performance \citep{ghavamzadeh2015bayesian, ouyang2017learning, liu2023efficient}. In addition, for our problem setting of best policy identification with noisy trajectory level feedback, posterior sampling is a natural method to incorporate beliefs about the environment.
\end{remark}


%% file: analysis.tex
\section{Theoretical Analysis of PSPL}
\label{sec:analysis}

This section focuses on regret analysis of $\PSPL$. The analysis has two main steps: (i) finding a \emph{prior-dependent} upper regret bound in terms of the sub-optimality of any optimal policy estimate $\hat{\pi}^{\star}$ constructed from $\Dcal_{0}$. This part characterizes the online learning phase by upper bounding simple Bayesian regret in terms of the estimate $\hat{\pi}^{\star}$ constructed by $\PSPL$ before the online phase, and  (ii) describing the procedure to construct this $\hat{\pi}^{\star}$ based on the attributes of $\Dcal_{0}$, such as size $N$ and rater competence $(\lambda, \beta)$. Proofs of results, if not given, are provided in Appendix \ref{sec:appendix}.

\subsection{General prior-dependent regret bound}

It is natural to expect some regret reduction if an offline preference dataset is available to warm-start the online learning. However, the degree of improvement must depend on the `quality' of this dataset, for example through its size $N$ or rater competence $(\lambda, \beta)$. Thus, analysis involves obtaining a prior-dependent regret bound, which we obtain next.

\vspace{-0.4cm}
\textcolor{black}{
\begin{restatable}{lemma}{psplerrorregret}
\label{lemma:pspl_error_regret}
For any confidence $\delta_{1} \in (0,\frac{1}{3})$, let $\delta_{2} \in (c,1)$ with $c \in (0,1)$, be the probability that any optimal policy estimate $\hat{\pi}^{\star}$ constructed from the offline preference dataset $\Dcal_{0}$ is $\varepsilon$-optimal with probability at least $(1-\delta_{2})$ i.e., $\Pr \left( \E{s \sim \rho} \left[ V_{\theta,\eta,0}^{\pi^{\star}}(s) -  V_{\theta,\eta,0}^{\hat{\pi}^{\star}}(s) \right] > \varepsilon \right) < \delta_{2}$. Then, the simple Bayesian regret of the learner $\Upsilon$ is upper bounded with probability at least $1-3\delta_{1}$ by, 
\begin{equation}
    \label{eq:pspl_prior_error_bound}
    \resizebox{0.8\linewidth}{!}{$
    \begin{aligned}
        \SR_{K}^{\Upsilon}(\pi^{\star}_{K+1}, \pi^{\star}) \leq \sqrt{ \frac{10 \delta_{2}  S^{2}AH^{3} \ln \left( \frac{2KSA}{\delta_{1}}  \right) + 3SAH^{2}\varepsilon^{2}}{2K \left(1 + \ln \frac{SAH}{\delta_{1}} \right) - \ln \frac{SAH}{\delta_{1}}} }
    \end{aligned}
    $}
    \end{equation}
\end{restatable}
}
\vspace{-0.2cm}

Please see Appendix \ref{proof:pspl_error_regret} for proof. This lemma tells us that if the learner can construct an $\varepsilon$-optimal policy estimate $\hat{\pi}^{\star}$ with a probability of $(1-\delta_{2})$, then the learner's simple regret decreases as $\delta_{2}$ decreases. Next, we describe how to incorporate information from $\Dcal_{0}$ before the online phase.

\subsection{Incorporating offline preferences for regret analysis}

Before we begin to construct $\hat{\pi}^{\star}$, we define the state visitation probability $p_{h}^{\pi}(s)$ for any given policy $\pi$. Construction of $\hat{\pi}^{\star}$ will revolve around classification of states in a planning step based on $p_{h}^{\pi}(s)$, which will lead to a reward free exploration strategy to enable learning the optimal policy.


\begin{definition}[\textcolor{black}{State Visitation Probability}]
\label{def:visitation}
    Given $(h,s)\in [H]\times \mathcal{S}$, the state (occupancy measure) and state-action visitation probabilities of a policy $\pi$ is defined for all $h' \in [h-1]$ as follows:
    \begin{equation}
    \resizebox{0.85\linewidth}{!}{$
    \begin{aligned}
        p_{h}^{\pi}(s) &= \Pr \left( s_h = s \given s_1 \sim \rho, a_{h'} \sim \pi(s_{h'})\right) \\
        p_{h}^{\pi}(s,a) &= \Pr \left( s_h = s, a_h = a \given s_1 \sim \rho, a_{h'} \sim \pi(s_{h'}) \right) \, .
    \end{aligned}
    $}
    \end{equation}
\end{definition}

Let $p_{\min} = \min_{h,s} \max_{\pi}  p_{h}^{\pi}(s)$, and we assume it is positive. Further, define the infimum probability of any reachable state under $\pi^{\star}$ as $p_{\min}^{\star} := \min_{h,s} p_{h}^{\pi^{\star}}(s)$ and we assume it positive as well.



\textcolor{black}{
We now describe a procedure to construct an estimator $\hat{\pi}^{\star}$ of the optimal policy  from $\mathcal{D}_0$ such that it is $\varepsilon$-optimal with probability $\delta_{2}$. For each ($\theta, \eta$), define a deterministic Markov policy $\pi^{{\star}}(\theta, \eta) = \{\pi_h^{{\star}}(\cdot ; \theta, \eta) \}_{h=1}^{H}$, 
\begin{equation}
    \resizebox{0.85\linewidth}{!}{$
    \begin{aligned}
    \pi_h^{{\star}}(s ; \theta, \eta) =    \begin{cases}
        \arg\max_a Q_{\theta, \eta, h}(s, a), & \text{if } p_{h}^{\pi^{\star}}(s) > 0 \\
        \mathring{a} \sim \mathrm{Unif}(\Acal), &\text{if } p_{h}^{\pi^{\star}}(s) = 0,
    \end{cases}
    \end{aligned}
    $}
\end{equation}
}
\textcolor{black}{where $Q_{\theta, \eta, h}(\cdot, \cdot)$ is the Q-value function in a MDP with reward-transition parameters $(\theta, \eta)$, and the tiebreaker for the argmax operation is based on any fixed order on actions}. It is clear from construction that $\{\pi_{h}^{\star}(\cdot ; \theta, \eta)\}_{h=1}^{H}$ is an optimal policy for the MDP with parameters ($\theta, \eta$). Furthermore, for those states that are impossible to be visited, we choose to take an action uniformly sampled from $\Acal$. 

\textbf{Construction of $\hat{\pi}^{\star}$.} To attribute preference of one trajectory to another, we build `winning' ($U_{h}^{W}$), and `undecided' ($U_{h}^{U}$) action (sub)sets for each state in the state space for each time step $h \in [H]$. To achieve this, we define the net count of each state-action pair from the offline dataset $\Dcal_{0}$. Recall that any offline trajectory $\bar{\tau}_{n}^{(\cdot)}$ is composed of $(\bar{s}_{n,1}^{(\cdot)}, \bar{a}_{n,1}^{(\cdot)}, \dots, \bar{s}_{n,H}^{(\cdot)})$. Then if the `winning' counts are defined as $w_{h}(s,a) = \sum_{n=1}^{N} \Ibf \{\bar{s}_{n,h}^{(\bar{Y_{n}})} = s , \bar{a}_{n,h}^{(\bar{Y_{n}})} = a \}$, and `losing' counts are defined as $l_{h}(s,a) = \sum_{n=1}^{N} \Ibf \{\bar{s}_{n,h}^{(1-\bar{Y_{n}})} = s , \bar{a}_{n,h}^{(1-\bar{Y_{n}})} = a \}$, then the `net' counts are defined as $c_{h}(s,a) = w_{h}(s,a) - l_{h}(s,a)$. The action sets are then constructed as,
$$
U_{h}^{W}(s) = \{a ; c_{h}(s,a) > 0 \} \,,  U_{h}^{U}(s) = \{a ; c_{h}(s,a) \leq 0 \}
$$

By construction, if any state-action pair does not appear in $\Dcal_{0}$, it is attributed to the undecided set $U_{h}^{U}$. Finally, for some fixed $\delta \in (0,1)$, construct the policy estimate $\hat{\pi}^{\star} := \{\hat{\pi}_{h}^{\star}\}_{h=1}^{H}$ as 

\vspace{-0.3cm}
\begin{equation}
\resizebox{0.92\linewidth}{!}{$
\begin{aligned}
    \hat{\pi}_{h}^{\star}(s) := 
    \begin{cases}
        \argmax_{a \in U_{h}^{W}(s)} c_{h}(s,a) & \text{if } \sum_{a'} c_{h}(s,a') \geq \delta N \\
        \mathring{a} \sim \mathrm{Unif}(U_{h}^{U}) & \text{if } \sum_{a'} c_{h}(s,a') < \delta N \\
    \end{cases} \nonumber
\end{aligned}
$}
\end{equation}
\vspace{-0.3cm}

To ensure that for any state-time pair $(s,h)$ with $p_{h}^{\pi^{\star}}(s) > 0$, the optimal action $\pi_{h}^{\star}(s ; \cdot) \in U_{h}^{W}(s)$ with high probability, we first need an upper bound on the `error' probability of the rater w.r.t any trajectory pair in the offline dataset $\Dcal_{0}$. Here, the `error' probability refers to the probability of the event in which the rater prefers the sub-optimal trajectory (w.r.t. the trajectory score $g(\cdot)$). Hence, for $n \in [N]$, define an event $\Ecal_{n} := \left\{\bar{Y}_{n} \neq \argmax_{i \in \{0,1\}} g_{\beta, \vartheta}\left (\bar{\tau}_{n}^{(i)} \right) \right\}$, i.e., at the $n$-th index of the offline preference dataset, the rater preferred the suboptimal trajectory. Then,

\begin{lemma}
\label{lemma:rater_error_upper_bound}
Given a rater with competence $(\lambda, \beta)$ such that $\beta > \frac{2\ln \left(2 d^{1/2} \right)}{\left| B\lambda^{2} - 2 \Delta_{\min}  \right|}$ with $\Delta_{\min} = \min_{n \in [N]} \left| r_{\theta}(\bar{\tau}_{n}^{(0)}) - r_{\theta}(\bar{\tau}_{n}^{(1)}) \right|$, and an offline preference dataset $\Dcal_{0}$ of size $N>2$, we have
\begin{equation}
\resizebox{0.98\linewidth}{!}{$
\begin{aligned}
    \Pr (\Ecal_{n} \given \beta, \vartheta) \leq \exp\left(- \beta B \sqrt{2\ln(2d^{1/2}N)}/\lambda - \beta \Delta_{\min} \right)  + \frac{1}{N} \; := \; \gamma_{\beta, \lambda, N} \nonumber
\end{aligned}
$}
\end{equation}
\end{lemma}
\vspace{-0.3cm}

See Appendix \ref{proof:pspl_final_regret} for proof. This lemma establishes the relationship between the competence of a rater (in terms of $(\lambda, \beta)$) and the probability with which it is sub-optimal in its preference. Observe that as the rater tends to an expert i.e., $\lambda, \beta \to \infty$ and we have increasing offline dataset size $N$, we get $\Pr (\Ecal_{n} \given \beta, \vartheta) \to 0$ i.e., the rater \emph{always} prefers the trajectory with higher rewards in an episode, which is intuitive to understand. Using this Lemma \ref{lemma:rater_error_upper_bound}, we can upper bound the simple Bayesian regret of $\PSPL$ as below.

\begin{restatable}{theorem}{psplfinalregretbound}
\label{lemma:pspl_final_simple_regret}
For any confidence $\delta_{1} \in (0,\frac{1}{3})$ and offline preference dataset size $N>2$, the simple Bayesian regret of the learner $\Upsilon$ is upper bounded with probability of at least $1-3\delta_{1}$ by, 
    \begin{equation}
    \resizebox{\linewidth}{!}{$
    \begin{aligned}
        &\SR_{K}^{\Upsilon}(\pi^{\star}_{K+1}, \pi^{\star}) \leq \sqrt{ \frac{20 \delta_{2}  S^{2}AH^{3} \ln \left( \frac{2KSA}{\delta_{1}}  \right)}{2K \left(1 + \ln \frac{SAH}{\delta_{1}} \right) - \ln \frac{SAH}{\delta_{1}}} } \;\; \\ ,\text{with} & \; \delta_{2} = 2 \exp \left( - N \left( 1 + \gamma_{\beta, \lambda, N} \right)^{2} \right) + \exp \left( - \frac{N}{4} (1 - \gamma_{\beta, \lambda, N})^{3} \right). \nonumber
    \end{aligned}
    $}
    \end{equation}
\end{restatable}

See Appendix \ref{proof:pspl_final_regret} for the proof. For a fixed $N>2$, and large $S$ and $A$, the simple regret bound is $\widetilde{\Ocal}\left( \sqrt{S^{2}AH^{3}K^{-1}} \right)$. Note that this bound converges to zero exponentially fast as $N \to \infty$ and as the rater tends to an expert (large $\beta, \lambda$). In addition, as the number of episodes $K$ gets large, $\PSPL$ is able to identify the best policy with probability at least $(1-3\delta_{1})$.

\vspace{-0.3cm}
\textcolor{black}{
\begin{remark}
(i) This paper is different from existing hybrid RL works which usually consider some notion of `coverability' in the offline dataset $\Dcal_{0}$ \citep{wagenmaker2022instance, song2022hybrid}. We instead deal directly with the estimation of the optimal policy, which depends on the visits to each state-action pair in $\Dcal_{0}$. Rather than dealing with concentratability coefficients, we deal with counts of visits, which provides a simple yet effective way to incorporate coverability of the offline phase into the online phase. If the coverability (counts of visits) is low in $\Dcal_{0}$, the failure event in Lemma \ref{lemma:rater_error_upper_bound} has a higher upper bound, which directly impacts the simple regret in Theorem \ref{lemma:pspl_final_simple_regret} negatively. (ii) \textcolor{black}{Since bandit models are a special case of our setting, results specific for them are presented in Appendix \ref{sec:appendixbandits}.}
\end{remark}
}

%% file: approximation.tex
\section{A Practical Approximation}
\label{sec:practical_approx_pspl}

The $\PSPL$ algorithm introduced above assumes that posterior updates can be solved in closed-form. In practice, the posterior update in equations \eqref{eq:theta_informed_prior} and \eqref{eq:eta_informed_prior} is challenging due to loss of conjugacy. Hence, we propose a novel approach using Bayesian bootstrapping to obtain approximate posterior samples. The idea is to perturb the loss function for the maximum a posteriori (MAP) estimate and treat the resulting point estimate as a proxy for an exact posterior sample.

We start with the MAP estimate problem for $(\theta, \vartheta, \eta)$ given the offline and online dataset $\Dcal_{k}$ at the beginning of episode time $k$. We show that this is equivalent to minimizing a particular surrogate loss function described below. For any parity of trajectory $j \in \{0,1\}$ during an episode $k$, we denote the estimated transition from state-action $s^{(j)}_{k,h}, a^{(j)}_{k,h}$ to state $s^{(j)}_{k,h+1}$ by $\Pr_{\eta} \left({s}_{t,h+1}^{(j)} \given {s}_{t,h}^{(j)}, {a}_{t,h}^{(j)} \right)$, where $\Pr_{\eta}(\cdot)$ is updated based on counts of visits. Note here that as transition dynamics are independent of trajectory parity, we update these counts based on all trajectories in $\Dcal_{k}$. See Appendix \ref{proof:mapestimatelemma} for proof. 

\begin{restatable}{lemma}{mapestimatelemma}
\label{th:mapestimatelemma} 
At episode $k$, the MAP estimate of $(\theta, \vartheta, \eta)$ can be constructed by solving the following equivalent optimization problem: 
\begin{equation}
\resizebox{\linewidth}{!}{$
\begin{aligned}
(\theta_{opt}, \vartheta_{opt}, \eta_{opt})  &= \underset{\theta, \vartheta, \eta}{\argmax} \; \Pr(\theta, \vartheta, \eta \, | \, \Dcal_{k}) \\ & \equiv \underset{\theta, \vartheta, \eta}{\argmin} \; \Lcal_{1}(\theta, \vartheta, \eta) +  \Lcal_{2}(\theta, \vartheta, \eta) +  \Lcal_{3}(\theta, \vartheta, \eta) \; , \, \text{where},  \\
 \Lcal_{1}(\theta, \vartheta, \eta)  &:= - \sum_{t=1}^{k-1} \vphantom{\int_1^2} \left[ \beta \langle {\tau}_t^{(Y_{t})} , \vartheta \rangle - \ln \bigg(e^{ \beta \langle {\tau}_t^{(0)}, \vartheta \rangle} + e^{\beta \langle {\tau}_t^{(1)}, \vartheta \rangle} \bigg) \right.  \\ & \qquad + \left. \sum_{j=0}^{1} \sum_{h=1}^{H-1} \ln \Pr_{\eta}\left({s}_{t,h+1}^{(j)} \given {s}_{t,h}^{(j)}, {a}_{t,h}^{(j)} \right) \vphantom{\int_1^2} \right], \\
 \Lcal_{2}(\theta, \vartheta, \eta) &:= - \sum_{n=1}^{N} \left[ \beta \langle \bar{\tau}_n^{(\bar{Y}_{n})} , \vartheta \rangle - \ln \bigg(e^{ \beta \langle \bar{\tau}_n^{(0)}, \vartheta \rangle} + e^{\beta \langle \bar{\tau}_n^{(1)}, \vartheta \rangle} \bigg) \right],  \\
 \Lcal_{3}(\theta, \vartheta, \eta) := \frac{\lambda^2}{2} & \norm{\theta - \vartheta}{2}{2} - SA \sum_{i=1}^{S} (\bm{\alpha}_{0,i} - 1) \ln \eta_{i} + \frac{1}{2} (\theta - \mu_{0})^{T} \Sigma_{0}^{-1} (\theta - \mu_{0}). \nonumber
\end{aligned}
$}
\label{eq:mapestimateproblem}
\end{equation}
\end{restatable}
\vspace{-0.3cm}

\textbf{Perturbation of MAP estimate.} As mentioned above, the idea now is to \emph{perturb} the loss function in Equation \eqref{eq:mapestimateproblem} with some noise, so that the MAP point estimates we get from this perturbed surrogate loss function serve as \emph{samples} from a distribution that approximates the true posterior \citep{osband2019deep, NIPS2017_49ad23d1, qin2022analysis, dwaracherla2022ensembles}. To that end, we use a perturbation of the `online' loss function $\Lcal_1(\cdot)$ and of the `offline' loss function $\Lcal_2(\cdot)$ by multiplicative random weights, and of the `prior' loss function $\Lcal_3(\cdot)$ by random samples from the prior distribution \footnote{$\mathrm{Bern}(\cdot)$ parameters that maximize performance are provided.}:

\noindent(i) \textit{Online perturbation.} Let $\zeta_{t} \sim \mathrm{Bern}(0.75)$, all i.i.d. Then, the perturbed $\Lcal_{1}(\cdot)$ becomes,
\vspace{-0.25cm}
\begin{equation}
\resizebox{0.91\linewidth}{!}{$
\begin{aligned}
 \Lcal_{1}'(\theta, \vartheta, \eta) &= - \sum_{t=1}^{k-1} \zeta_{t}  \left[ \beta \langle {\tau}_t^{(Y_{t})} , \vartheta \rangle - \ln \bigg(e^{ \beta \langle {\tau}_t^{(0)}, \vartheta \rangle} + e^{\beta \langle {\tau}_t^{(1)}, \vartheta \rangle} \bigg)  \right. \\ & + \left. \sum_{j=0}^{1} \sum_{h=1}^{H-1} \ln P_{\eta} \left({s}_{t,h+1}^{(j)} \given {s}_{t,h}^{(j)}{a}_{t,h}^{(j)} \right) \right] \nonumber
\end{aligned}
$}
\end{equation}
\vspace{-0.5cm}

\noindent(ii) \textit{Offline perturbation.} Let $\omega_{n} \sim \mathrm{Bern}(0.6)$, all i.i.d. Then, the perturbed $\Lcal_{2}(\cdot)$ becomes,
\vspace{-0.3cm}
    \begin{equation}
    \resizebox{0.92\linewidth}{!}{$
    \begin{aligned}
     \Lcal_{2}'(\theta, \vartheta, \eta) = - \sum_{n=1}^{N} \omega_{n} \left[ \beta \langle \bar{\tau}_n^{(\bar{Y}_{n})} , \vartheta \rangle - \ln \bigg(e^{ \beta \langle \bar{\tau}_n^{(0)}, \vartheta \rangle} + e^{\beta \langle \bar{\tau}_n^{(1)}, \vartheta \rangle} \bigg) \right] \nonumber
    \end{aligned}
    $}
    \end{equation}
    
\noindent(iii) \textit{Prior perturbation.} Let $\theta' \sim \Ncal(\mu_{0}, \Sigma_{0})$, $\vartheta' \sim \Ncal (\mu_{0}, \Ibf_{d} /\lambda^{2})$, all i.i.d. The perturbed $\Lcal_{3}(\cdot)$ is,
\vspace{-0.6cm}
    \begin{equation}
    \resizebox{0.8\linewidth}{!}{$
    \begin{aligned}
       \Lcal_{3}'(\theta, \vartheta, \eta)  &= \frac{\lambda^2}{2} \norm{\theta - \vartheta + \vartheta'}{2}{2} - SA \sum_{i=1}^{S} (\bm{\alpha}_{0,i} - 1) \ln \eta_{i} \\ & + \frac{1}{2} (\theta - \mu_{0} - \theta')^{T} \Sigma_{0}^{-1} (\theta - \mu_{0}- \theta') \nonumber
    \end{aligned}
    $}
    \end{equation}
\vspace{-0.55cm}
    
Then, for the $k^{th}$ episode, we get the following MAP point estimate from the perturbed loss function, 

\vspace{-0.4cm}
\begin{equation}
\label{eq:surrogate_perturbed_loss}
\resizebox{0.87\linewidth}{!}{$
\begin{aligned}
    (\hat{\theta}_{k}, \hat{\vartheta}_{k}, \hat{\eta}_{k}) = \underset{\theta, \vartheta, \eta}{\argmin} \; \Lcal_{1}'(\theta, \vartheta, \eta) + \Lcal_{2}'(\theta, \vartheta, \eta) + \Lcal_{3}'(\theta, \vartheta, \eta), 
\end{aligned}
$}
\end{equation}
\vspace{-0.8cm}

which are understood to have a distribution that approximates the actual posterior distribution. Note that this loss function can be extended easily to the setting where the offline dataset comes from \emph{multiple} raters with different $(\lambda_{i}, \beta_{i})$ competencies. Specifically, for $M$ raters, there will be $M$ similar terms for $\Lcal_{1}'(\cdot)$, $\Lcal_{2}'(\cdot)$, and $\Lcal_{3}'(\cdot)$. The final algorithm is given as Algorithm \ref{alg:main_algo_empirical}.

\begin{figure}[t]
\vspace{-1em}
\begin{algorithm}[H]
   \caption{Bootstrapped Posterior Sampling for Preference Learning}
\begin{algorithmic}[1]
   \STATE {\bfseries Input:} Initial dataset $\Dcal_{0}$, priors $\nu_{0}(\theta)$ and $\chi_{0}(\eta)$.
    \FOR{$k = 1,2, \dots ,K+1$}
        \FOR{$i = 0,1$}
    	 \STATE Sample a set of perturbations $\Pcal_{k}^{(i)} = \{(\zeta_{t}^{(i)}, \omega_{n}^{(i)}, \theta'^{(i)}, \vartheta'^{(i)})\}$.
    	 \STATE Solve Equation \eqref{eq:surrogate_perturbed_loss} using $\Pcal_{k}^{(i)}$ to find $(\hat{\theta}_{k}^{(i)}, \hat{\vartheta}_{k}^{(i)}, \hat{\eta}_{k}^{(i)})$.
         \ENDFOR
         \STATE Compute optimal policies $\pi_{k}^{(0)}$ using $(\hat{\eta}_{k}^{(0)}, \hat{\theta}_{k}^{(0)})$ and $\pi_{k}^{(1)}$ using $(\hat{\eta}_{k}^{(1)}, \hat{\theta}_{k}^{(1)})$.
    \STATE Run two trajectories $\tau_{k}^{(0)} \sim \pi_{k}^{(0)}$ and $\tau_{k}^{(1)} \sim \pi_{k}^{(1)}$ for $H$ horizon.
    \STATE Get feedback $Y_{k}$ on $\tau_{k}^{(0)}$ and $\tau_{k}^{(1)}$, and append to dataset as $\Dcal_{k} = \Dcal_{k-1} \oplus \left(\tau_{k}^{(0)}, \tau_{k}^{(1)}, Y_{k} \right)$.
    \ENDFOR
    \STATE {\bfseries Output:} Optimal policy $\pi_{K+1}^{\star} \equiv \pi_{K+1}^{(0)} $.
\end{algorithmic}
\label{alg:main_algo_empirical}
\end{algorithm}
\vspace{-2em}
\end{figure}

%% file: experiments.tex
\vspace{-0.3cm}
\section{Empirical results}
\label{sec:experiments}
\vspace{-0.3cm}

We now present results on the empirical performance of Bootstrapped $\PSPL$ algorithm. We first demonstrate the effectiveness of PSPL on synthetic simulation benchmarks and then on realworld datasets. We study: 
(i) How much is the reduction in simple Bayesian regret after warm starting with an offline dataset? (ii) How much does the competence ($\lambda$ and $\beta$) of the rater who generated the offline preferences affect simple regret? (iii) Is $\PSPL$ robust to mis-specification of $\lambda$ and $\beta$? 

\noindent\textbf{Baselines.} To evaluate the Bootstrapped $\PSPL$ algorithm, we  consider the following baselines:
(i) Logistic Preference based Reinforcement Learning (LPbRL) \citep{saha2023dueling}, 
and (ii) Dueling Posterior Sampling (DPS) \citep{novoseller2020dueling}. LPbRL does not specify how to incorporate prior offline data in to the optimization problem, and hence, no data has been used to warm start LPbRL, but, we initialize the transition and reward models in DPS using $\Dcal_{0}$. We run and validate the performance of all algorithms in the {\fontfamily{lmss}\selectfont RiverSwim} \citep{STREHL20081309} and {\fontfamily{lmss}\selectfont MountainCar} \citep{moore1990efficient} environments. See Figure \ref{fig:simple-cum-regrets} for empirical Bayesian regret comparison.


\textbf{Value of Offline Preferences.} We first understand the impact of $\Dcal_{0}$ on the performance of $\PSPL$ as the parameters ($\beta$, $\lambda$ and $N$) vary. Figure \ref{fig:main_ablation} shows that as $\lambda$ increases (the rater has a better estimate of the reward model $\theta$), the regret reduces substantially. Also notice that (for fixed $\lambda$ and $N$) as $\beta$ increases, the regret reduces substantially. Lastly, for fixed $\beta$ and $\lambda$, as dataset size $N$ increases, even with a `mediocre' expert ($\beta = 5$) the regret reduces substantially. \textcolor{black}{Overall, notice that $\PSPL$ correctly incorporates rater competence, which is intuitive to understand since in practice, rater feedback is imperfect, so weighting by competence is crucial.}

\textbf{Sensitivity to specification errors.}
The  Bootstrapped $\PSPL$ algorithm in Section \ref{sec:practical_approx_pspl} requires knowledge of rater's parameters $\lambda, \beta$. We study the sensitivity of $\PSPL$'s performance to mis-specification of these parameters. See Appendix \ref{sec:appendix_practical_pspl} for procedure to estimate $\lambda, \beta$ in practice, and Appendix \ref{sec:appendix_ablation} for ablation studies on robustness of PSPL to mis-specified parameters and mis-specified preference generation expert policies.

\textcolor{black}{
\textbf{Results on Image Generation Tasks.}
We instantiate our framework on the Pick-a-Pic dataset of human preferences for text–to–image generation \citep{kirstain2023pick}. The dataset contains over 500,000 examples and 35,000 distinct prompts. Each example contains a prompt, sequence generations of two images, and a corresponding preference label. We let each generation be a trajectory, so the dataset contains trajectory preferences $\mathcal D_{0} = { (\tau_{i}^{+},\tau_{i}^{-},y_{i}) }_{i=1}^{N}$ with $y_i = 1$ iff $\tau_{i}^{+} \succ \tau_{i}^{-}$. Each trajectory $\tau =(p,z_{0 : T})$ is the entire latent denoising chain $z_{0 : T}$ of length $T$ for prompt $p$ sampled from some prompt distribution. See Appendix \ref{appendix:pickapic} for the full experimental setup. As LPbRL involves optimization over confidence sets that grow exponentially in state-action space, we only present results of DPS and PSPL. See Figure \ref{fig:pickapic_example} for example generations given some prompts, where we observe higher final reward (hence, lower simple regret) of PSPL as compared to DPS.
}

\begin{figure*}[ht]
\noindent
\begin{minipage}{\textwidth}
\begin{tcolorbox}[width=.495\textwidth, nobeforeafter, coltitle = black, fonttitle=\fontfamily{lmss}\selectfont, title=MountainCar, halign title=flush center, colback=backg_blue!5, colframe=brown!25, boxrule=2pt, grow to left by=-0.5mm, left=-1pt, right=-1pt]
    \centering
    {   
        \includegraphics[height=0.33\textwidth, width=0.42\textwidth]{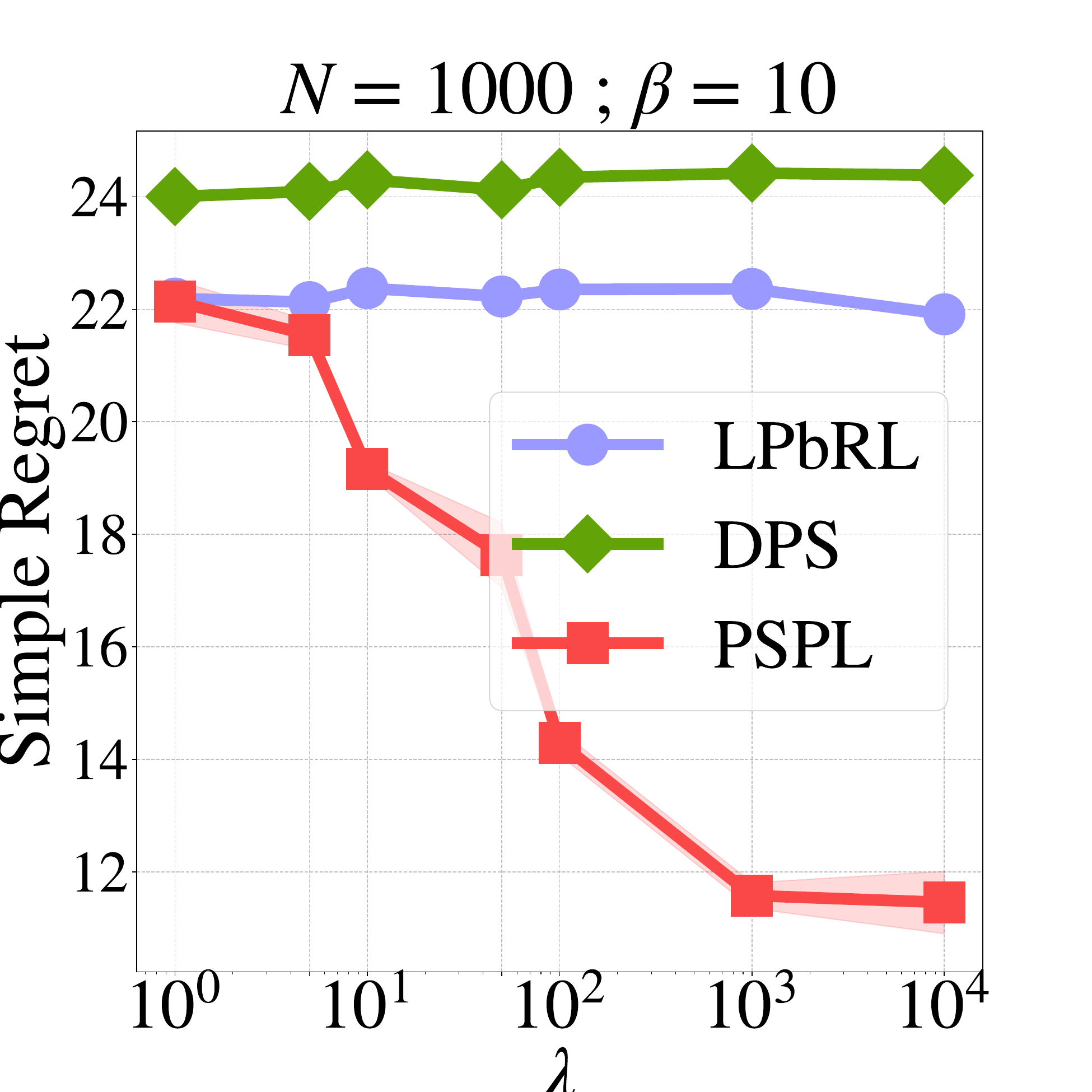}
    } 
    {
        \includegraphics[height=0.33\textwidth, width=0.42\textwidth]{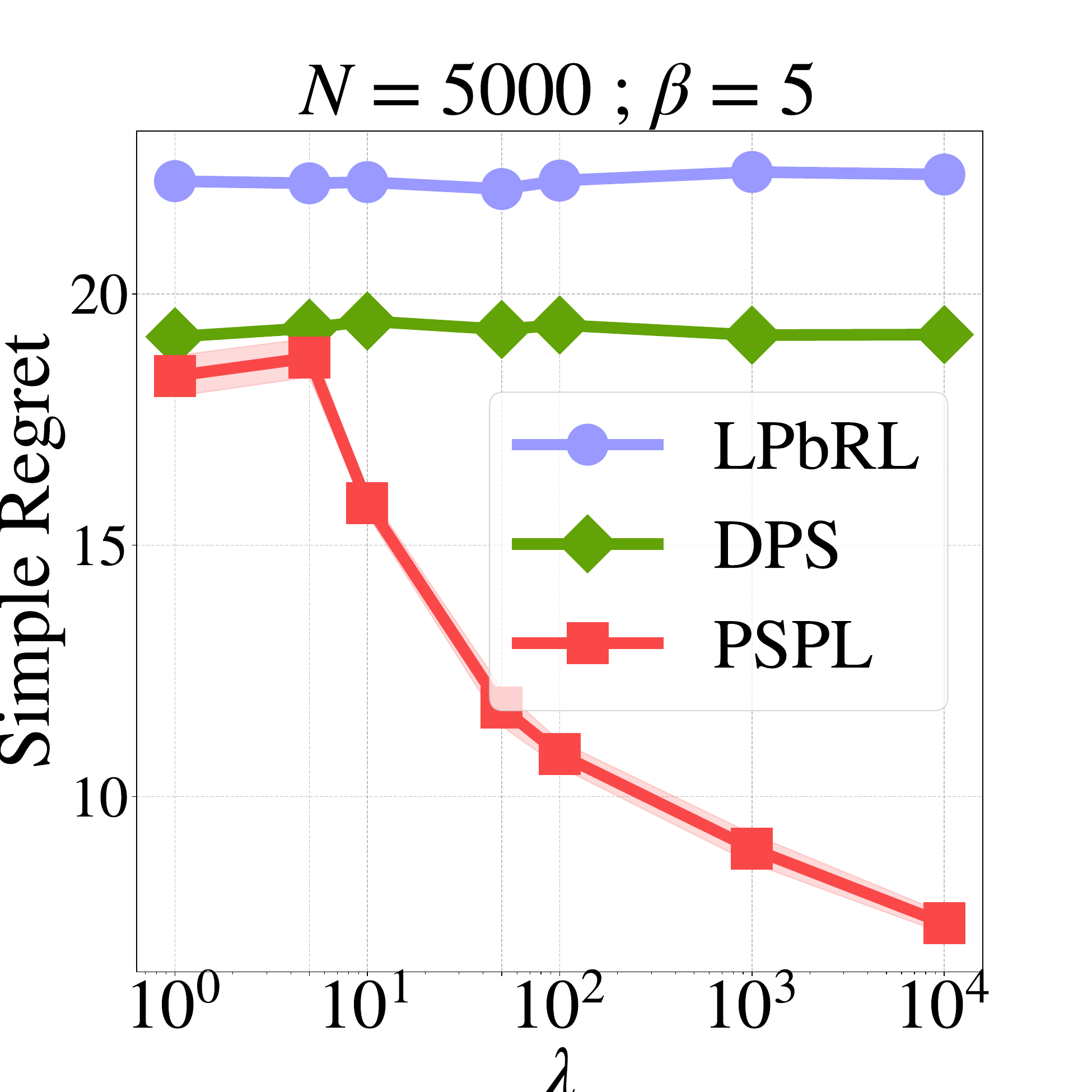}
    } 
    \newline
    \begin{center} \vspace{-0.60cm} 
        (a) Varying $\lambda$ for fixed $\beta$ and $N$. 
    \end{center}
    {
        \includegraphics[height=0.33\textwidth, width=0.42\textwidth]{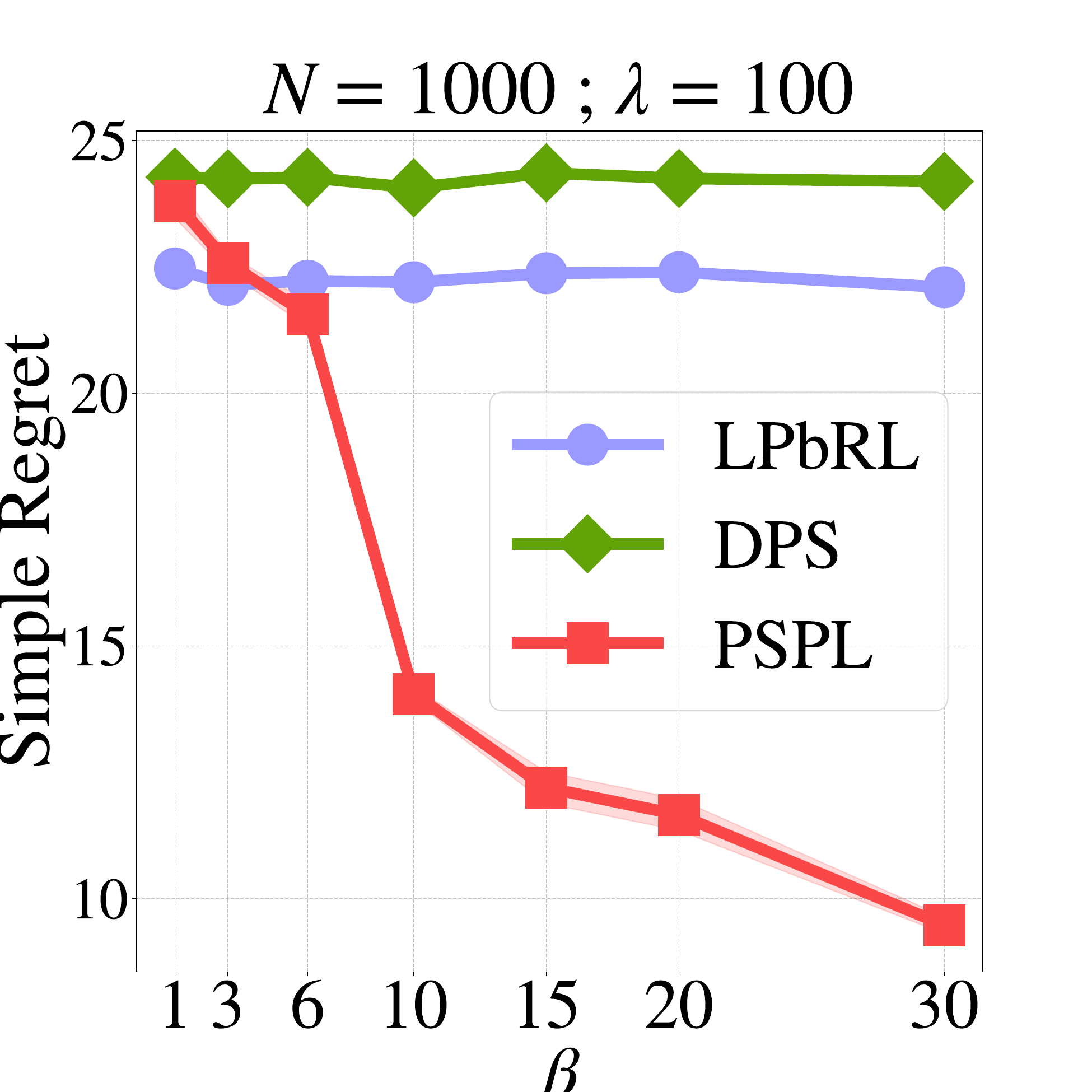}
    } 
    {
        \includegraphics[height=0.33\textwidth, width=0.42\textwidth]{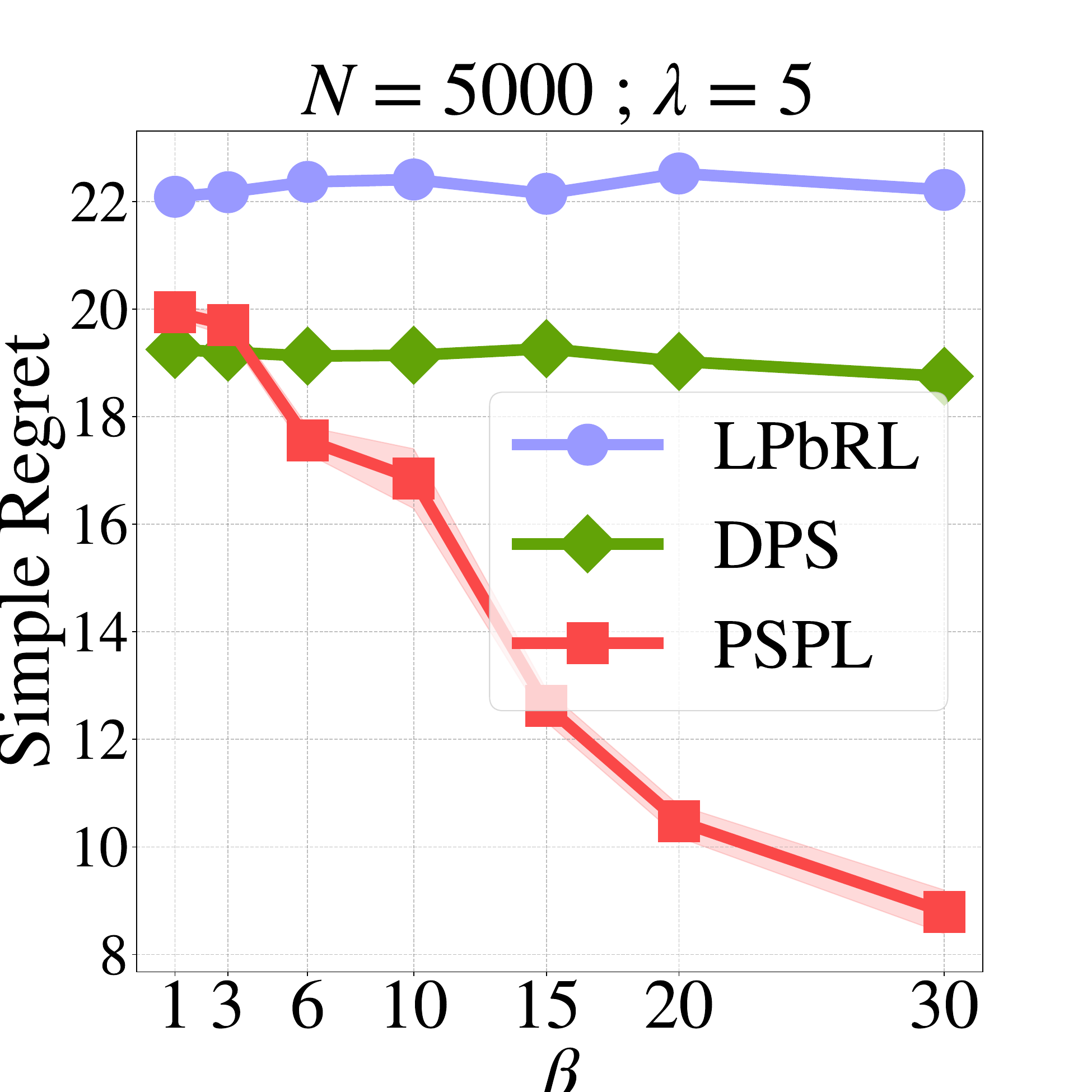}
    } 
    \newline
    \begin{center} \vspace{-0.60cm}
        (b) Varying $\beta$ for fixed $\lambda$ and $N$. 
    \end{center}
    {
        \includegraphics[height=0.33\textwidth, width=0.42\textwidth]{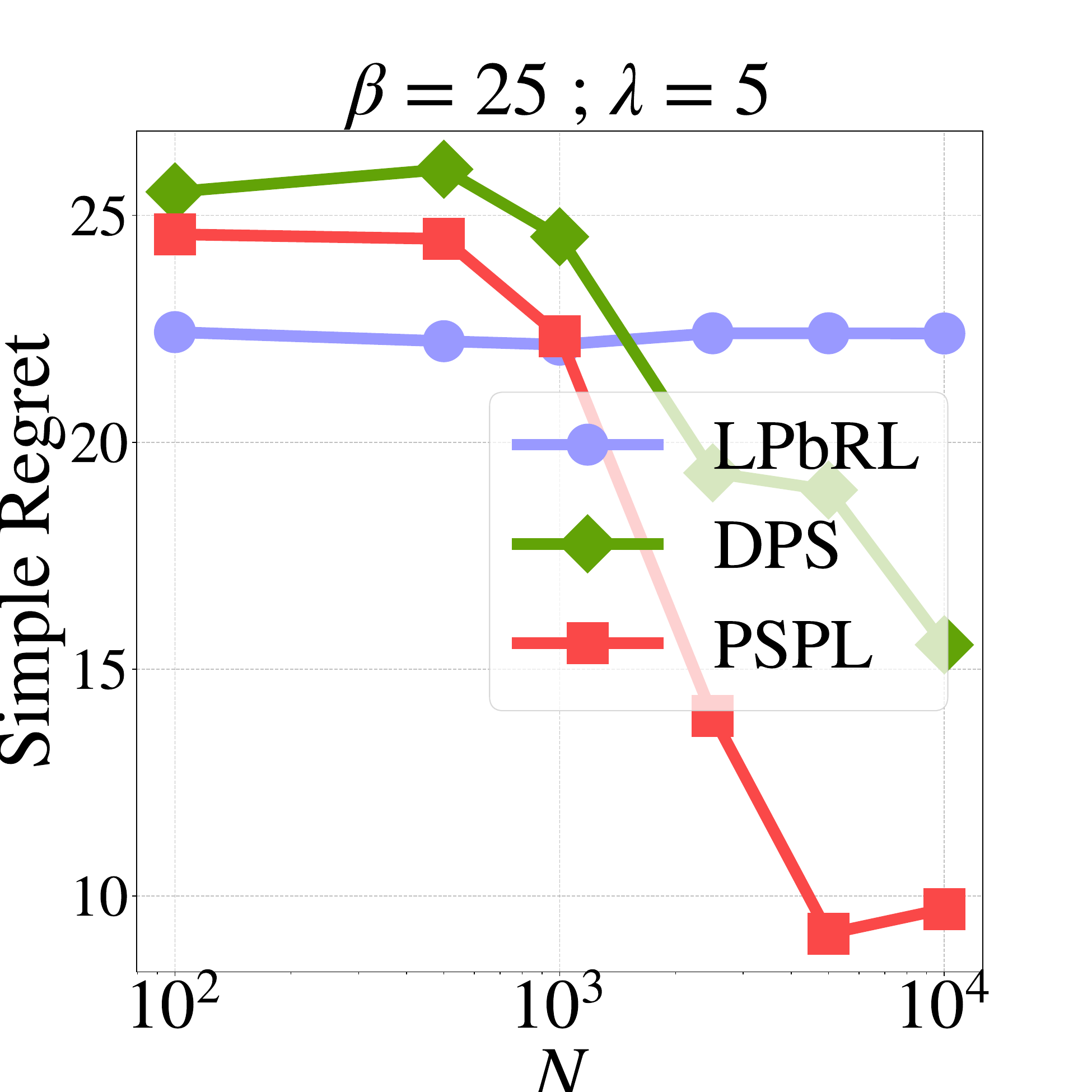}
    } 
    {
        \includegraphics[height=0.33\textwidth, width=0.42\textwidth]{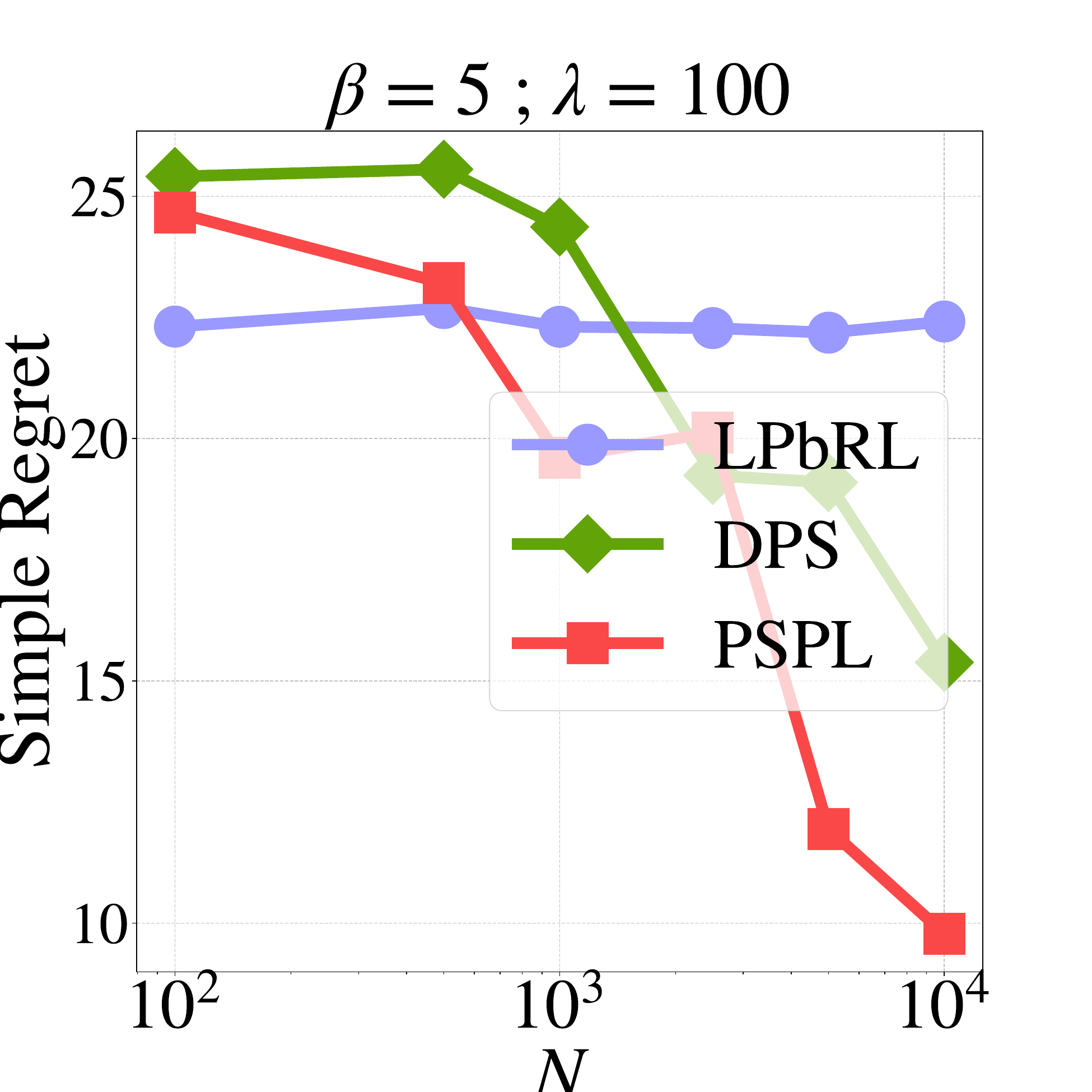}
    } 
    \newline
    \begin{center} \vspace{-0.60cm}
        (c) Varying $N$ for fixed $\beta$ and $\lambda$. 
    \end{center}
\end{tcolorbox} \hfill
\begin{tcolorbox}[width=.495\textwidth, nobeforeafter, fonttitle=\fontfamily{lmss}\selectfont, coltitle = black, title=RiverSwim, halign title=flush center,  colback=backg_blue!5, colframe=teal!15, boxrule=2pt, grow to left by=-0.5mm, grow to left by=-0.5mm, left=-1pt, right=-1pt]
\centering
    {   
        \includegraphics[height=0.33\textwidth, width=0.42\textwidth]{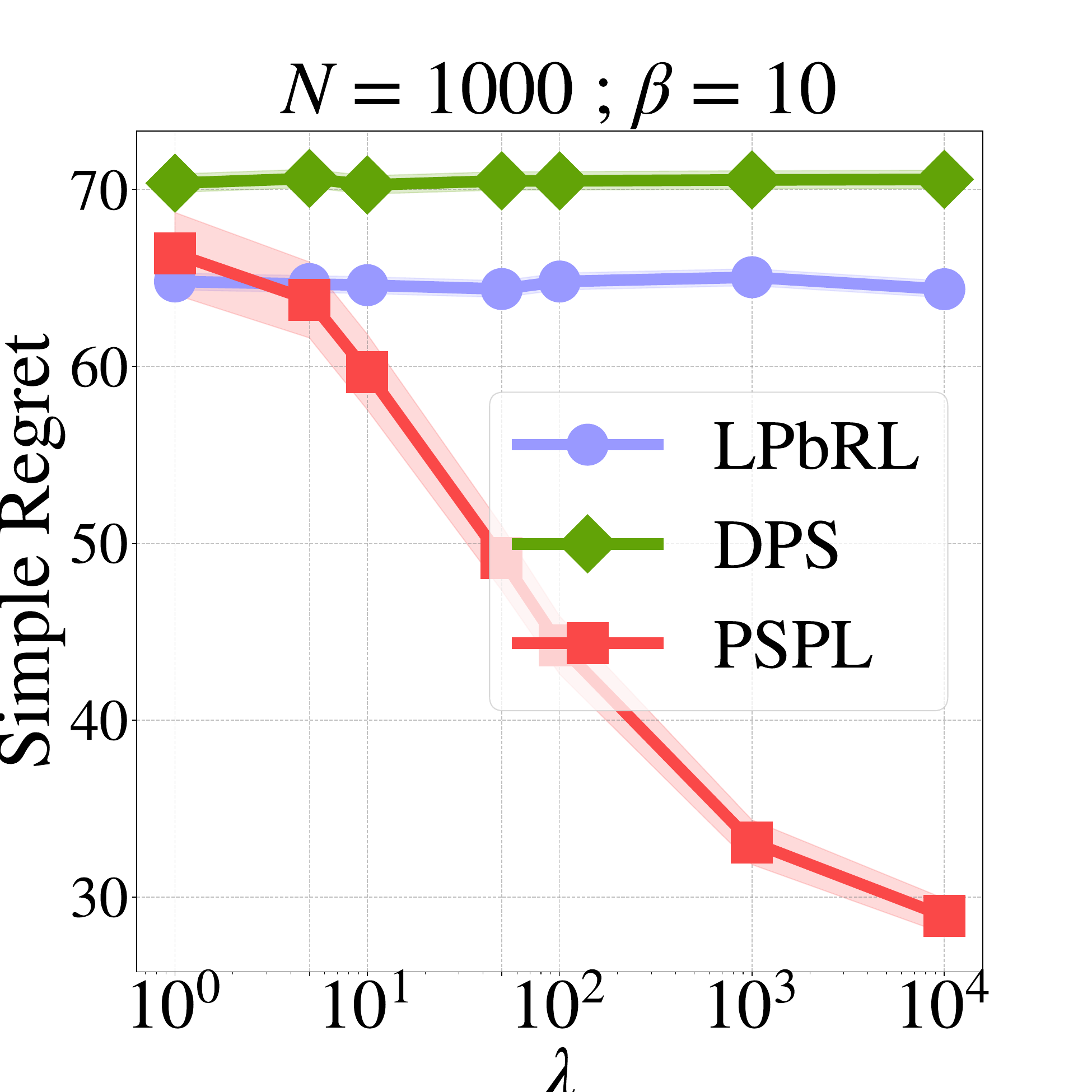}
    } 
    {
        \includegraphics[height=0.33\textwidth, width=0.42\textwidth]{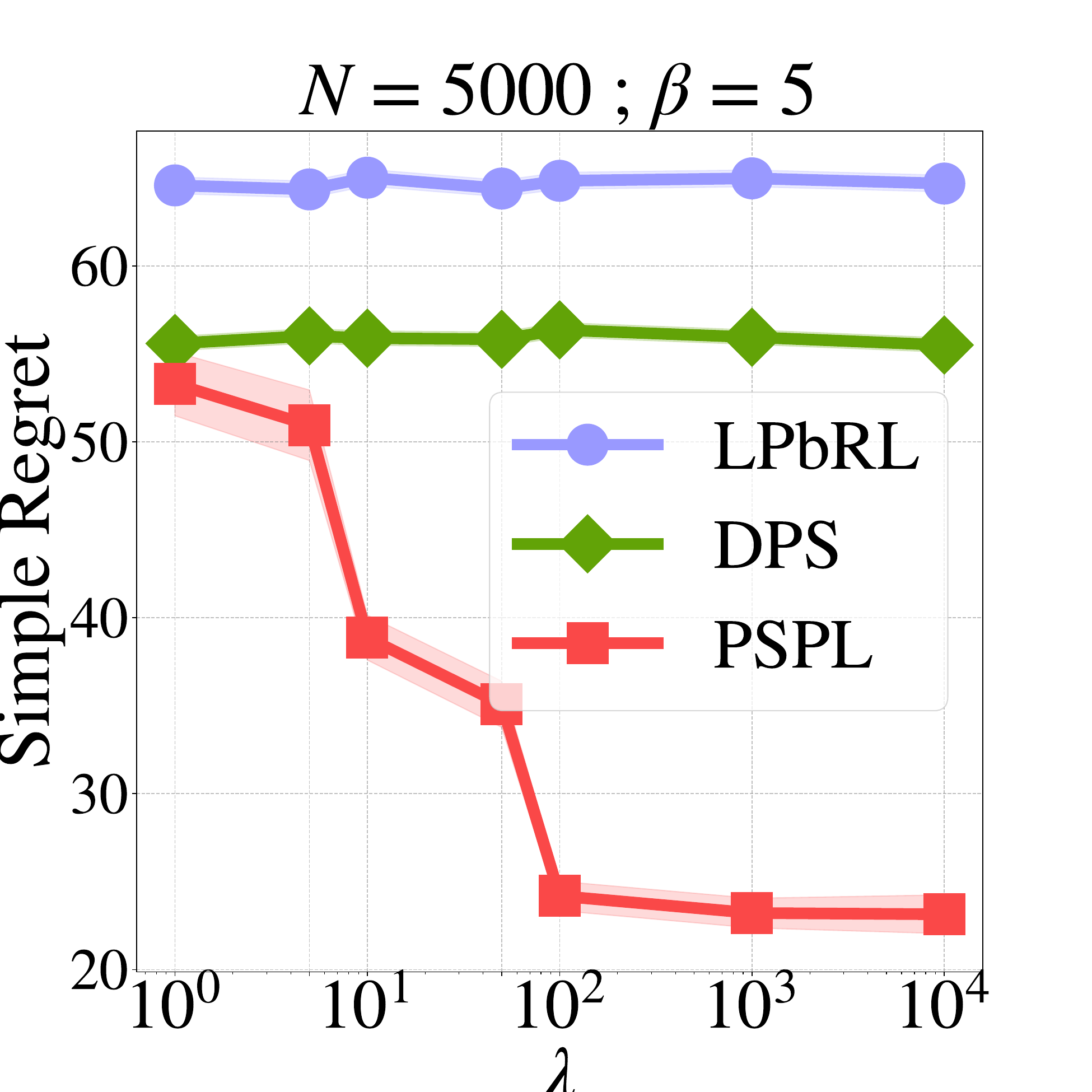}
    } 
    \newline
    \begin{center} \vspace{-0.60cm} 
        (a) Varying $\lambda$ for fixed $\beta$ and $N$. 
    \end{center}
    {
        \includegraphics[height=0.33\textwidth, width=0.42\textwidth]{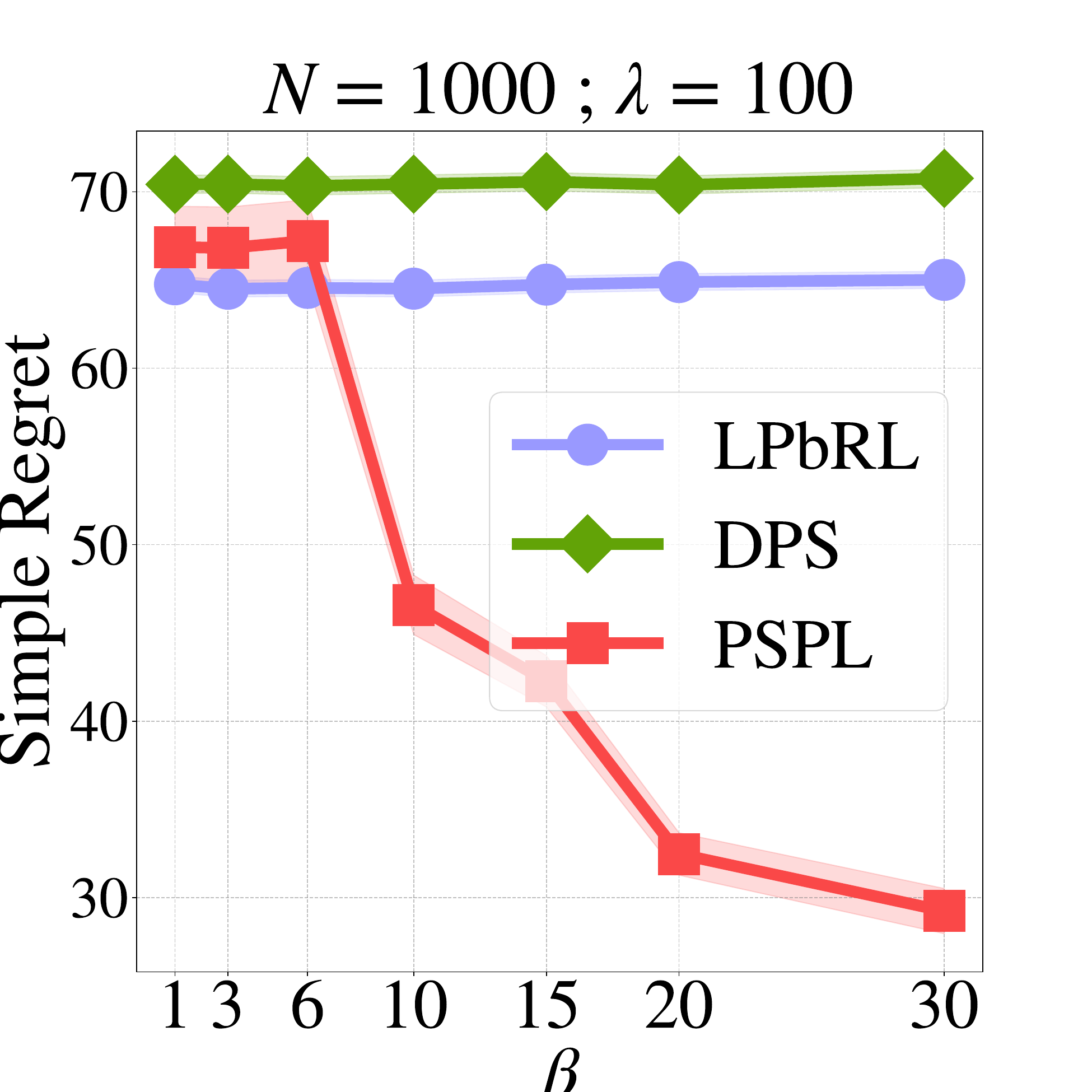}
    } 
    {
        \includegraphics[height=0.33\textwidth, width=0.42\textwidth]{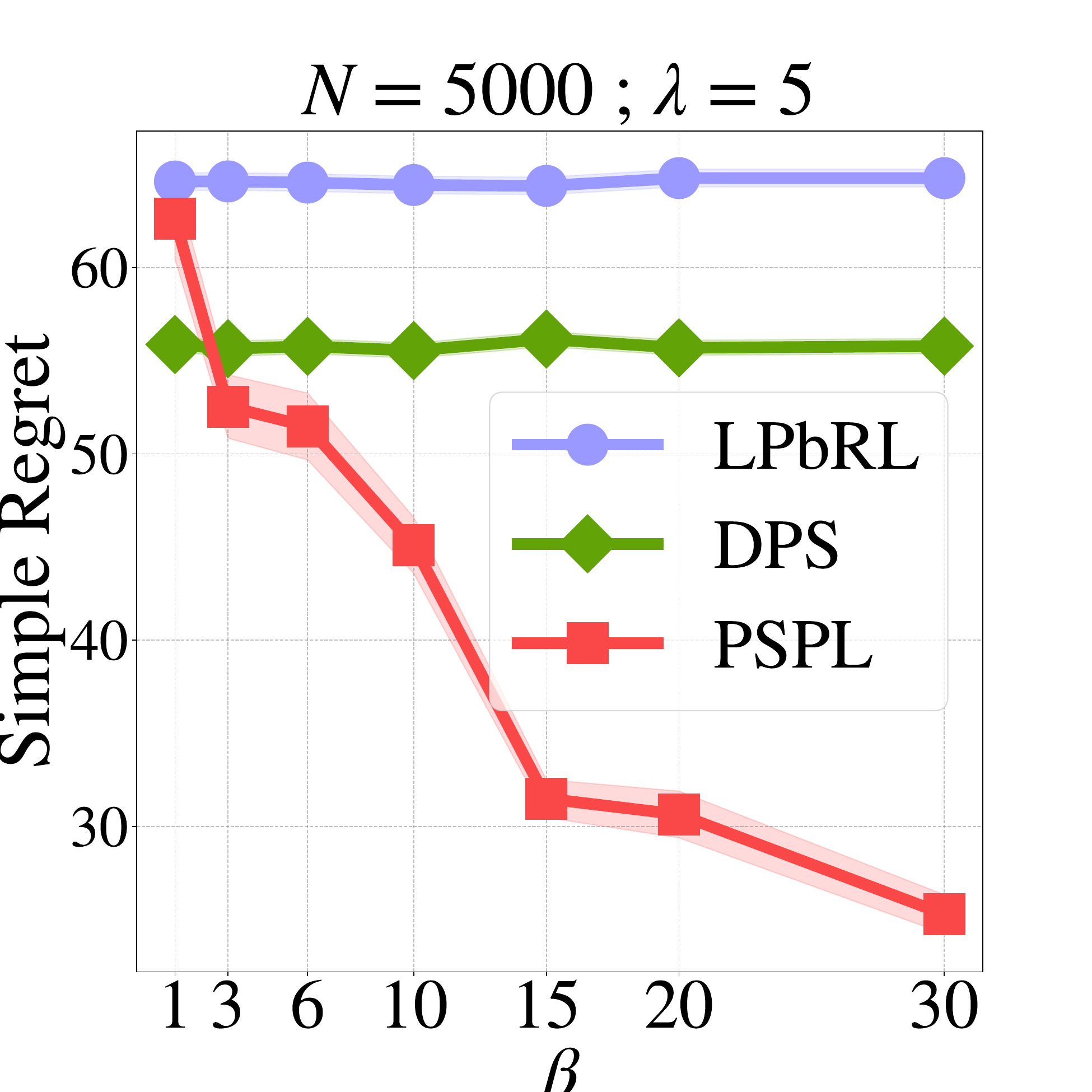}
    } 
    \newline
    \begin{center} \vspace{-0.60cm}
        (b) Varying $\beta$ for fixed $\lambda$ and $N$. 
    \end{center}
    {
        \includegraphics[height=0.33\textwidth, width=0.42\textwidth]{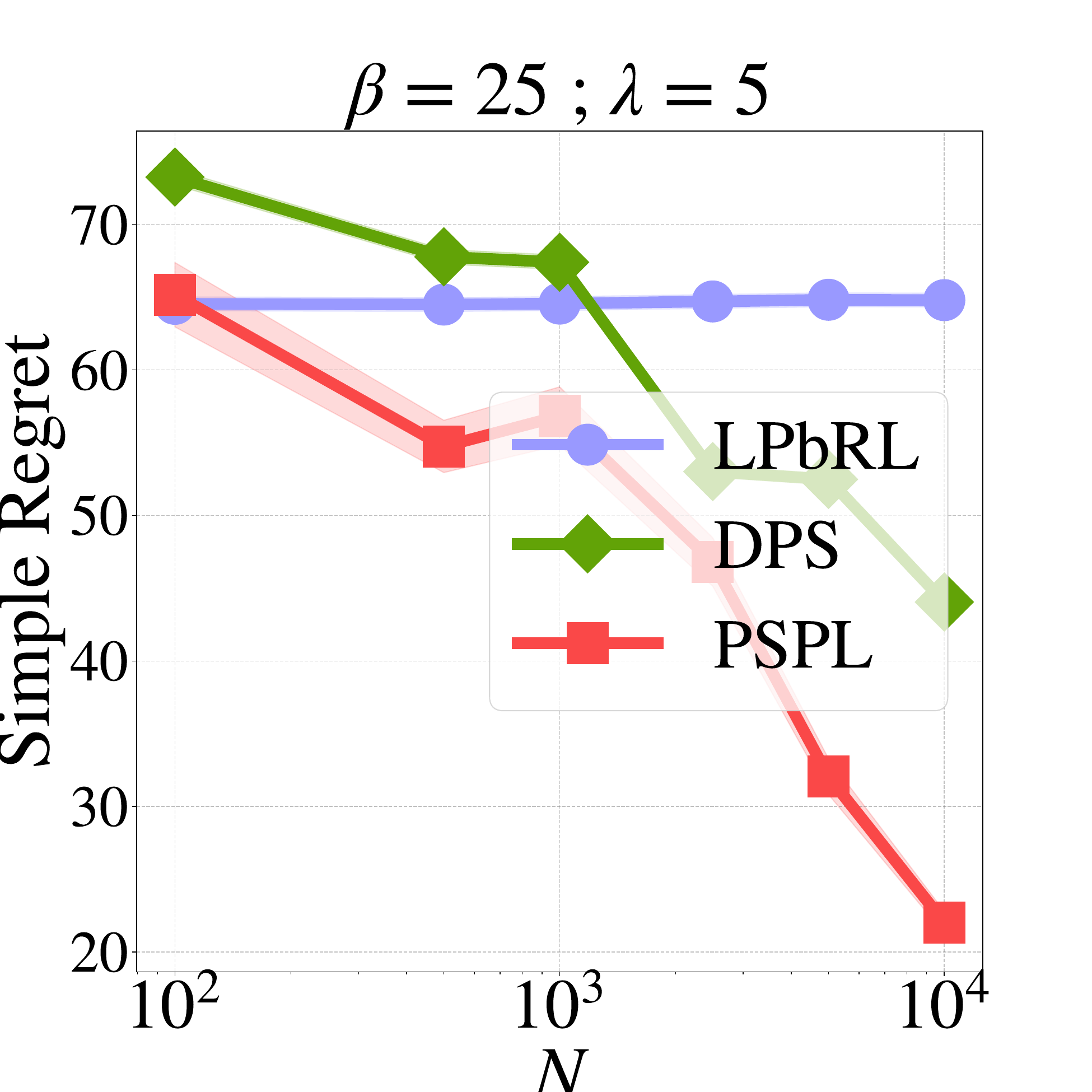}
    } 
    {
        \includegraphics[height=0.33\textwidth, width=0.42\textwidth]{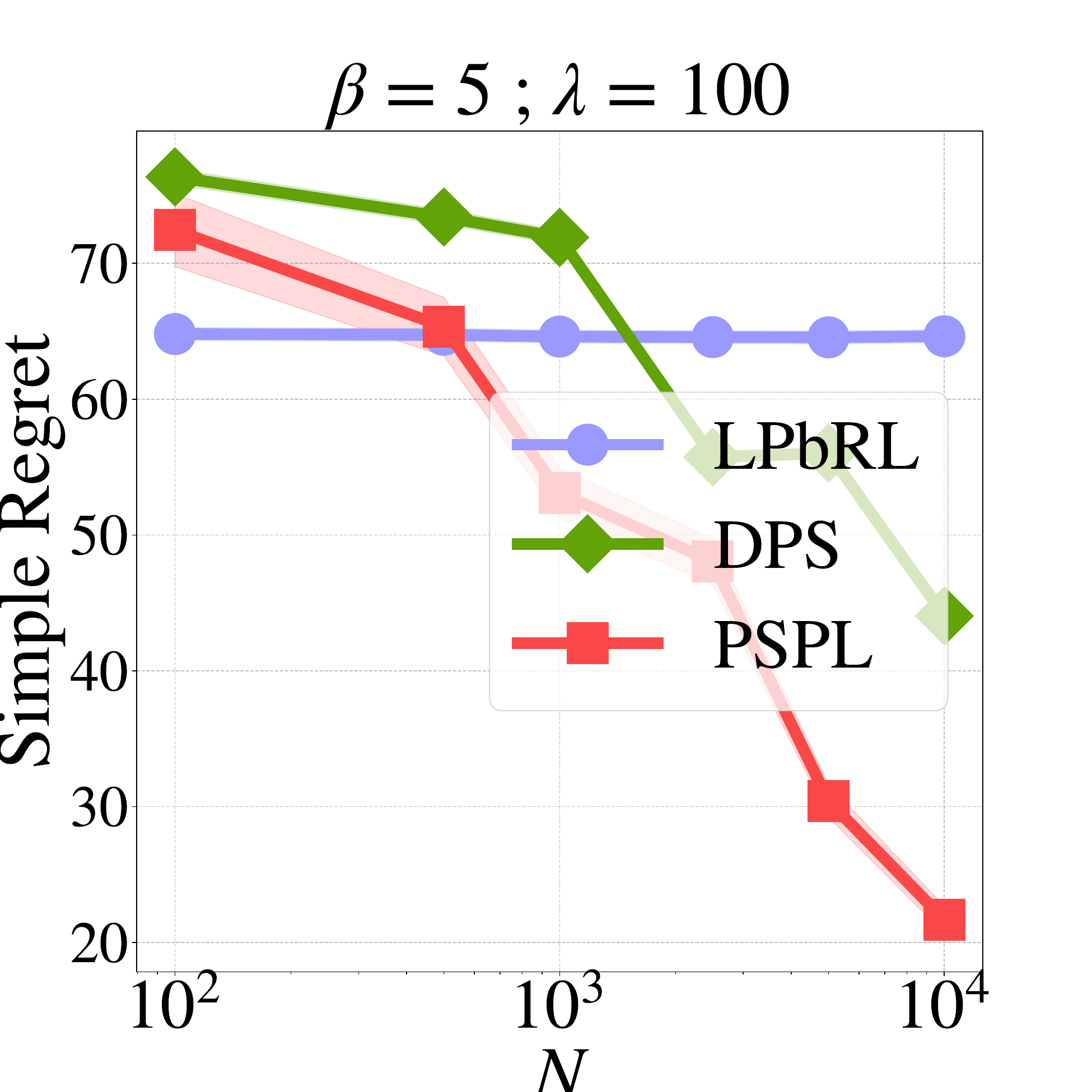}
    } 
    \newline
    \begin{center} \vspace{-0.60cm}
        (c) Varying $N$ for fixed $\beta$ and $\lambda$. 
    \end{center}
\end{tcolorbox}
\end{minipage}
\vspace{-0.3cm}
\caption{$\PSPL$ with varying $N$, $\beta$, and $\lambda$ in benchmark environments. Shaded region around mean line represents 1 standard deviation over 5 independent runs.}
\label{fig:main_ablation}
\end{figure*}

\begin{figure}[ht]
\centering
\begin{tcolorbox}[width=.45\textwidth, nobeforeafter, coltitle = black, fonttitle=\fontfamily{lmss}\selectfont, title= MountainCar, halign title=flush center, colback=backg_blue!5, colframe=lightorange!40, boxrule=2pt, grow to left by=-0.5mm, left=-15pt, right=-15pt]
    \centering
    {
        \includegraphics[height=0.32\textwidth, width=0.38\textwidth]{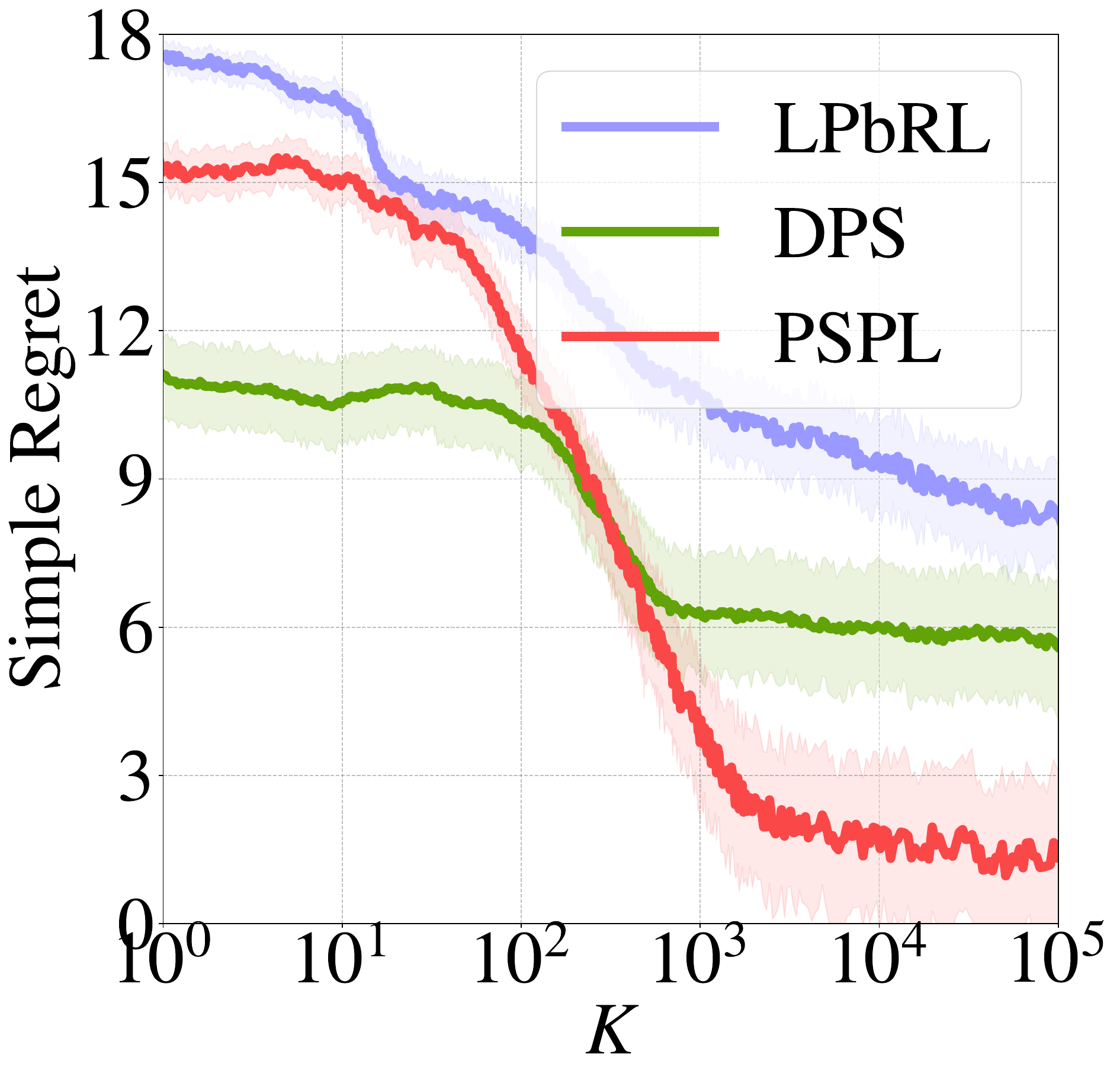}
    } 
    {
        \includegraphics[height=0.32\textwidth, width=0.38\textwidth]{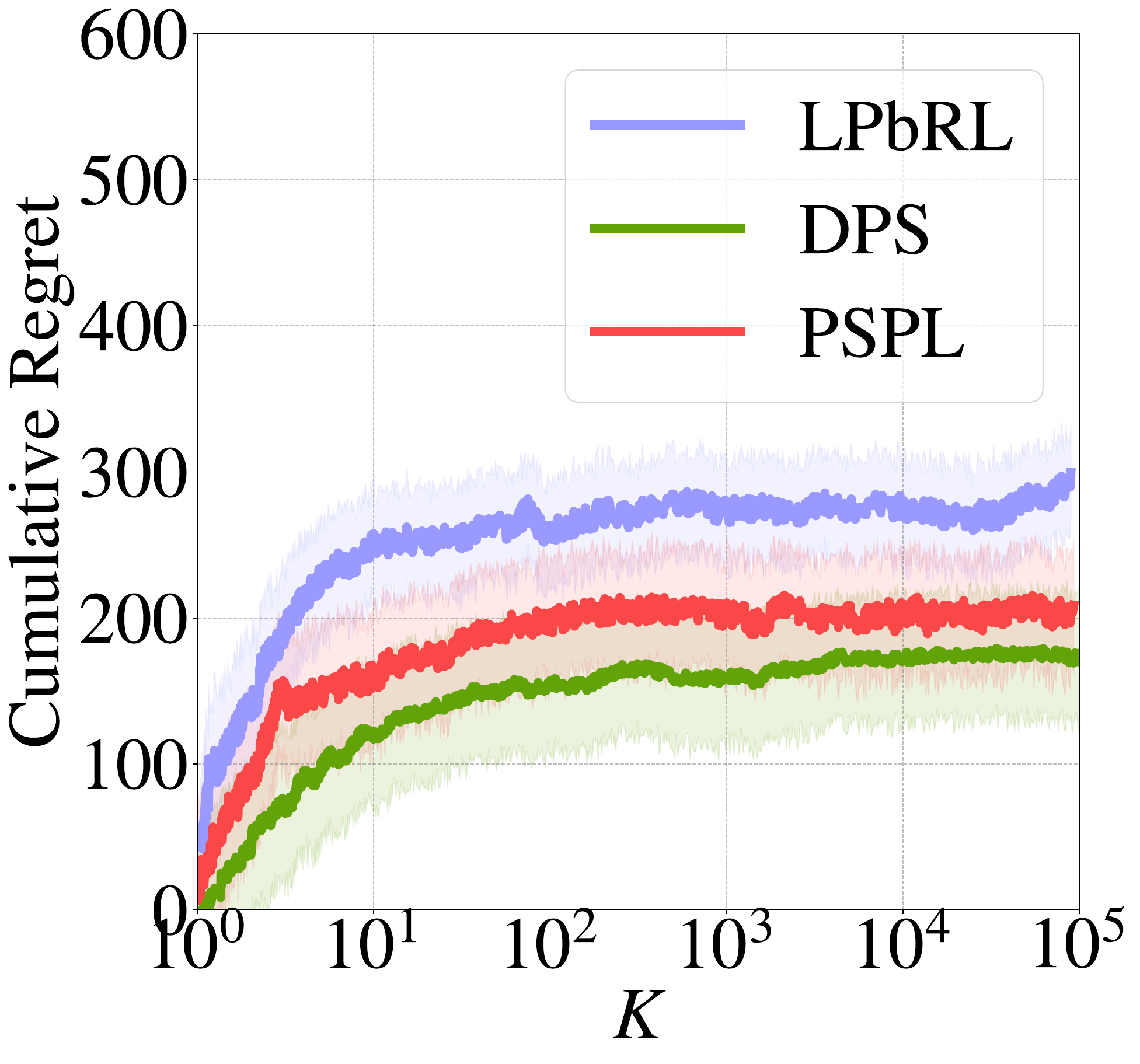}
    } 
\end{tcolorbox} 
\begin{tcolorbox}[width=.45\textwidth, nobeforeafter, coltitle = black, fonttitle=\fontfamily{lmss}\selectfont, title= RiverSwim, halign title=flush center, colback=backg_blue!5, colframe=brightblue!25, boxrule=2pt, grow to left by=-0.5mm, left=-15pt, right=-15pt]
    \centering
    {
        \includegraphics[height=0.32\textwidth, width=0.38\textwidth]{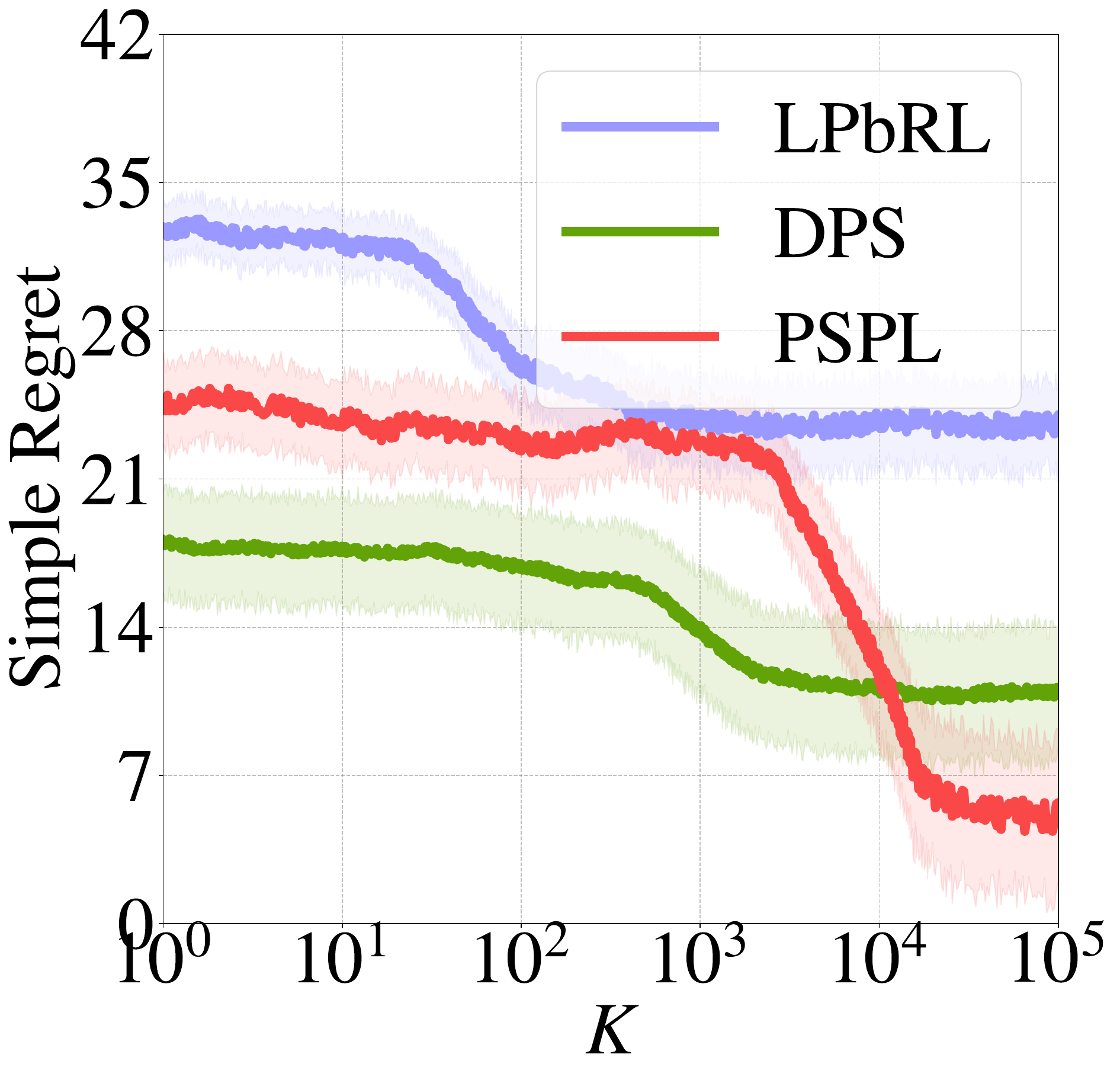}
    } 
    {
        \includegraphics[height=0.32\textwidth, width=0.38\textwidth]{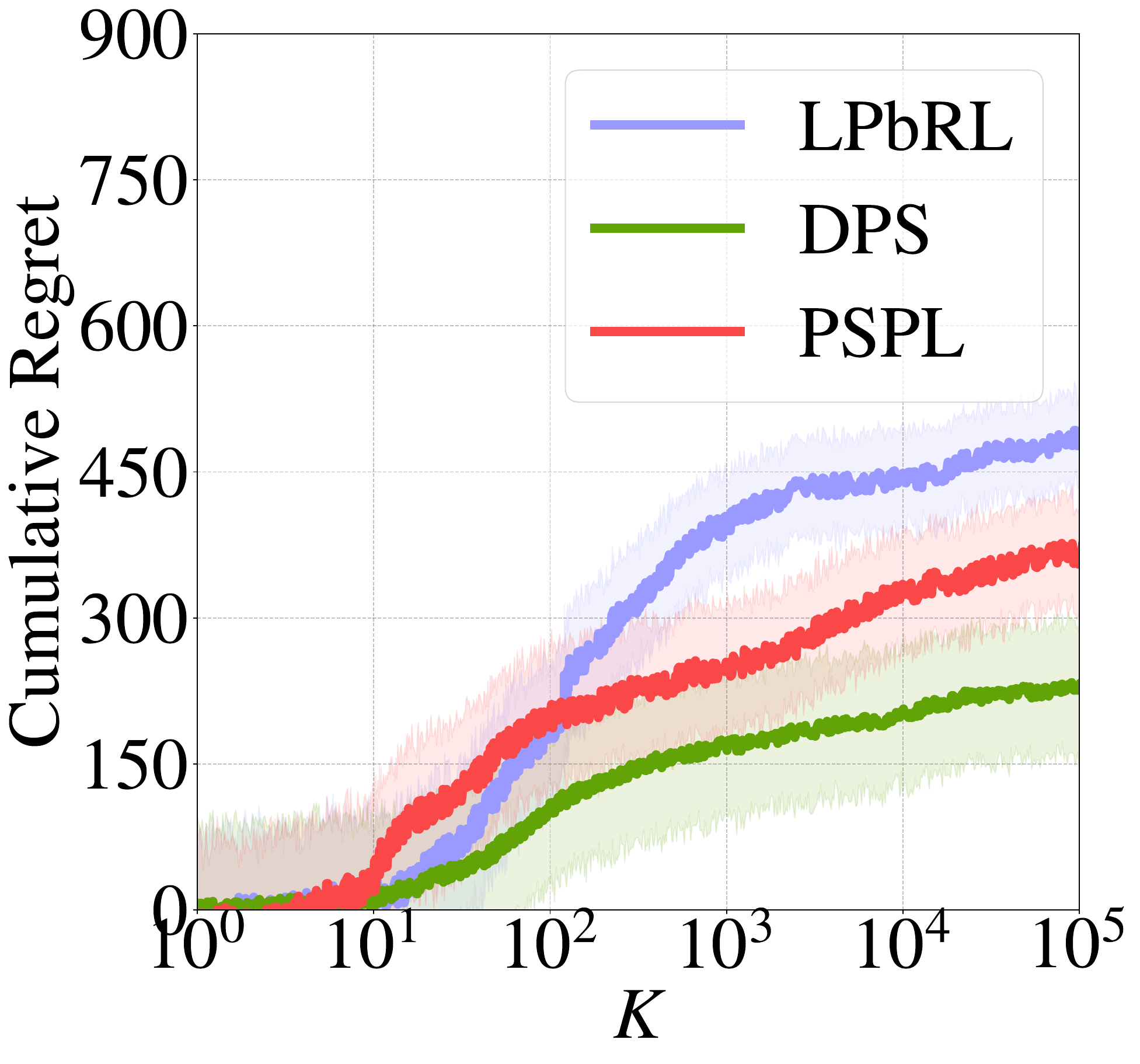}
    } 
\end{tcolorbox} 
\vspace{-0.2cm}
\caption{\small Simple and Cumulative Regret ($\div 10^{3}$) vs $K$ plots. $\PSPL$ is run with $\lambda=50,\beta=10,N=10^{3}$.}
\label{fig:simple-cum-regrets}
\end{figure}

\begin{figure}[ht]
\centering
\begin{tcolorbox}[width=.4\textwidth, nobeforeafter, coltitle = black, fonttitle=\fontfamily{lmss}\selectfont, title= Image Generation Tasks, halign title=flush center, colback=backg_blue!5, colframe=violet!25, boxrule=2pt, grow to left by=-0.2mm, left=-15pt, right=-15pt]
    \centering
    {
        \includegraphics[height=0.85\textwidth, width=0.85\textwidth]{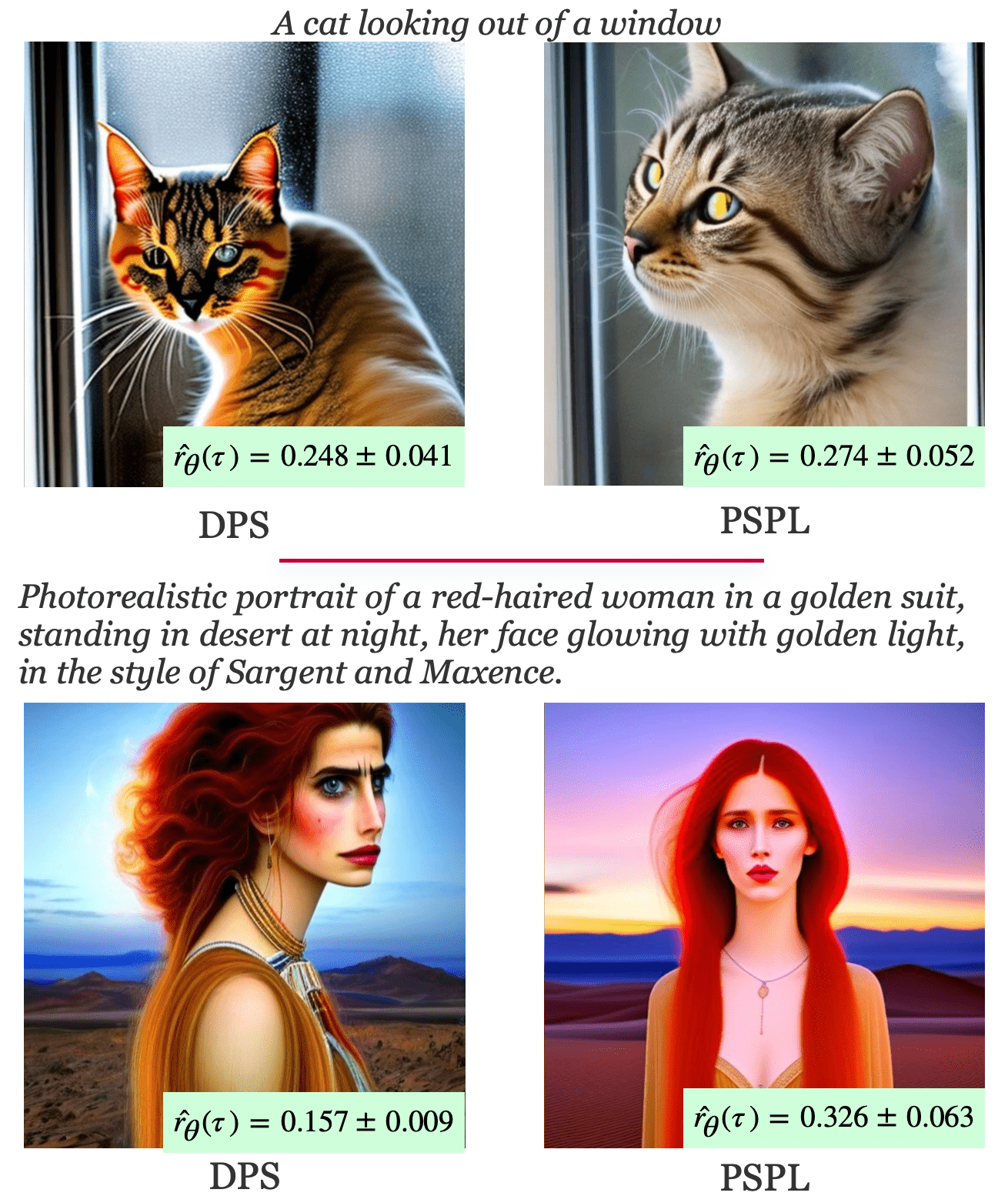}
    } 
\end{tcolorbox} 
\caption{\small Sample image generations along with final image reward $\hat{r}_{\theta}(\cdot)$ over 5 independent runs.}
\label{fig:pickapic_example}
\end{figure}

%% file: conclusion.tex
\vspace{-0.2cm}
\section{Conclusion}
\label{sec:pspl_conclusion}
\vspace{-0.2cm}

We have proposed the $\PSPL$ algorithm that incorporates offline preferences to learn an MDP with unknown rewards and transitions. We provide the first analysis to bridge the gap between offline preferences and online learning. We also introduced Bootstrapped $\PSPL$, a computationally tractable extension, and our empirical results demonstrate superior performance of $\PSPL$ in challenging benchmarks as compared to baselines. Overall, we provide a promising foundation for further development in PbRL and RLHF space.

%% file: appendix.tex
\section{Appendix}
\label{sec:appendix}



\paragraph{Network Architecture for DPO \cite{rafailov2024direct} and IPO \cite{ipo}.} We use a lightweight shared‐trunk MLP with two hidden layers of 64 ReLU units each (Xavier initialization \cite{pmlr-v9-glorot10a}), followed by a policy action head, outputting action logits for {\fontfamily{lmss}\selectfont MountainCar} \citep{moore1990efficient} and {\fontfamily{lmss}\selectfont RiverSwim} \citep{STREHL20081309} environments. We train with Adam (learning rate \(3\times10^{-4}\) \cite{kingma2014adam}), and apply dropout (0.1) on the trunk to regularize. This architecture balances expressivity and sample efficiency on these control tasks.

DPO \cite{rafailov2024direct} is an alternative approach to the RL paradigm, which avoids the training of a reward model altogether. The loss that DPO optimizes to obtain the optimal policy, given an empirical dataset $\Dcal = \{y_{w}, y_{l}\}$ of the winning (preferred) $y_{w}$ and losing (not preferred) $y_{l}$ trajectories, as a function of the reference policy $\pi_{\text{ref}}$ and regularization strength $\tau \in \Rbb$, is given by:

\begin{align*}
    \pi^{\star}_{\text{DPO}} = \argmin_{\pi} \quad \Ebb_{(y_{w}, y_{l}) \sim \Dcal} \left[ - \log \sigma \left( \tau \log \left( \frac{\pi(y_{w})}{\pi(y_{l})} \right)  - \tau \log \left( \frac{\pi_{\text{ref}}(y_{w})}{\pi_{\text{ref}}(y_{l})} \right) \right)  \right]
\end{align*}

, where $\sigma(\cdot)$ denotes the sigmoid function.

IPO is an instance of the $\Psi\text{PO}$ algorithm \cite{ipo} . The loss function that IPO optimizes is given by,

\begin{align*}
    \pi^{\star}_{\text{IPO}} = \argmin_{\pi} \quad \Ebb_{(y_{w}, y_{l}) \sim \Dcal} \left[ h_{\pi}(y_{w}, y_{l}) - \frac{1}{2 \tau}  \right]^{2} \; \text{where} \, , \; h_{\pi}(y, y') := \log \left( \frac{\pi(y) \pi_{\text{ref}}(y')}{\pi(y') \pi_{\text{ref}}(y)} \right) \, .
\end{align*}

\subsection{Prior-dependent analysis}
\label{appendix:prior_dependent_analysis}

\begin{lemma}[Monotone Contraction]
\label{lemma:contractionPS}
$\PSPL$ is monotone with respect to optimality of candidate policies $\pi_{k}^{(0)}$ and $\pi_{k}^{(1)}$ at any episode $k \in [K]$ i.e. $\Pr(\pi_{k}^{(i)} \neq \pi^{\star}) \leq \Pr(\pi_{1}^{(i)} \neq \pi^{\star})$ for $i \in \{0,1\}$ $\forAll k \geq 1$.
\end{lemma}
\begin{proof}
    Define $f(x) = x(1-x)$. $f$ is a concave function. We have for any $i \in \{0,1\}$, 
    \begin{align*}
        \Pr(\pi_{k}^{(i)} \neq \pi^{\star} ) &= \Ebb \left[\Pr \left(\pi_{k}^{(i)} \neq \pi^{\star} \Given \Dcal_{k} \right) \right]= \Ebb \left[ \sum_{\pi \in \Pi} \Pr(\pi_{k}^{(i)} = \pi, \pi^{\star} \neq \pi \Given \Dcal_{k})  \right]\\
        &= \Ebb \left[ \sum_{\pi\in\Pi} f \left(\Pr(\pi^{\star} = \pi \Given \Dcal_{k}) \right)  \right] \\
        &=  \Ebb \left[\sum_{\pi\in\Pi} \E{\Dcal_{1}} \left[f \left(\Pr \left(\pi^{\star} = \pi \given \Dcal_{k} \right) \right) \right] \right] \leq \Ebb \left[ \sum_{\pi\in\Pi} f \left(\E{\Dcal_{1}} \left[\Pr \left(\pi^{\star} = \pi \given \Dcal_{k} \right) \right] \right)\right] \\
        &= \Ebb\left[\sum_{\pi\in\Pi} f \left( \Pr \left(\pi^{\star} = \pi \given \Dcal_{1} \right) \right) \right]= \Ebb\left[ \sum_{\pi\in\Pi} \Pr(\pi_{1}^{(i)} = \pi, \pi^{\star} \neq \pi \given \Dcal_{1})  \right]\\
        &= \Pr(\pi_{1}^{(i)} \neq \pi^{\star}) \\
        \implies &  \Pr(\pi_{k}^{(i)} \neq \pi^{\star} ) \leq \Pr(\pi_{1}^{(i)} \neq \pi^{\star})
    \end{align*}
\end{proof}

\begin{lemma}
\label{lemma:empiricalcountsbound}
For any policy $\pi$, let an event $ E_{1} := \left\{  k, s, a: \, n_{k}(s, a) < \frac{1}{2} \sum_{j<k} w_{j} (s, a) - H\ln \left( \frac{SAH}{\delta'} \right) \right\}$ for some $\delta' \in (0,1)$, where $w_{h,j} (s, a)$ is the probability of visiting the $(s, a)$ pair at time $h$ of episode $j$ under the chosen policy, and $w_{j}(s,a) = \sum_{h} w_{h,j} (s, a)$ is the sum of the probabilities of visiting the $(s, a)$ pair in episode $j$. Then, $\Pr(E_{1}) \leq \delta' SAH = \delta$, where we set $\delta' = \frac{\delta}{SAH}$.
\end{lemma}
\begin{proof}

Consider a fixed $s \in \Scal, a \in \Acal, h \in [H]$, and denote the state and action visited in the $k^{\text{th}}$ episode at step $h$ as $s_{k,h}$ and $a_{k,h}$ respectively. We define $\mathcal F_k$ to be the sigma-field induced by the first $k-1$ episodes. Let $X_k$ be the indicator whether $s, a$ was observed in episode $k$ at time $h$. The probability $\Pr(s=s_{k,h}, a=a_{k,h} \given s_{k,1}, \pi_{k}^{(i)})$ (for $i \in \{0,1\}$) of whether $X_k = 1$ is $\mathcal F_k$-measurable and hence we can apply Lemma F.4 from \cite{dann2017unifying} with $W = \ln \frac{SAH}{\delta'}$ and obtain that $\Pr(E_{1}) \leq SAH \delta'$ after summing over all statements for $h \in [H]$ and applying the union bound over $s, a, h$.
\end{proof}

\begin{lemma}
\label{lemma:transition_est_error}
For all estimates $\hat{\theta}_{k}^{(i)}, \hat{\eta}_{k}^{(i)}$  ($i \in \{0,1\}$) of $\; \theta, \eta \;$ in any fixed episode $k \in [K]$, and for all $(s, a) \in \Scal \times \Acal$, with probability at least $1 - 2 \delta$ we have, 
$$
\sum_{s' \in \Scal} \left| \Pbb_{\theta, \eta}(s' \given s,a) - \Pbb_{\hat{\theta}_{k}^{(i)}, \hat{\eta}_{k}^{(i)}} (s' \given s,a) \right| \leq \sqrt{\frac{4S \ln \left( \frac{2KSA}{\delta}  \right)}{\sum_{j<k} w_{j} (s, a) - 2H\ln \left( \frac{SAH}{\delta} \right)}}
$$
\end{lemma}
\begin{proof}

\cite{weissman2003inequalities} gives the following high probability bound on the one norm of the maximum likelihood estimate using posterior sampling. In particular, with probability at least $1-\delta$, it holds that:
$$
\sum_{s' \in \Scal} \left| \Pbb_{\theta, \eta}(s' \given s,a) - \Pbb_{\hat{\theta}_{k}^{(i)}, \hat{\eta}_{k}^{(i)}} (s' \given s,a) \right| \leq \sqrt{2S \ln \left( \frac{2KSA}{\delta}  \right) /  \, n_{k}(s,a)} \quad .
$$

Then, using Lemma \ref{lemma:empiricalcountsbound}, the proof is complete.
\end{proof}

\begin{lemma}
\label{lemma:all_states_actions_visited}
Under Algorithm \ref{alg:main_algo_theoretical} for large enough K, every reachable state-action pair is visited infinitely-often.
\end{lemma}

\begin{proof}
The proof proceeds by assuming that there exists a state-action pair that is visited only finitely-many times.  This assumption will lead to a contradiction \footnote{Note that in finite-horizon MDPs, the concept of visiting a state finitely-many times is not the same as that of a transient state in an infinite Markov chain, because: 1)  due to a finite horizon, the state is resampled from the initial state distribution $\rho$ every $H$ timesteps, and 2) the policy---which determines which state-action pairs can be reached in an episode---is also resampled every $H$ timesteps.}: once this state-action pair is no longer visited, the reward model posterior is no longer updated with respect to it. Then, Algorithm \ref{alg:main_algo_theoretical}~is guaranteed to eventually sample a high enough reward for this state-action that the resultant policy will prioritize visiting it.   

First we note that Algorithm \ref{alg:main_algo_theoretical} is guaranteed to reach at least one state-action pair infinitely often: given our problem's finite state and action spaces, at least one state-action pair must be visited infinitely-often during execution of Algorithm \ref{alg:main_algo_theoretical}. If all state-actions are \textit{not} visited infinitely-often, there must exist a state-action pair $(s, a)$ such that $s$ is visited infinitely-often, while $(s, a)$ is not. Otherwise, if all actions are selected infinitely-often in all infinitely-visited states, the finitely-visited states are unreachable (in which case these states are irrelevant to the learning process and simple regret minimization, and can be ignored). Without loss of generality, we label this state-action pair $(s, a)$ as $\tilde{s}_1$. To reach a contradiction, it suffices to show that $\tilde{s}_1$ is visited infinitely-often.

Let $\bm{r}_1$ be the reward vector with a reward of $1$ in state-action pair $\tilde{s}_1$ and rewards of zero elsewhere. Let $\pi_{pi}(\tilde{\eta}, \bm{r}_1)$ be the policy that maximizes the expected number of visits to $\tilde{s}_1$ under dynamics $\tilde{\eta}$ and reward vector $\bm{r}_1$:
\begin{equation*}
\resizebox{0.85\linewidth}{!}{$
\begin{aligned}
        \pi_{pi}(\tilde{\eta}, \bm{r}_1) = \text{argmax}_\pi V(\tilde{\eta}, \bm{r}_1, \pi), \quad \text{where}, \quad V(\tilde{\eta}, \bm{r}, \pi) = \E{s_{1} \sim \rho} \left[ \sum_{t = 1}^H \overline{r}(s_t, \pi(s_t, t)) \,\Big|\, s_{t+1} \sim  \Pbb_{\tilde{\eta}}, \bm{\overline{r}} = \bm{r} \right].
\end{aligned}
$}
\end{equation*}
where $V(\tilde{\eta}, \bm{r}_1, \pi)$ is the expected total reward of a length-$H$ trajectory under $\tilde{\eta}, \bm{r}_1$, and $\pi$, or equivalently (by definition of $\bm{r}_1$), the expected number of visits to state-action $\tilde{s}_1$.

We next show that there exists a $\kappa > 0$ such that $\Pr(\pi = \pi_{pi}(\tilde{\eta}, \bm{r}_1)) > \kappa$ for all possible values of $\tilde{\eta}$. That is, for any sampled parameters $\tilde{\eta}$, the probability of selecting policy $\pi_{pi}(\tilde{\eta}, \bm{r}_1)$ is uniformly lower-bounded, implying that Algorithm \ref{alg:main_algo_theoretical} ~must eventually select $\pi_{pi}(\tilde{\eta}, \bm{r}_1)$.

Let $\tilde{r}_j$ be the sampled reward associated with state-action pair $\tilde{s}_j$ in a particular Algorithm \ref{alg:main_algo_theoretical}~episode, for each state-action $j \in \{1, \ldots, d\}$, with $d=SA$. We show that conditioned on $\tilde{\eta}$, there exists $v > 0$ such that if $\tilde{r}_1$ exceeds $\text{max}\{v \tilde{r}_2, v \tilde{r}_3, \ldots, v \tilde{r}_d\}$, then policy iteration returns the policy $\pi_{pi}(\tilde{\eta}, \bm{r}_1)$, which is the policy maximizing the expected amount of time spent in $\tilde{s}_1$. This can be seen by setting $v := \frac{H}{\kappa_1}$, where $\kappa_1$ is the expected number of visits to $\tilde{s}_1$ under $\pi_{pi}(\tilde{\eta}, \bm{r}_1)$. Under this definition of $v$, the event $\left\{\tilde{r}_{1} \ge \text{max}\{v \tilde{r}_2, v \tilde{r}_3, \ldots, v \tilde{r}_d\}\right\}$ is equivalent to $\{\tilde{r}_{1}\kappa_1 \ge h \, \text{max}\{\tilde{r}_2, \tilde{r}_3, \ldots, \tilde{r}_d\}\}$; the latter inequality implies that given $\tilde{\eta}$ and $\bm{\tilde{r}}$, the expected reward accumulated solely in state-action $\tilde{s}_1$ exceeds the reward gained by repeatedly (during all $H$ time-steps) visiting the state-action pair in the set $\{\tilde{s}_2, \ldots, \tilde{s}_d\}$ having the highest sampled reward. Clearly, in this situation, policy iteration results in the policy $\pi_{pi}(\tilde{\eta}, \bm{r}_1)$.

Next we show that $v = \frac{h}{\kappa_1}$ is continuous in the sampled dynamics $\tilde{\eta}$ by showing that $\kappa_1$ is continuous in $\tilde{\eta}$. Recall that $\kappa_1$ is defined as expected number of visits to $\tilde{s}_1$ under $\pi_{pi}(\tilde{\eta}, \bm{r}_1)$. This is equivalent to the expected reward for following $\pi_{pi}(\tilde{\eta}, \bm{r}_1)$ under dynamics $\tilde{\eta}$ and rewards $\bm{r}_1$:
\begin{flalign}\label{eqn:defn_p_1}
\kappa_1 = V(\tilde{\eta}, \bm{r}_1, \pi_{pi}(\tilde{\eta}, \bm{r}_1)) = \max_\pi V(\tilde{\eta}, \bm{r}_1, \pi).
\end{flalign}
The value of any policy $\pi$ is continuous in the transition dynamics parameters, and so $V(\tilde{\eta}, \bm{r}_1, \pi)$ is continuous in $\tilde{\eta}$. The maximum in \eqref{eqn:defn_p_1} is taken over the finite set of deterministic policies; because a maximum over a finite number of continuous functions is also continuous, $\kappa_1$ is continuous in $\tilde{\eta}$.

Next, recall that a continuous function on a compact set achieves its maximum and minimum values on that set. The set of all possible dynamics parameters $\tilde{\eta}$ is such that for each state-action pair $j$, $\sum_{k = 1}^S p_{jk} = 1$ and $p_{jk} \ge 0 \, \forall \, k$; the set of all possible vectors $\tilde{\eta}$ is clearly closed and bounded, and hence compact. Therefore, $v$ achieves its maximum and minimum values on this set, and so for any $\tilde{\eta}$, $v \in [v_{\text{min}}, v_{\text{max}}]$, where $v_{\text{min}} > 0$ ($v$ is nonnegative by definition, and $v = 0$ is impossible, as it would imply that $\tilde{s}_1$ is unreachable).

Then, $\Pr(\pi = \pi_{pi}(\tilde{\eta}, \bm{r}_1))$ can then be expressed in terms of $v$ and the parameters of the reward posterior. Firstly,
\small
\begin{align*}
\Pr(\pi = \pi_{pi}(\tilde{\eta}, \bm{r}_1)) \ge \Pr(\tilde{r}_1 > \text{max}\{v \tilde{r}_2, \ldots, v \tilde{r}_d\}) \ge \prod_{j = 2}^d \Pr(\tilde{r}_1 > v \tilde{r}_j) = \prod_{j = 2}^d [1 - \Pr(\tilde{r}_1 - v \tilde{r}_j \le 0)]
\end{align*}
\normalsize

The posterior updates for the reward model are given by Equation \eqref{eq:theta_informed_prior}, which is intractable to compute in closed form. Since we have $\theta_{0} \sim \Ncal(\mu_{0}, \Sigma_{0})$, we can use the result of Lemma 3 in Appendix of \cite{novoseller2020dueling}. The remaining proof thereby follows, and as a result, there exists some $\kappa > 0$ such that $\Pr(\pi = \pi_{pi}(\tilde{\eta}, \bm{r}_1)) \ge \kappa > 0$.

In consequence, Algorithm \ref{alg:main_algo_theoretical}~is guaranteed to infinitely-often sample pairs $(\tilde{\eta}, \pi)$ such that $\pi = \pi_{pi}(\tilde{\eta}, \bm{r}_1)$. As a result, Algorithm \ref{alg:main_algo_theoretical}~infinitely-often samples policies that prioritize reaching $\tilde{s}_1$ as quickly as possible. Such a policy always takes action $a$ in state $s$. Furthermore, because $s$ is visited infinitely-often, either a) $p_0(s) > 0$ or b) the infinitely-visited state-action pairs include a path with a nonzero probability of reaching $s$. In case a), since the initial state distribution is fixed, the MDP will infinitely-often begin in state $s$ under the policy $\pi = \pi_{pi}(\tilde{\eta}, \bm{r}_1)$, and so $\tilde{s}_1$ will be visited infinitely-often. In case b), due to Lemma 1 in Appendix of \cite{novoseller2020dueling}, the transition dynamics parameters for state-actions along the path to $s$ converge to their true values (intuitively, the algorithm knows how to reach $s$). In episodes with the policy $\pi = \pi_{pi}(\tilde{\eta}, \bm{r}_1)$, Algorithm \ref{alg:main_algo_theoretical}~is thus guaranteed to reach $\tilde{s}_1$ infinitely-often. Since Algorithm \ref{alg:main_algo_theoretical}~selects $\pi_{pi}(\tilde{\eta}, \bm{r}_1)$ infinitely-often, it must reach $\tilde{s}_1$ infinitely-often. This presents a contradiction, and so every state-action pair must be visited infinitely-often as the number of episodes tend to infinity.  
\end{proof}

\begin{restatable}{lemma}{psplerrorregretappendix}
\label{lemma:pspl_error_regret_appendix}
For any confidence $\delta_{1} \in (0,\frac{1}{3})$, let $\delta_{2} \in (c,1)$ with $c \in (0,1)$, be the probability that any optimal policy estimate $\hat{\pi}^{\star}$ constructed from the offline preference dataset $\Dcal_{0}$ is $\varepsilon$-optimal with probability at least $(1-\delta_{2})$ i.e. $\Pr \left( \E{s \sim \rho} \left[ V_{\theta,\eta,0}^{\pi^{\star}}(s) -  V_{\theta,\eta,0}^{\hat{\pi}^{\star}}(s) \right] > \varepsilon \right) < \delta_{2}$. Then, the simple Bayesian regret of the learner $\Upsilon$ is upper bounded with probability of at least $1-3\delta_{1}$ by, 
    \begin{equation}
    \label{eq:pspl_prior_error_bound_appendix}
    \begin{aligned}
        \SR_{K}^{\Upsilon}(\pi^{\star}_{K+1}, \pi^{\star}) \leq \sqrt{ \frac{10 \delta_{2}  S^{2}AH^{3} \ln \left( \frac{2KSA}{\delta_{1}}  \right) + 3\delta_{2}SAH^{2}\varepsilon^{2}}{2K \left(1 + \ln \frac{SAH}{\delta_{1}} \right) - \ln \frac{SAH}{\delta_{1}}} }
    \end{aligned}
    \end{equation}
\end{restatable}

\begin{proof}
\label{proof:pspl_error_regret}

Let $\Theta := (\theta, \eta)$ denote the unknown true parameters of the MDP $\Mcal$, and let $\hat{\Theta}_{k}^{(i)} := \left(\hat{\theta}_{k}^{(i)}, \hat{\eta}_{k}^{(i)} \right)$ denote the sampled reward and transition parameters at episode $k$, which are used to compute policy $\pi_{k}^{(i)}$ for $i \in \{0,1\}$. Let $J_{\pi}^{\tilde{\Theta}} := \E{\tilde{\Theta}, \tau \sim \pi}[r(\tau)]$ denote the expected total reward for a trajectory sampled from policy $\pi$ under some environment $\tilde{\Theta} := \left(\tilde{\theta}, \tilde{\eta} \right)$. Then define for each $i \in \{0,1\}$,
    $$
        Z_{k}(i) := J_{\pi^{\star}}^{\Theta} - J_{{\pi}_{k}^{(i)}}^{\Theta} - \varepsilon \quad ; \quad \Tilde{Z}_{k}(i) := J_{{\pi}_{k}^{(i)}}^{\hat{\Theta}_{k}^{(i)}} - J_{{\pi}_{k}^{(i)}}^{\Theta} - \varepsilon \quad ; \quad   I_{k}(i) := \Ibf \left\{ J^{\Theta}_{\pi_{k}^{(i)}} \neq J^{\Theta}_{\pi^{\star}} - \varepsilon \right\}
   $$

    First, note that $Z_{k}(i) = Z_{k}(i) I_{k}(i)$ with probability 1. Then compute,
    \begin{align*}
        \E{\Dcal_{k}}[Z_{k}(i) - \Tilde{Z}_{k}(i) I_{k}(i) ] &= \E{\Dcal_{k}} \left[ \left(Z_{k}(i) - \Tilde{Z}_{k}(i) \right) I_{k}(i) \right] = \E{\Dcal_{k}} \left[ \left(J_{\pi^{\star}}^{\Theta} - J_{{\pi}_{k}^{(i)}}^{\hat{\Theta}_{k}^{(i)}} \right) I_{k}(i) \right] \\
        &= \E{\Dcal_{k}} \left[J_{\pi^{\star}}^{\Theta} \Ibf \left\{ J^{\Theta}_{\pi_{k}^{(i)}} \neq J^{\Theta}_{\pi^{\star}} - \varepsilon \right\}  \right] - \E{\Dcal_{k}} \left[J_{{\pi}_{k}^{(i)}}^{\hat{\Theta}_{k}^{(i)}} \Ibf \left\{ J^{\Theta}_{\pi_{k}^{(i)}} \neq J^{\Theta}_{\pi^{\star}} - \varepsilon \right\} \right]= 0,
    \end{align*}
    where the last equality is true since $\Theta$ and $\hat{\Theta}_{k}^{(i)}$ are independently identically distributed given $\mathcal{D}_k$. Therefore, we can write the simple Bayesian regret upper bounded by $ \Ebb[\Tilde{Z}_{K+1}(0) I_{K+1}(0)]$. By Cauchy-Schwartz inequality, we have
    \begin{align*}
    \Ebb[\Tilde{Z}_{K+1}(0) I_{K+1}(0)] \leq \sqrt{\left(\Ebb[I_{K+1}(0)^2] \right) \left(  \Ebb[\Tilde{Z}_{K+1}(0)^2]\right) } 
    \end{align*}

    Since $\pi^{\star} := \argmax_{\pi} J^{\Theta}_{\pi}$, using Lemma \ref{lemma:contractionPS} in conjunction with Appendix B of \cite{hao2023bridging}, the first part can be bounded by
    \begin{align*} 
        \Ebb[I_{K+1}(0)^2] &\leq \Pr \left( J^{\Theta}_{\pi_{K+1}^{(0)}} < J^{\Theta}_{\pi^{\star}} - \varepsilon \right) \leq \max_{i \in \{0,1\}} \Pr \left( J^{\Theta}_{\pi_{1}^{(i)}} < J^{\Theta}_{\pi^{\star}} - \varepsilon \right) \leq \delta_{2} . \label{eq:sumik}
    \end{align*}

    Let $\mathcal{T}_{\pi_h}^\Theta$ be the Bellman operator at time $h$ defined by $\mathcal{T}_{\pi_h}^\Theta V_{h+1}(s) := r_{\theta}(s_{h}, a_{h}) + \sum_{s'\in\mathcal{S}} V_{h+1}(s') \Pbb_{\eta}(s'|s_{h}, \pi(s_{h}))$ and $\mathcal{T}_{\pi_{H}}^\Theta V_{H+1}(s) = 0$ for all $s \in \Scal$. Using Equation (6) of \cite{osband2013more} (also see Lemma A.14 of \cite{zhang2024policy}), we have
    \begin{align*}
        \Tilde{Z}_{K+1}(0) = \E{\Theta, \hat{\Theta}_{K+1}^{(0)}} \left[\sum_{h=1}^{H} \left[ \mathcal{T}_{{\pi}_h^{K+1}}^{\hat{\Theta}_{K+1}^{(0)}} V_{h+1}^{\hat{\Theta}_{K+1}^{(0)}} (s_{K+1,h}) - \mathcal{T}_{{\pi}_h^{K+1}}^{\Theta} V_{h+1}^{\hat{\Theta}_{K+1}^{(0)}} (s_{K+1,h}) \right]  \right]
    \end{align*}
    
    Recall that the instantaneous reward satisfies $r_{\theta}(\cdot, \cdot)\in [0, 1]$. So we have with $a_{K+1,h} := \pi_{K+1}(s_{K+1,h})$,

    \begin{equation*}
    \resizebox{\linewidth}{!}{$
    \begin{aligned}
       \left| \mathcal{T}_{{\pi}_h^{K+1}}^{\hat{\Theta}_{K+1}^{(0)}} V_{h+1}^{\hat{\Theta}_{K+1}^{(0)}} (s_{K+1,h}) - \mathcal{T}_{{\pi}_h^{K+1}}^{\Theta} V_{h+1}^{\hat{\Theta}_{K+1}^{(0)}} (s_{K+1,h}) \right| \leq H \norm{{\Pbb}_{\hat{\eta}_{K+1}^{(0)}}(\cdot \given s_{K+1,h}, a_{K+1,h}) - \Pbb_{\eta}(\cdot \given s_{K+1,h}, a_{K+1,h})}{1}{1} 
    \end{aligned}
    $}
    \end{equation*}

    Therefore,
    \begin{align*}
        \Ebb[\Tilde{Z}_{K+1}(0)^2] &\leq H \Ebb\left[\sum_{h=1}^{H} \left[ \mathcal{T}_{{\pi}_h^{K+1}}^{\hat{\Theta}_{K+1}^{(0)}} V_{h+1}^{\hat{\Theta}_{K+1}^{(0)}} (s_{K+1,h}) - \mathcal{T}_{{\pi}_h^{K+1}}^{\Theta} V_{h+1}^{\hat{\Theta}_{K+1}^{(0)}} (s_{K+1,h}) \right]^2 \right] \tag{Cauchy-Schwartz} \\
        &\leq H^3 \Ebb\left[\sum_{h=1}^{H} \norm{{\Pbb}_{\hat{\eta}_{K+1}^{(0)}}(\cdot \given s_{K+1,h}, a_{K+1,h}) - \Pbb_{\eta}(\cdot \given s_{K+1,h}, a_{K+1,h})}{1}{2} \right] \\
        &\leq 5H^{3} \Ebb \left[ \sum_{h=1}^{H} \frac{4S \ln \left( \frac{2KSA}{\delta_{1}}  \right)}{\sum_{j<K+1} w_{j} (s_{K+1,h}, a_{K+1,h}) - 2H\ln \frac{SAH}{\delta_{1}}} \right]
    \end{align*}

, where the last line holds with probability $(1-2\delta_{1})$ from Lemma \ref{lemma:transition_est_error} and Section H.3 in \cite{zanette2019tighter}. Now it remains to bound $\sum_{j<K+1} w_{j} (s_{K+1,h}, a_{K+1,h})$, which is essentially the sum of probabilities of visiting the pair $(s_{K+1,h}, a_{K+1,h})$ before the $(K+1)^{\text{th}}$ episode. We know from Appendix G of \cite{zanette2019tighter} (also see Appendix B of \cite{jin2019learning}) that with probability at least $(1-\delta')$, for any state-action $(s,a)$ in any episode $e$, we have $ \frac{1}{4} w_{e}(s, a) \geq H \ln \frac{SAH}{\delta'} + H \; .$ Using this, we have with probability at least $(1-3\delta_{1})$,


\begin{align*}
    \Ebb[\Tilde{Z}_{K+1}(0)^2] &\leq 5H^{3} \Ebb \left[ \frac{4SH \ln \left( \frac{2KSA}{\delta_{1}} \right) + \varepsilon^{2}  }{4HK \left(1 + \ln \frac{SAH}{\delta_{1}} \right) - 2H\ln \frac{SAH}{\delta_{1}}} \right] \\    
    &\leq \frac{20S^{2}AH^{4} \ln \left( \frac{2KSA}{\delta_{1}}  \right) + 5SAH^{3}\varepsilon^{2}  }{4HK \left(1 + \ln \frac{SAH}{\delta_{1}} \right) - 2H\ln \frac{SAH}{\delta_{1}}}  \\
    &\leq \frac{10S^{2}AH^{3} \ln \left( \frac{2KSA}{\delta_{1}}  \right) + 3SAH^{2}\varepsilon^{2} }{2K \left(1 + \ln \frac{SAH}{\delta_{1}} \right) - \ln \frac{SAH}{\delta_{1}}} \; .
\end{align*}
 
Putting it all together, we have with probability at least $(1-3\delta_{1})$, 

$$
\SR_{K}^{\Upsilon}(\pi^{\star}_{K+1}, \pi^{\star}) \leq \sqrt{ \frac{10 \delta_{2}  S^{2}AH^{3} \ln \left( \frac{2KSA}{\delta_{1}}  \right) + 3\delta_{2}SAH^{2}\varepsilon^{2}}{2K \left(1 + \ln \frac{SAH}{\delta_{1}} \right) - \ln \frac{SAH}{\delta}} }
$$

\end{proof}

\begin{restatable}{theorem}{psplfinalregretboundappendix}
\label{lemma:pspl_final_simple_regretappendix}
For any confidence $\delta_{1} \in (0,\frac{1}{3})$ and offline preference dataset size $N>2$, the simple Bayesian regret of the learner $\Upsilon$ is upper bounded with probability of at least $1-3\delta_{1}$ by, 
    \begin{equation}
    \begin{aligned}
        \SR_{K}^{\Upsilon}(\pi^{\star}_{K+1}, \pi^{\star}) \leq \sqrt{ \frac{20 \delta_{2}  S^{2}AH^{3} \ln \left( \frac{2KSA}{\delta_{1}}  \right)}{2K \left(1 + \ln \frac{SAH}{\delta_{1}} \right) - \ln \frac{SAH}{\delta_{1}}} } \;\; ,\text{where}\, \\ \delta_{2} = 2 \exp \left( - N \left( 1 + \gamma_{\beta, \lambda, N} \right)^{2} \right) + \exp \left( - \frac{N}{4} (1 - \gamma_{\beta, \lambda, N})^{3} \right)
    \end{aligned}
    \end{equation}
\end{restatable}

\begin{proof}
\label{proof:pspl_final_regret}

Define an event $\Ecal_{n} = \left\{\bar{Y}_{n} \neq \argmax_{i \in \{0,1\}} g_{\beta, \vartheta}\left (\bar{\tau}_{n}^{(i)} \right) \right\}$, i.e. at the $n$-th index of the offline preference dataset, the rater preferred the suboptimal trajectory (wrt trajectory score $g(\cdot)$). Given the optimal trajectory parity at index $n$ as $\bar{Y}_{n}^{\star} = \argmax_{i \in \{0,1\}} g_{\beta, \vartheta}(\bar{\tau}_{n}^{(i)})$, we have,

\begin{equation}
\resizebox{0.85\linewidth}{!}{$
\begin{aligned}
    \Pr (\Ecal_{n} \given \beta, \vartheta) &= 1 - \Pr (\Ecal_{n}^{c} \given \beta, \vartheta) \\
    &= 1 - \frac{g_{\beta, \vartheta} \left(\bar{\tau}_{n}^{\bar{Y}_{n}^{\star}} \right)}{g_{\beta, \vartheta} \left(\bar{\tau}_{n}^{\bar{Y}_{n}^{\star}} \right) + g_{\beta, \vartheta}\left(\bar{\tau}_{n}^{(1-\bar{Y}_{n}^{\star})} \right) } \\ 
    &= 1 - \frac{1}{1 + \exp \left( \beta \left\langle \phi \left(\bar{\tau}_{n}^{(\bar{Y}_{n}^{\star})}\right) - \phi \left(\bar{\tau}_{n}^{(1-\bar{Y}_{n}^{\star})} \right) ,  - \vartheta \right\rangle \right)} \\
    &= 1 - \frac{1}{1 + \exp \left( \beta \left\langle \phi \left(\bar{\tau}_{n}^{(\bar{Y}_{n}^{\star})}\right) - \phi \left(\bar{\tau}_{n}^{(1-\bar{Y}_{n}^{\star})} \right) , \theta - \vartheta \right\rangle - \beta \left\langle \phi \left(\bar{\tau}_{n}^{(\bar{Y}_{n}^{\star})}\right) - \phi \left(\bar{\tau}_{n}^{(1-\bar{Y}_{n}^{\star})} \right) , \theta \right\rangle \right)} \\ 
    &\leq 1 - \frac{1}{1 + \exp \left( \beta \left\| \phi \left(\bar{\tau}_{n}^{(\bar{Y}_{n}^{\star})}\right) - \phi \left(\bar{\tau}_{n}^{(1-\bar{Y}_{n}^{\star})} \right) \right\|_{1} \|\theta-\vartheta\|_{\infty} - \beta \left\langle \phi \left(\bar{\tau}_{n}^{(\bar{Y}_{n}^{\star})}\right) - \phi \left(\bar{\tau}_{n}^{(1-\bar{Y}_{n}^{\star})} \right) , \theta \right\rangle \right)} \\
    &\leq 1 - \frac{1}{1 + \underbrace{\exp\left( \beta B \|\vartheta-\theta\|_{\infty} - \beta \left\langle \phi \left(\bar{\tau}_{n}^{(\bar{Y}_{n}^{\star})}\right) - \phi \left(\bar{\tau}_{n}^{(1-\bar{Y}_{n}^{\star})} \right) , \theta \right\rangle   \right)}_{\clubsuit}} \nonumber
\end{aligned}
$}
\end{equation}

, where the last two lines use H\"{o}lder's inequality, and bounded trajectory map assumption respectively. Since $\vartheta-\theta\sim \Ncal(0, \Ibf_d/\lambda^2)$, using the Dvoretzky–Kiefer–Wolfowitz inequality bound \citep{massart1990tight, vershynin2010introduction} implies 
\begin{equation*}
     \Pr \left(\|\vartheta-\theta\|_{\infty}\geq t\right)\leq 2 d^{1/2} \exp\left(-\frac{t^2\lambda^2}{2}\right)\,.
\end{equation*}
Set $t=\sqrt{2\ln(2d^{1/2}N)}/\lambda$ and define an event $\Ecal_{\theta, \vartheta}:=\{\|\vartheta-\theta\|_{\infty}\leq \sqrt{2\ln(2d^{1/2}N)}/\lambda\}$ such that $P(\Ecal_{\theta, \vartheta}^{c})\leq 1/N$. We apply Union Bound on the $\clubsuit$ term, and decompose the entire right hand side term as:

\begin{equation}
\label{eq:rater_pref_error}
\resizebox{0.55\linewidth}{!}{$
\begin{aligned}
     \Pr (\Ecal_{n} \given \beta, \vartheta) &\leq \frac{1}{1 + \exp \left( \beta B \sqrt{2\ln(2d^{1/2}N)}/\lambda + \beta \Delta_{\min} \right)} + \frac{1}{N} \\
     &\leq \exp\left(- \beta B \sqrt{2\ln(2d^{1/2}N)}/\lambda - \beta \Delta_{\min} \right) + \frac{1}{N} := \gamma_{\beta, \lambda, N} \; .
\end{aligned}
$}
\end{equation}

Now we need to provide conditions on the rater's competence, in terms of $(\lambda, \beta)$ to be a valid expert i.e. for $\gamma_{\beta, \lambda, N} \in (0,1)$. Let $k_{1} = \beta B, k_{2} = \frac{2 \ln \left(2 d^{1/2}\right)}{\lambda^{2}}, k_{3} = 2/\lambda^{2}$, and $k_{4} = \beta \Delta_{\min}$. We then have,

\begin{align*}
    \exp\left(- k_{1} \sqrt{k_{2} + k_{3} \ln N}  - k_{4} \right) + \frac{1}{N} &< 1 \\
    -k_{1} \sqrt{k_{2} +  k_{3} \ln N } - k_{4}< \ln N \\
    \left( \ln N \right)^{2} + (2k_{4} - k_{1}^{2}k_{3}) \ln N + (k_{4}^{2} - k_{1}^{2}k_{2}) > 0
\end{align*}

The above is a quadratic inequality that holds for all $N>1$ if $\beta$ is large enough i.e.

$$
\text{If} \; \beta > \frac{2\ln \left(2 d^{1/2} \right)}{\left| B\lambda^{2} - 2 \Delta_{\min}  \right|}, \text{then,} \; \gamma_{\beta, \lambda, N} \in (0,1) \, .
$$

Now, since we have Lemma \ref{lemma:all_states_actions_visited}, the remaining argument shows to find a separation of probability between two types of states and time index pairs under the rater's preference, parameterized by $\beta$ and $\lambda$, and the offline dataset, characterized by its size $N$: the ones that are probable under the optimal policy $\pi^{\star}$, and the ones that are not. We have two cases:

\textbf{Case I.} $p^{\pi^{\star}}_{h}(s) > 0$

In this case, we want the rater to prefer trajectories that are most likely to occur under the optimal policy $\pi^{\star}$. Given $p^{\pi^{\star}}_{h}(s) > 0$, we now lower bound the probability of the state $s$ having $c_{h}(s) \equiv \sum_{a \in \Acal} c_{h}(s,a) > 0$ i.e.

\begin{equation}
\label{eq:lower_bound_chs}
\begin{aligned}
     \Pr (c_{h}(s) > 0) &= \Pr \left( \sum_{a} w_{h}(s,a) > \sum_{a} l_{h}(s,a) \right) = 1 - \Pr \left( \sum_{a} w_{h}(s,a) \leq \sum_{a} l_{h}(s,a) \right) \\
     &= 1 - \sum_{t=0}^{H} \Pr \left( \sum_{a} w_{h}(s,a) \leq t \right) \cdot \Pr \left( \sum_{a} l_{h}(s,a) = t \right) \\
     &\geq 1 - \sum_{t=1}^{H+1} \Pr \left( \sum_{a} w_{h}(s,a) < t \right) \\
     &\geq 1 - \gamma_{\beta, \lambda, N} / (1 - \gamma_{\beta, \lambda, N})
\end{aligned}
\end{equation}

, where the last step uses Equation \eqref{eq:rater_pref_error}.

\textbf{Case II.} $p^{\pi^{\star}}_{h}(s) = 0$

In this case, we wish to upper bound the probability of the rater preferring trajectories (and hence, states) which are unlikely to occur under the optimal policy $\pi^{\star}$. If for some state $s$ we have $p^{\pi^{\star}}_{h}(s) = 0$ but also this state $s \in \bar{\tau}_{n}^{(1 - \bar{Y}_{n}^{\star})}$, this means the rater preferred the suboptimal trajectory. Similar to the proof above we conclude that,

\begin{equation}
\label{eq:upper_bound_chs}
\Pr (c_{h}(s) > 0) = \Pr \left(\bigcup_{n=1}^{N} \{ \Ecal_{n} \} \given \beta, \vartheta \right) \leq \sum_{n=1}^{N} \Pr_\theta(\Ecal_{n} \given \beta, \vartheta) \leq \gamma_{\beta, \lambda, N} / (1 - \gamma_{\beta, \lambda, N})
\end{equation}

The above argument shows that there's a probability gap between parity of states and time index pairs under the rater's preference model: the ones that are probable under the optimal policy $\pi^{\star}$ and the ones that are not i.e. we have shown that the states which are \emph{more} likely to be visited under $\pi^{\star}$ have a lower bound on the probability of being part of preferred trajectories, and also the states which are \emph{less} likely to be visited by $\pi^{\star}$ have an upper bound on the probability of being part of preferred trajectories by the rater. Using this decomposition, we will show that when the rater tends to an expert ($\beta \to \infty, \lambda \to \infty$) and size $N$ of the offline preference dataset $\Dcal_{0}$ is large, we can distinguish the two types of state and time index pairs through their net counts in $\Dcal_0$. This will allow us to construct an $\varepsilon$-optimal estimate of $\pi^{\star}$ with probability at least $(1-\delta_{2})$. Noticing the structure of Equation \eqref{eq:pspl_prior_error_bound}, we see that the upper bound is minimized when $\varepsilon \to 0$, i.e. we construct an \emph{optimal} estimate of the optimal policy from the offline preference dataset. We now upper bound the probability that the estimated optimal policy $\hat{\pi}^{\star}$ is \emph{not} the optimal policy.

Let $\hat{\pi}^{\star}$ be the optimal estimator of $\pi^{\star}$ constructed with probability at least $(1-\delta_{2})$. Based on the separability of states and time index pairs, we have four possible events for each $(s,h) \in \Scal \times [H]$ pair. For $\delta = (1-\gamma_{\beta, \lambda, N})/2$, we have,

\begin{enumerate}
    \item $ E_{1} := \{ p_{h}^{\pi^{\star}}(s) > 0$ and $c_h(s) < \delta N \}$;
    \item $E_{2} := \{p_{h}^{\pi^{\star}}(s) > 0$ and $c_h(s) \geq \delta N$, but $\pi_{h}^{\star}(s) = a_{h}^{\star}(s) \neq \argmax_{a} c_h(s, a) = \hat{\pi}_{h}^{\star}(s) \}$.
    \item $E_{3} := \{p_{h}^{\pi^{\star}}(s) = 0$ and $c_h(s) \geq \delta N \}$;
    \item $E_{4} := \{p_{h}^{\pi^{\star}}(s) = 0$ and $c_h(s) < \delta N\}$;
\end{enumerate}

Denoting the the event to occur with high probability as $\Tcal$ i.e. $\Tcal := \{ r_{\theta}(s_{h}, \pi_{h}^{\star}(s_{h})) - r_{\theta}(s_{h}, \hat{\pi}^{\star}_{h}(s_{h})) \leq 0 \}$.
If we can show that $\Pr(\Tcal \given E_{i}) > 1 - \delta_{2}$ for $i \in \{1,2,3,4\}$, then union bound implies that $\hat{\pi}^{\star}$ is not optimal with probability at most $\delta_{2}$. 


\begin{enumerate}
    \item \textbf{Under the event $E_{1}$.}  Let $b \sim \mathrm{Bin}(T, q)$ denote a binomial random variable with parameters $T \in \Nbb$ and $q \in [0, 1]$. Notice that the each $c_{h}(s)$ is the difference of two binomial random variables $b_{1} \sim \mathrm{Bin}(N, 1 - \gamma_{\beta, \lambda, N})$ and $b_{2} \sim \mathrm{Bin}(N, \gamma_{\beta, \lambda, N})$. This implies that $c_{h}(s) + N \sim \mathrm{Bin}(2N, 1 - \gamma_{\beta, \lambda, N})$. We then have,

    \begin{align*}
        \Pr(\Tcal^{c} \given E_{1}) &\leq \Pr(c_{h}(s) < \delta N)  
        \leq \Pr \left( \mathrm{Bin}(2N, 1 - \gamma_{\beta, \lambda, N}) < (1 + \delta) N \right) \\
        &\leq \exp \left( - 4N \left( \delta + \gamma_{\beta, \lambda, N} \right)^{2}  \right) \leq \exp \left( - N \left( 1 + \gamma_{\beta, \lambda, N} \right)^{2} \right)
    \end{align*}

    \item \textbf{Under the event $E_{2}$.} Given Equation \eqref{eq:lower_bound_chs}, we have that,

    \begin{align*}
        \Pr(\Tcal^{c} \given E_{2}) &\leq   \Pr \left( \argmax_{a \in U^{W}_{h}(s)} Q^{\pi^{\star}}_{\theta, \eta, h}(s, a) \neq \argmax_{a \in U^{W}_{h}(s)} Q^{\pi^{\star}}_{\theta, \eta, h}(s, a) \given E_{2} \right) \\
        &\leq   \Pr \left(c_{h} \left(s, \argmax_{a \in U^{W}_{h}(s)} Q^{\pi^{\star}}_{\theta, \eta, h}(s, a) \right) \leq c_{h}(s)/2 \given E_{2} \right) \\ 
        &\leq   \Pr \left( \mathrm{Bin}(c_{h}(s), 1 - \gamma_{\beta, \lambda, N}) \geq c_{h}(s)/2 \given E_{2} \right) \\
        &\leq   \left[ \exp \left( -2c_{h}(s) (1 - \gamma_{\beta, \lambda, N} - c_{h}(s)/2)^{2} \right) \right] \rvert_{c_{h}(s) = \delta N} \\
        &\leq   \exp \left( - \frac{N}{4} (1 - \gamma_{\beta, \lambda, N})^{3} \right) 
    \end{align*}

    \item \textbf{Under the event $E_{3}$.} Similar to event $E_{1}$, we have, 

    \begin{align*}
        \Pr(\Tcal^{c} \given E_{3}) &\leq \Pr \left( \mathrm{Bin}(2N, 1 - \gamma_{\beta, \lambda, N}) > (1+\delta) N \right) \\ 
        &\leq \exp \left( - 4N \left( \delta + \gamma_{\beta, \lambda, N} \right)^{2}  \right) \leq \exp \left( - N \left( 1 + \gamma_{\beta, \lambda, N} \right)^{2} \right)
    \end{align*}

    \item \textbf{Under the event $E_{4}$.} Under this event, notice that conditioned on $\theta$, we have 

    \begin{align*}
        r_{\theta}(s_{h}, \pi_{h}^{\star}(s_{h})) - r_{\theta}(s_{h}, \hat{\pi}^{\star}_{h}(s_{h})) &= \Ebb \left[ \E{\theta} \left[ r_{\theta}(s_{h}, \pi_{h}^{\star}(s_{h})) - r_{\theta}(s_{h}, \hat{\pi}^{\star}_{h}(s))  \right]  \right] \\
        &= \E{\mathring{a} \sim \Acal} \left[ \E{\theta} \left[ r_{\theta}(s_{h}, \mathring{a}) - r_{\theta}(s_{h}, \mathring{a})  \right]  \right] = 0
    \end{align*}

    This means that $\Pr(\Tcal \given E_{4})$ is a non-failure event that occurs with probability 1 i.e. $\Pr(\Tcal^{c} \given E_{4}) = 0$.

\end{enumerate}

Combining all of the above, we have for $N > 2$, 

\begin{align*}
\Pr(\Tcal) &\geq  1 - \left(2 \exp \left( - N \left( 1 + \gamma_{\beta, \lambda, N} \right)^{2} \right) + \exp \left( - \frac{N}{4} (1 - \gamma_{\beta, \lambda, N})^{3} \right) \right) \\
&\geq 1 - \delta_{2}
\end{align*}

Using Lemma \ref{lemma:pspl_error_regret}, the proof is complete.

\end{proof}

\subsection{Results for Bandits}
\label{sec:appendixbandits}

For the bandit setting we let the action set be $\Acal \subseteq \Rbb^{d} $ with number of arms be $A=|\Acal|$, online episodes (rounds) be $K$, and offline dataset $\Dcal_{0} = \left\{ \left( \bar{a}^{(0)}_{n}, \bar{a}^{(1)}_{n}, Y_{n} \right) \right\}_{n=1}^{N}$ of size $N$, where $\bar{a}^{(0)}_{n}, \bar{a}^{(1)}_{n} \in \Acal$. Also let $\mu_{\text{min}}(\cdot) \in (0,1)$ be the minimum action sampling distribution during construction of $\Dcal_0$.

Letting $\Delta_{\text{min}} = \min_{n \in [N]} \left| r_{\theta}(\bar{a}^{(0)}_{n}) - r_{\theta}(\bar{a}^{(1)}_{n}) \right|$, where $\theta$ is the underlying reward model of the environment with $r_{\theta}(\cdot) : \Rbb^{d} \to \Rbb$, and $\gamma_{\beta, \lambda, N}$ to be error upper bound of the rater's preference (similar to Lemma \ref{lemma:rater_error_upper_bound}), we have the following result for simple regret of the learner $\Upsilon$, where simple regret is defined as $\SR_{K}^{\Upsilon}(\pi_{K+1}^{\star}, \pi^{\star}) = r_{\theta}(a^{\star}) - r_{\theta}(a^{\star}_{K+1})$, where $\pi^{\star}$ is the optimal policy that picks the optimal action $a^{\star} = \argmax_{a \in \Acal} r_{\theta}(a)$, and $\pi_{K+1}^{\star}$ is the policy of the learner after $K$ online rounds. We shall use $A^{\star}$ and $a^{\star}$ interchangebly to refer to the optimal action.

Analogous to the winning and undecided subsets of Section \ref{sec:analysis}, we construct an \emph{information} subset of $\Acal$, call it $\Ucal_{\Dcal_{0}}$ such that $\Pr(a^{\star} \in \Ucal_{\Dcal_{0}}) \geq 1 - \epsilon$, where $\epsilon \in (0,1)$ is the \emph{error} probability. As an algorithmic choice, we let $\Ucal_{\Dcal_{0}}$ consist of actions that have been preferred to \emph{at least} once in the offline preference dataset and of actions that do not appear in $\Dcal_{0}$. Given this construction, we have the following lemma, which has similarities to Appendix A of \cite{agnihotri2024online}.

\begin{restatable}{theorem}{banditinformationset}
\label{th:bandit_information_set}
Given the construction of $\Ucal_{\Dcal_{0}}$ above, we have $\Pr(a^{\star} \in \Ucal_{\Dcal_{0}}) \geq 1 - f_{1}$, where
\begin{equation}
\begin{aligned}
    f_{1} &:=  \left(1 - \frac{1}{1 + \exp \left( \beta \big( \min(1, \Delta) + \alpha_{2} - \alpha_{1}^{\Delta} \big) \right)} \right)^{N} + (1 - \mu_{\text{min}})^{2N} + \frac{1}{K} \\ 
    & \text{and}, \quad \Delta := \ln(K\beta)/\beta \, , \, \alpha_{1}^{\Delta} := A \min(1,\Delta), \; \text{and} \; \alpha_{2} := \lambda^{-1} \sqrt{2\ln(2d^{1/2}K)}
\end{aligned}
\end{equation}
\end{restatable}

\begin{proof}

We construct $\UD$ as a \emph{set} of actions that have been preferred to \emph{at least} once in the offline dataset $\Dcal_{0}$ and of actions that do not appear in the $\Dcal_{0}$. Thus, $\UD$ contains at most $A$ actions.

Now, we consider the formulation below. Recall that $\bar{A}_{n}^{(0)}$ and $\bar{A}_{n}^{(1)}$ are i.i.d. sampled from the action set and each datapoint in the dataset $\Dcal_{0}^{i}$, conditioned on $\vartheta, \beta$, is independent of $\Dcal_{0}^{j}$ for $i \neq j$. Now,

\begin{equation}
\begin{aligned}
    P \left(A^{\star} \notin \UD \right)&\leq P(A^{\star} \; \text{has lost all comparisons in} \; \Dcal_{0}) + P(A^{\star} \; \text{is not present in} \; \Dcal_{0}) \\
    &\leq \Ebb \bigg[ \prod_{n=1}^{N} \frac{\exp\big( \beta \langle a_{n}, \vartheta \rangle \big)}{\exp\big( \beta \langle a_{n}, \vartheta \rangle \big) + \exp\big( \beta \langle A^{\star}, \vartheta \rangle \big)} + (1-\mu_{\text{min}})^{2N} \bigg] \\
    &\leq \Ebb \bigg[ \prod_{n=1}^{N} \bigg( 1 - \frac{\exp\big( \beta \langle A^{\star}, \vartheta \rangle \big)}{\exp\big( \beta \langle a_{n}, \vartheta \rangle \big) + \exp\big( \beta \langle A^{\star}, \vartheta \rangle \big)}  \bigg)\bigg]  + (1-\mu_{\text{min}})^{2N} \\
    &\leq \Ebb \bigg[ \prod_{n=1}^{N} \bigg( 1 - \underbrace{\frac{1}{1 + \exp\big( - \beta \langle A^{\star} - a_{n}, \vartheta \rangle \big)}}_{\clubsuit} \bigg) \bigg]  + (1-\mu_{\text{min}})^{2N} 
\label{eq:optimalnotinu}
\end{aligned}
\end{equation} 
, where $A^{\star}$ is a function of $\theta$ and thus a random variable as well. Looking closely at the term $\clubsuit$ above, it can be written as $P(Y_{n} = A^{\star} \given \vartheta)$. We now analyze this term.

\begin{align*}
    P(Y_{n} = A^{\star} \given \vartheta) &= \frac{1}{1 + \exp \big( - \beta \langle A^{\star} - a_{n}, \vartheta \rangle \big)} \\
    &=\big( 1 + \exp\left(\beta \langle A^{\star}-a_{n}, \theta-\vartheta\rangle-\beta \langle A^{\star}-a_{n}, \theta \rangle \right) \big)^{-1} \\
    &\geq \big( 1 + \exp\left(\beta \| A^{\star}-a_{n} \|_{1} \| \theta-\vartheta \|_{\infty} - \beta \langle A^{\star}-a_{n}, \theta \rangle \right) \big)^{-1} \tag{H\"{o}lder's inequality} \\
    &\geq \big( 1 + \exp\left(\beta \| \vartheta-\theta \|_{\infty} - \beta \langle A^{\star}-a_{n}, \theta \rangle \right) \big)^{-1} \tag{$\| A^{\star}-a_{n} \|_{1} \leq 1 \forAll a_{n} \in \Acal$}
\end{align*}

Since $\vartheta-\theta\sim N(0, \Ibf_d/\lambda^2)$, using the Dvoretzky–Kiefer–Wolfowitz inequality bound \citep{massart1990tight, vershynin2010introduction} implies 
\begin{equation*}
     P\left(\|\vartheta-\theta\|_{\infty}\geq t\right)\leq 2 d^{1/2} \exp\left(-\frac{t^2\lambda^2}{2}\right)\,.
\end{equation*}
Set $t=\sqrt{2\ln(2d^{1/2}K)}/\lambda$ and define an event $\Ecal_1:=\{\|\vartheta-\theta\|_{\infty}\leq \sqrt{2\ln(2d^{1/2}K)}/\lambda\}$ such that $P(\Ecal_1^c)\leq 1/K$. We decompose Equation \eqref{eq:optimalnotinu} using Union Bound as:

\begin{equation}
\label{eqn2}
\resizebox{.85\linewidth}{!}{$
\begin{aligned}
     P\left(A^{\star} \notin \UD \right)&\leq \mathbb E\left[ \prod_{n=1}^N\left(1-P\left(Y_n=A^{\star} \Given \theta, \vartheta \right) \right) \Ibb_{\Ecal_1} \right] + P(\Ecal_1^c) + (1-\mu_{\text{min}})^{2N}  \\
     &\leq \mathbb E\left[\prod_{n=1}^N\left(1-\left(1 + \exp\left(\frac{\beta \sqrt{2\ln(2d^{1/2}K)}}{\lambda}\right) \underbrace{\exp\left(-\beta \langle A^{\star}-a_{n}, \theta \rangle \right)}_{\blacktriangle} \right)^{-1} \right) \right]+ \frac{1}{K} + (1-\mu_{\text{min}})^{2N} \,.
\end{aligned}
$}
\end{equation}

Now, we define another event $\Ecal_{(n)}:=\{ \langle A^{\star}-a_{n}, \theta \rangle \leq \Delta \}$. Based on $\Ecal_{(n)}$ we analyze the $\blacktriangle$ term as follows.

\begin{align*}
\exp\left(-\beta \langle A^{\star}-a_{n}, \theta \rangle \right) &= \Ebb \big[ \exp\left(-\beta \langle A^{\star}-a_{n}, \theta \rangle \right) \Ibb_{\Ecal_{(n)}} \big] + \Ebb \big[ \exp\left(-\beta \langle A^{\star}-a_{n}, \theta \rangle \right) \Ibb_{\Ecal_{(n)}^{c}} \big] \\
&\leq \exp\left(0 \right) P(\Ecal_{(n)}) + \exp\left(-\beta \Delta \right) P(\Ecal_{(n)}^{c}) \\
&\leq P(\Ecal_{(n)}) + (1-P(\Ecal_{(n)}))\exp\left(-\beta \Delta \right)
\end{align*}

Plugging this back in Equation \eqref{eqn2} we get,

\begin{equation}
\resizebox{.98\linewidth}{!}{$
\begin{aligned}
 P \left( A^{\star} \notin \UD \right) &\leq \Ebb \left[\prod_{n=1}^N \left(1 - \left(1 + \exp\left(\frac{\beta \sqrt{2\ln(2d^{1/2}K)}}{\lambda}\right) \left( \Ibb_{\Ecal_{(n)}} + (1-\Ibb_{\Ecal_{(n)}}) \exp\left(-\beta \Delta \right) \right) \right)^{-1} \right) \right]+ \frac{1}{K} + (1-\mu_{\text{min}})^{2N} \\
\end{aligned}
$}
\label{eq:eqn3}
\end{equation}

Note that the random variable $\Ibb_{\Ecal_{(n)}}$ depends on the action sampling distribution $\mu$. Denote the probability of sampling this action $a_{n}$ by $\mu_{n}$, and as before we have $\mu$ supported by $[\mu_{\text{min}}, \mu_{\text{max}}]$. We first analyze this for any arbitrary $n \in [N]$ and study the the distribution of $\Ibb_{\Ecal_{(n)}}$ conditionaled on $A^{\star}$. Without loss of generality, we first condition on $A^{\star}= \ring a$ for some $\ring a \in \Acal$. For that, let $\rho(\cdot)$ be the univariate Gaussian distribution and $\theta_{a} = \langle a, \theta \rangle$ for any action $a$.

\begin{equation}
\resizebox{.9\linewidth}{!}{$
\begin{aligned}
    P\left(\Ibb_{\Ecal_{(n)}} = 1 \given A^{\star} = \ring a \right) &= P\left( \Ibb \left(\langle A^{\star} - a_{n}, \theta \rangle \leq \Delta \right) = 1 \Given A^{\star} = \ring a \right) & \\
    &= \frac{1}{P\left(A^{\star} = \ring a \right)}  P\left( \Ibb \left(\langle A^{\star} - a_{n}, \theta \rangle \leq \Delta \right) = 1 \, , \, A^{\star} = \ring a \right) & \\
    &= \frac{1}{ P \left(A^{\star}= \ring a \right)} P \left( \Ibb \left(\theta_{a_{n}} \geq \theta_{\ring a} - \Delta \right) = 1 \, , \, \bigcap_{a \in \Acal} \{\theta_{\ring a} \geq \theta_a \} \right) & \\
    &=\frac{1}{ P\left(A^{\star} = \ring a \right)} \bigintssss_{\Rbb} \left[\int_{\theta_{\ring a} - \Delta}^{\infty} d \rho(\theta) \right] d \rho(\theta_{\ring a}) & \\
    &=\frac{1}{ P\left(A^{\star} = \ring a \right)} \bigintsss_{\Rbb} \left[\int_{\theta_{\ring a} - \Delta}^{\theta_{\ring a}} d \rho(\theta) \right] d \rho(\theta_{\ring a}) \, & (\text{since} \, \theta_{a} \leq \theta_{A^{\star}} \forAll a)
\label{eqn:B_dis}
\end{aligned}
$}
\end{equation}

Noticing that the term inside the integral can be represented as a distribution, we first find a normalizing constant to represent the probabilities. So, define

$$
\Phi(\theta_{\ring a}) =\int_{-\infty}^{\theta_{\ring a}} (2\pi)^{-1/2} \exp(-x^2/2) \; d x \; \qquad ; \qquad g(\theta_{\ring a}) = \frac{1}{\Phi(\theta_{\ring a})} \int_{\theta_{\ring a} - \Delta}^{\theta_{\ring a}} d \rho(\theta) \; .
$$

For fixed $\theta_{\ring a}$, let $X_{\theta_{\ring a}} \sim \text{Bernoulli}(1, g(\theta_{\ring a}))$. With Eq. \eqref{eqn:B_dis} and letting $d \mu(\theta_{\ring a}) = \frac{\Phi(\theta_{\ring a})}{ P(A^{\star} = a_{\ring a})} d \rho(\theta_{\ring a})$ we have,

\begin{equation}
\label{eqn:B_prob}
\resizebox{.75\linewidth}{!}{$
     P\left(\Ibb_{\Ecal_{(n)}} = 1 \given A^{\star} = \ring a \right) = \int_{\Rbb} P(X_{\theta_{\ring a}}=1) \frac{\Phi(\theta_{\ring a})} { P(A^{\star} = \ring a)} d \rho(\theta_{\ring a})=\int_{\Rbb} P(X_{\theta_{\ring a}}=1) \; d \mu(\theta_{\ring a}) \; .
$}
\end{equation}

Plugging this back in Equation \eqref{eq:eqn3} and upper bounding the probabilities we get,

\begin{equation}
\resizebox{.98\linewidth}{!}{$
\begin{aligned}
 P \left( A^{\star} \notin \UD \right) &\leq \sum_{a \in \Acal} 
 \bigintssss_{\Rbb} P(X_{\theta_{a}}=0) d \mu(\theta_a) P\left(A^{\star} = a \right) \left(1 - \left(1 + \exp\left( \beta \left( \lambda^{-1} \sqrt{2\ln(2d^{1/2}K)} - \Delta \right) \right) \right)^{-1} \right)^{N} \cdot \mu_{\text{max}}^{N} \\ & \qquad + \frac{1}{K} + (1-\mu_{\text{min}})^{2N} \\
 &\leq  \sum_{a \in \Acal} \bigintssss_{\Rbb} P(X_{\theta_{a}}=0) \left(1 - \left(1 + \exp\left( \beta \left( \lambda^{-1} \sqrt{2\ln(2d^{1/2}K)} - \Delta \right) \right) \right)^{-1} \right)^{N} \cdot \mu_{\text{max}}^{N} \; d \mu(\theta_a) P\left(A^{\star} = a \right)  \\ & \qquad + \frac{1}{K} + (1-\mu_{\text{min}})^{2N} \\
 &\leq  \bigintssss_{\Rbb} \E{\ring a \, \in \, \Acal} \left[
 \left(1 - \left(1 + \exp\left( \beta \left( \lambda^{-1} \sqrt{2\ln(2d^{1/2}K)} - (1 - X_{\theta_{\ring a}})\Delta \right) \right) \right)^{-1} \right)^{N} \right] \cdot \mu_{\text{max}}^{N} d \mu(\theta_{\ring a})  + \frac{1}{K} + (1-\mu_{\text{min}})^{2N} \, ,
 \label{eq:eq02}
\end{aligned}
$}
\end{equation}

where $\mu_{\text{max}}$ is used to obtain the exponent $N$ by accounting for the sampling distribution $\mu$, and last step follows from the uniformity of each action being optimal. Finally, we need to find the supremum of $g(\theta_{\ring a})$ and hence Equation \eqref{eq:eq02}. Recall that,

$$
g(\theta_{\ring a}) = \frac{1}{\Phi(\theta_{\ring a})} \int_{\theta_{\ring a} - \Delta}^{\theta_{\ring a}} d \rho(\theta) = \frac{\int_{\theta_{\ring a} - \Delta}^{\theta_{\ring a}} d \rho(\theta)}{\int_{-\infty}^{\theta_{\ring a}} d \rho(\theta)} = \frac{\int_{-\infty}^{\theta_{\ring a}} d \rho(\theta) - \int_{-\infty}^{\theta_{\ring a} - \Delta} d \rho(\theta)}{\int_{-\infty}^{\theta_{\ring a}} d \rho(\theta)} = 1 - h_{\Delta}(\ring a) \;
$$

, where $h_{\Delta}(\ring a) := \frac{\int_{-\infty}^{\theta_{\ring a} - \Delta} d \rho(\theta)}{\int_{-\infty}^{\theta_{\ring a}} d \rho(\theta)}$. Setting $\nabla_{\ring a} h_{\Delta}(\ring a) = 0$ and analyzing $\nabla_{\ring a}^{2} h_{\Delta}(\ring a) > 0$, we find that 

\begin{equation}
    g(\theta_{\ring a}) \leq 1 - \Delta \exp \left( - \frac{(2\theta_{\ring a} - \Delta)\Delta}{2} \right) \leq \min(1, \Delta)  \; .
\end{equation}

Finally we, decompose Equation \eqref{eq:eq02} based on the event $\Ecal_{2} := \{ X_{\theta_{\ring a}} = 0 \}$, and upper bound the probability to simplify. Setting $\Delta = \ln(K\beta) / \beta$, we conclude with the following bound:

\begin{equation}
\resizebox{.98\linewidth}{!}{$
\begin{aligned}
 P \left( A^{\star} \notin \UD \right) &\leq  \left(1 - \left(1 + \exp\left( \beta \left( \lambda^{-1} \sqrt{2\ln(2d^{1/2}K)} - (A-1)\min(1, \ln(K\beta) / \beta) \right) \right) \right)^{-1} \right)^{N} + \frac{1}{K} + (1-\mu_{\text{min}})^{2N} 
 \label{eq:eq03}
\end{aligned}
$}
\end{equation}

\end{proof}

\begin{restatable}{theorem}{banditregretbound}
\label{th:bandit_final_simple_regret}
For any confidence $\delta_{1} \in (0,\frac{1}{3})$ and offline preference dataset size $N>2$, the simple Bayesian regret of the learner $\Upsilon$ is upper bounded with probability of at least $1-3\delta_{1}$ by, 
    \begin{equation}
    \resizebox{0.6\linewidth}{!}{$
    \begin{aligned}
        \SR_{K}^{\Upsilon}(\pi^{\star}_{K+1}, \pi^{\star}) & \leq \sqrt{ \frac{20 \delta_{2}  A \ln \left( \frac{2KA}{\delta_{1}}  \right)}{2K \left(1 + \ln \frac{A}{\delta_{1}} \right) - \ln \frac{A}{\delta_{1}}} } \quad \text{with} , \\
        \; \gamma_{\beta, \lambda, N} &= \exp\left(- \beta B \sqrt{2\ln(2d^{1/2}N)}/\lambda - \beta \Delta_{\min} \right)  + \frac{1}{N}  \;\; \text{, and} \\ 
        \; \delta_{2} &= 2 \exp \left( - N \left( 1 + \gamma_{\beta, \lambda, N} \right)^{2} \right) + \exp \left( - \frac{N}{4} (1 - \gamma_{\beta, \lambda, N})^{3} \right).
    \end{aligned}
    $}
    \end{equation}
\end{restatable}

For a fixed $N>2$, and large $A$, the simple regret bound is $\widetilde{\Ocal}\left( \sqrt{AK^{-1}} \right)$. Note that this bound converges to zero exponentially fast as $N \to \infty$ and as the rater tends to an expert (large $\beta, \lambda$). In addition, as the number of online episodes $K$ gets large, $\PSPL$ is able to identify the best policy (arm) with probability at least $(1-3\delta_{1})$.

\subsection{Constructing Surrogate Loss Function}
\label{appendix:surrogatelossfunction}

\begin{restatable}{lemma}{mapestimatelemmaappendix}
\label{th:mapestimatelemmaappendix} 
At episode $k$, the MAP estimate of $(\theta, \vartheta, \eta)$ can be constructed by solving the following equivalent optimization problem: 
\begin{equation}
\begin{aligned}
(\theta_{opt}, \vartheta_{opt}, \eta_{opt})  &= \underset{\theta, \vartheta, \eta}{\argmax} \; \Pr(\theta, \vartheta, \eta \, | \, \Dcal_{k}) \\ & \equiv \underset{\theta, \vartheta, \eta}{\argmin} \; \Lcal_{1}(\theta, \vartheta, \eta) +  \Lcal_{2}(\theta, \vartheta, \eta) +  \Lcal_{3}(\theta, \vartheta, \eta) \; , \text{where},  \\
 \Lcal_{1}(\theta, \vartheta, \eta)  &:= - \sum_{t=1}^{k-1} \vphantom{\int_1^2} \left[ \beta \langle {\tau}_t^{(Y_{t})} , \vartheta \rangle - \ln \bigg(e^{ \beta \langle {\tau}_t^{(0)}, \vartheta \rangle} + e^{\beta \langle {\tau}_t^{(1)}, \vartheta \rangle} \bigg) \right. \\ & \left. \qquad + \sum_{j=0}^{1} \sum_{h=1}^{H-1} \ln \Pr_{\eta}\left({s}_{t,h+1}^{(j)} \given {s}_{t,h}^{(j)}, {a}_{t,h}^{(j)} \right) \vphantom{\int_1^2} \right], \\
 \Lcal_{2}(\theta, \vartheta, \eta) &:= - \sum_{n=1}^{N} \left[ \beta \langle \bar{\tau}_n^{(\bar{Y}_{n})} , \vartheta \rangle - \ln \bigg(e^{ \beta \langle \bar{\tau}_n^{(0)}, \vartheta \rangle} + e^{\beta \langle \bar{\tau}_n^{(1)}, \vartheta \rangle} \bigg) \right],  \\
 \Lcal_{3}(\theta, \vartheta, \eta) &:= \frac{\lambda^2}{2} \norm{\theta - \vartheta}{2}{2} - SA \sum_{i=1}^{S} (\bm{\alpha}_{0,i} - 1) \ln \eta_{i} \\ & \qquad + \frac{1}{2} (\theta - \mu_{0})^{T} \Sigma_{0}^{-1} (\theta - \mu_{0}).
\end{aligned}
\label{eq:mapestimateproblemappendix}
\end{equation}
\end{restatable}
\begin{proof}
\label{proof:mapestimatelemma}

We first analyze the posterior distribution of $\vartheta, \theta, \eta$ given the dataset $\Dcal_{k}$ at the beginning of episode $k$, and then optimize it by treating these random variables as parameters.

\begin{equation}
\label{eq:map_surrogate}
\begin{aligned}
    \underset{\theta, \vartheta, \eta}{\argmax} \; \Pr(\theta, \vartheta, \eta \, | \, \Dcal_{k}) &= \underset{\theta, \vartheta, \eta}{\argmax} \; \Pr( \Dcal_{k} \, | \, \theta, \vartheta, \eta) \cdot \Pr(\theta, \vartheta, \eta) \\ 
    &= \underset{\theta, \vartheta, \eta}{\argmax} \; \ln \Pr(\Dcal_{k} \, | \, \theta, \vartheta, \eta) + \ln \Pr(\theta, \vartheta, \eta) \\ 
    &= \underset{\theta, \vartheta, \eta}{\argmax} \underbrace{\ln \Pr(\Hcal_{k} \,| \, \Dcal_{0}, \theta, \vartheta, \eta)}_{\Lcal_{1}} + \underbrace{\ln \Pr(\Dcal_{0} \, | \, \theta, \vartheta, \eta)}_{\Lcal_{2}} +  \underbrace{\ln \Pr(\theta, \vartheta, \eta)}_{\Lcal_{3}}
\end{aligned}
\end{equation}

Then,
\begin{equation}
\resizebox{0.82\linewidth}{!}{$
    \begin{aligned}
        \Lcal_{1} &= \sum_{t=1}^{k-1} \ln \Pr \bigg( \left({\tau}_t^{(0)}, {\tau}_t^{(1)}, Y_t\right) \, \big| \, \Dcal_{t}, \theta, \vartheta, \eta \bigg) \\
        &= \sum_{t=1}^{k-1} \ln \Pr \bigg (Y_t \, \big | \, {\tau}_t^{(0)}, {\tau}_t^{(1)}, \theta, \vartheta, \eta \bigg) + \ln \Pr \bigg( {\tau}_t^{(0)}, {\tau}_t^{(1)} \Given \Dcal_{t}, \theta, \vartheta, \eta \bigg) \\
        &= \sum_{t=1}^{k-1} \left[ \beta \langle {\tau}_t^{(Y_{t})} , \vartheta \rangle - \ln \bigg(e^{ \beta \langle {\tau}_t^{(0)}, \vartheta \rangle} + e^{\beta \langle {\tau}_t^{(1)}, \vartheta \rangle} \bigg) + \sum_{j=0}^{1} \sum_{h=1}^{H-1} \ln \Pr_{\eta} \left({s}_{t,h+1}^{(j)} \given {s}_{t,h}^{(j)}, {a}_{t,h}^{(j)} \right) \right] \\
        \Lcal_{2} &= \sum_{n=1}^{N} \ln \Pr \bigg( \left(\bar{\tau}_n^{(0)}, \bar{\tau}_n^{(1)}, \bar{Y}_n\right) \, \big| \, \theta, \vartheta, \eta \bigg) \\
        &= \sum_{n=1}^{N} \ln \Pr \bigg (\bar{Y}_n \, \big | \, \bar{\tau}_n^{(0)}, \bar{\tau}_n^{(1)}, \theta, \vartheta, \eta \bigg) + \underbrace{\ln \Pr \bigg( \bar{\tau}_n^{(0)}, \bar{\tau}_n^{(1)} \Given \theta, \vartheta, \eta \bigg)}_{\textcolor{gray}{\textbf{indep. of $\theta, \vartheta, \eta \implies$ constant}}} \\
        &= \sum_{n=1}^{N} \beta \langle \bar{\tau}_n^{(\bar{Y}_{n})} , \vartheta \rangle - \ln \bigg(e^{ \beta \langle \bar{\tau}_n^{(0)}, \vartheta \rangle} + e^{\beta \langle \bar{\tau}_n^{(1)}, \vartheta \rangle} \bigg) + \textcolor{gray}{\text{constant}} \\
        \Lcal_{3} &= \ln \Pr(\vartheta \, | \, \theta) + \ln \Pr(\theta) + \ln \Pr(\eta) \\
        &= \frac{d}{2} \ln \bigg(\frac{2\pi}{\lambda^2} \bigg) - \frac{\lambda^2}{2} \norm{\theta - \vartheta}{2}{2} - \frac{1}{2} \ln \big(|2\pi \Sigma_{0}| \big) - \frac{1}{2} (\theta - \mu_{0})^{T} \Sigma_{0}^{-1} (\theta - \mu_{0}) + SA \sum_{i=1}^{S} (\bm{\alpha}_{0,i} - 1) \ln \eta_{i}.
    \end{aligned}
$}
\end{equation}

Hence, final surrogate loss function is

\begin{equation}
\label{eq:surrogate_loss_function}
\resizebox{0.83\linewidth}{!}{$
\begin{aligned}
    \Lcal(\theta, \vartheta, \eta) &= \Lcal_{1}(\theta, \vartheta, \eta) +  \Lcal_{2}(\theta, \vartheta, \eta) +  \Lcal_{3}(\theta, \vartheta, \eta), \qquad \text{where} \\
    \Lcal_{1}(\theta, \vartheta, \eta) &= - \sum_{t=1}^{k-1} \left[ \beta \langle {\tau}_t^{(Y_{t})} , \vartheta \rangle - \ln \bigg(e^{ \beta \langle {\tau}_t^{(0)}, \vartheta \rangle} + e^{\beta \langle {\tau}_t^{(1)}, \vartheta \rangle} \bigg) + \sum_{j=0}^{1} \sum_{h=1}^{H-1} \ln \Pr_{\eta} \left({s}_{t,h+1}^{(j)} \given {s}_{t,h}^{(j)}, {a}_{t,h}^{(j)} \right) \right] \, , \\ 
    \Lcal_{2}(\theta, \vartheta, \eta) &= - \sum_{n=1}^{N} \left[ \beta \langle \bar{\tau}_n^{(\bar{Y}_{n})} , \vartheta \rangle - \ln \bigg(e^{ \beta \langle \bar{\tau}_n^{(0)}, \vartheta \rangle} + e^{\beta \langle \bar{\tau}_n^{(1)}, \vartheta \rangle} \bigg) \right] \, ,\\
    \Lcal_{3}(\theta, \vartheta, \eta) &= \frac{\lambda^2}{2} \norm{\theta - \vartheta}{2}{2} + \frac{1}{2} (\theta - \mu_{0})^{T} \Sigma_{0}^{-1} (\theta - \mu_{0}) - SA \sum_{i=1}^{S} (\bm{\alpha}_{0,i} - 1) \ln \eta_{i}.
\end{aligned}
$}
\end{equation}

Finally the problem becomes equivalent as follows:
\begin{equation}
\label{eq:surrogate_opt_problem}
(\theta_{opt}, \vartheta_{opt}, \eta_{opt}) = \underset{\theta, \vartheta, \eta}{\argmax} \; \Pr(\theta, \vartheta, \eta \, | \, \Dcal_{k}) \equiv \underset{\theta, \vartheta, \eta}{\argmin} \; \Lcal(\theta, \vartheta, \eta)
\end{equation}

\end{proof}

\subsection{Practical PSPL (cont.)}

\subsubsection{Estimating Rater Competence in Practice}
\label{sec:appendix_practical_pspl}

There are two main methods of estimating rater competence in practice:
\begin{enumerate}
    \item Based on maximum likelihood estimation (MLE). Similar idea has been proposed to estimate the expertise level in imitation learning \cite{beliaev2022imitation, beliaev2025inverse}.
    \item The second method is to simply look at the entropy of the empirical distribution of the action in the offline dataset. Suppose the empirical distribution of $\zeta$. Then we use $c/\mathcal{H}(\zeta)$ as an estimation for $\beta$, where $c > 0$ is a hyperparameter. The intuition is that for smaller $\beta$, the net state-action pair visit counts tend to be more uniform and thus the entropy will be larger. This is an unsupervised approach and agnostic to specific offline data generation processes. The knowledgeability $\lambda$ is not quite ‘estimable’ because for a single environment, even though we know the true environment $\theta$ and the expert’s knowledge $\vartheta$, we only have one pair of observations generated with the same $\vartheta$. Thus, the variance of the estimation for $\lambda$ could be infinite. However, exact estimation of $\lambda$ is not necessary as we show that the algorithm is robust to misspecified $\lambda$ through experiments in Section \ref{sec:experiments}.
\end{enumerate}

\subsubsection{Ablation Studies}
\label{sec:appendix_ablation}

The  Bootstrapped $\PSPL$ algorithm in Section \ref{sec:practical_approx_pspl} requires a knowledge of rater's parameters $\lambda, \beta$. We study the sensitivity of the algorithm's performance to mis-specification of these parameters. 

(i) \textbf{Different Preference Generation Expert Policy.} Though the learning agent assumes Equation \eqref{eq:pref_logistic_trajectories} as the expert's generative model, we consider it to use a deterministic greedy policy. Trajectories $\bar{\tau}_{n}^{(0)}$ and $\bar{\tau}_{n}^{(1)}$ are sampled, and then choose $\bar{Y}_{n} = \argmax_{i \in \{0,1\}} \beta \langle \bar{\tau}_{n}^{(i)} , \vartheta \rangle$, where $\vartheta\sim\Ncal(\theta, \Ibf_d / \lambda^2)$. In Figure \ref{fig:more_ablation}, we see that even when the learning agent's assumption of the rater policy is \textit{flawed}, $\PSPL$ significantly outperforms the baselines.
\\
(ii) \textbf{Misspecified Competence parameters.} First, we generate offline data with true $\lambda=10^{3}$ but $\PSPL$ uses a misspecified $\lambda$. Second, we generate offline data with true $\beta=10$ but $\PSPL$ uses a misspecified $\beta$. Figure \ref{fig:more_ablation} shows that although the performance of $\PSPL$ decreases as the degree of flawness increases, it still outperforms the baselines.      

\begin{figure*}
\vspace{-0.2cm}
\begin{minipage}{\textwidth}
\begin{tcolorbox}[width=.495\textwidth, nobeforeafter, coltitle = black, fonttitle=\fontfamily{lmss}\selectfont, title=RiverSwim, halign title=flush center, colback=backg_blue!5, colframe=darkgreen!15, boxrule=2pt, grow to left by=-0.5mm, left=-15pt, right=-15pt]
    \centering
    {
        \includegraphics[height=0.35\textwidth, width=0.4\textwidth]{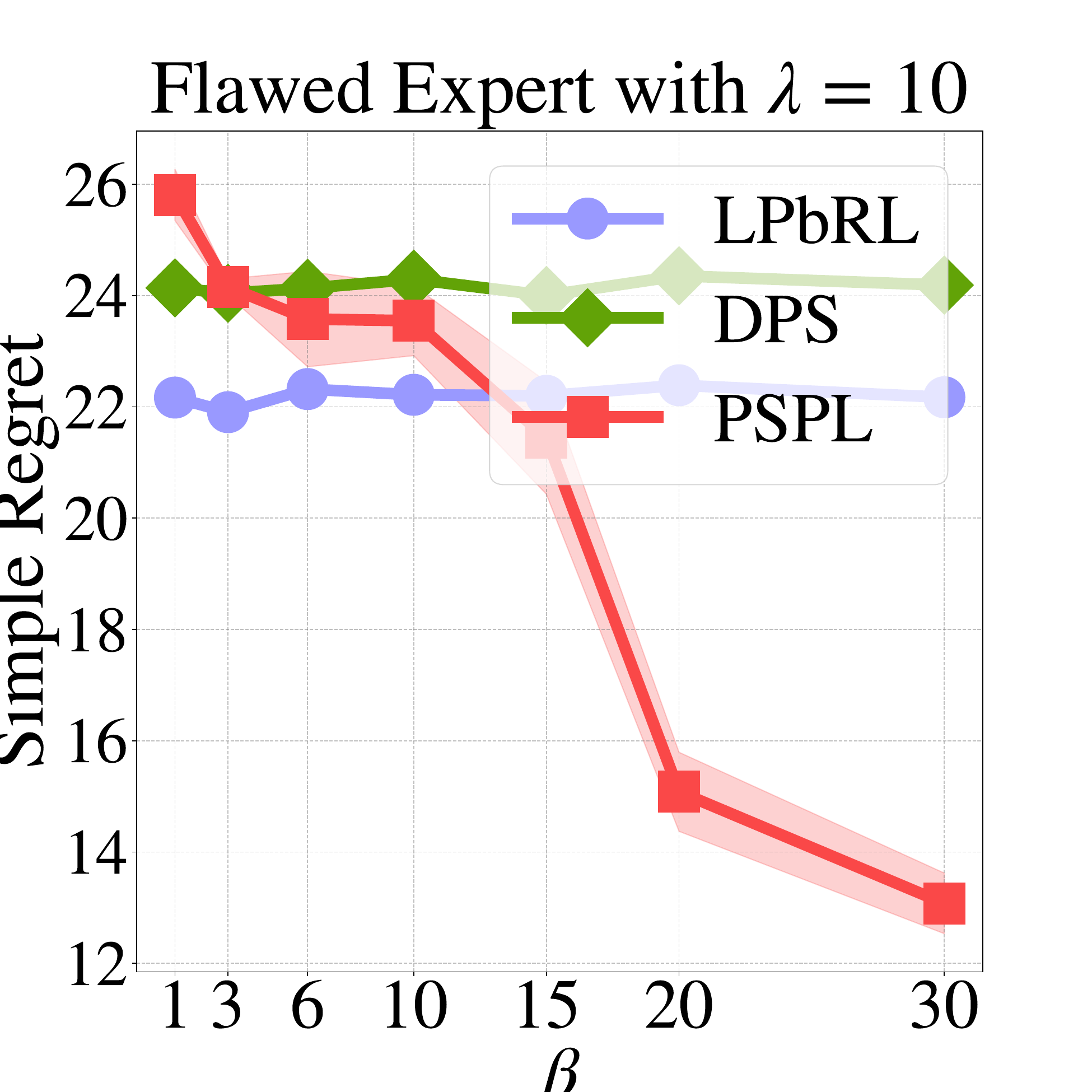}
    }
    {
        \includegraphics[height=0.35\textwidth, width=0.4\textwidth]{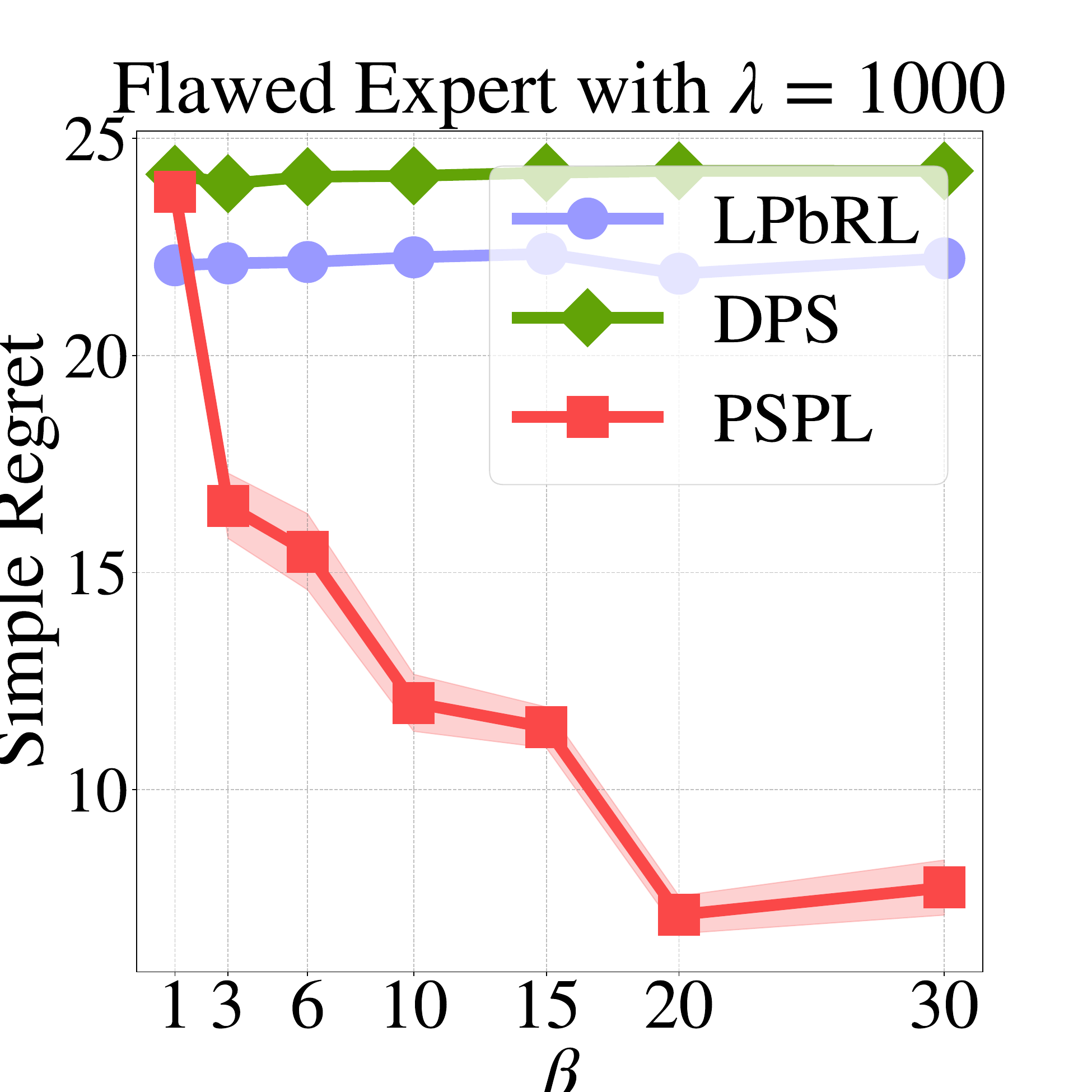}
    } 
\end{tcolorbox}  \hfill
\begin{tcolorbox}[width=.495\textwidth, nobeforeafter, coltitle = black, fonttitle=\fontfamily{lmss}\selectfont, title=RiverSwim, halign title=flush center, colback=backg_blue!5, colframe=purple!10, boxrule=2pt, grow to left by=-0.5mm, left=-15pt, right=-15pt]
    \centering
    {
        \includegraphics[height=0.35\textwidth, width=0.4\textwidth]{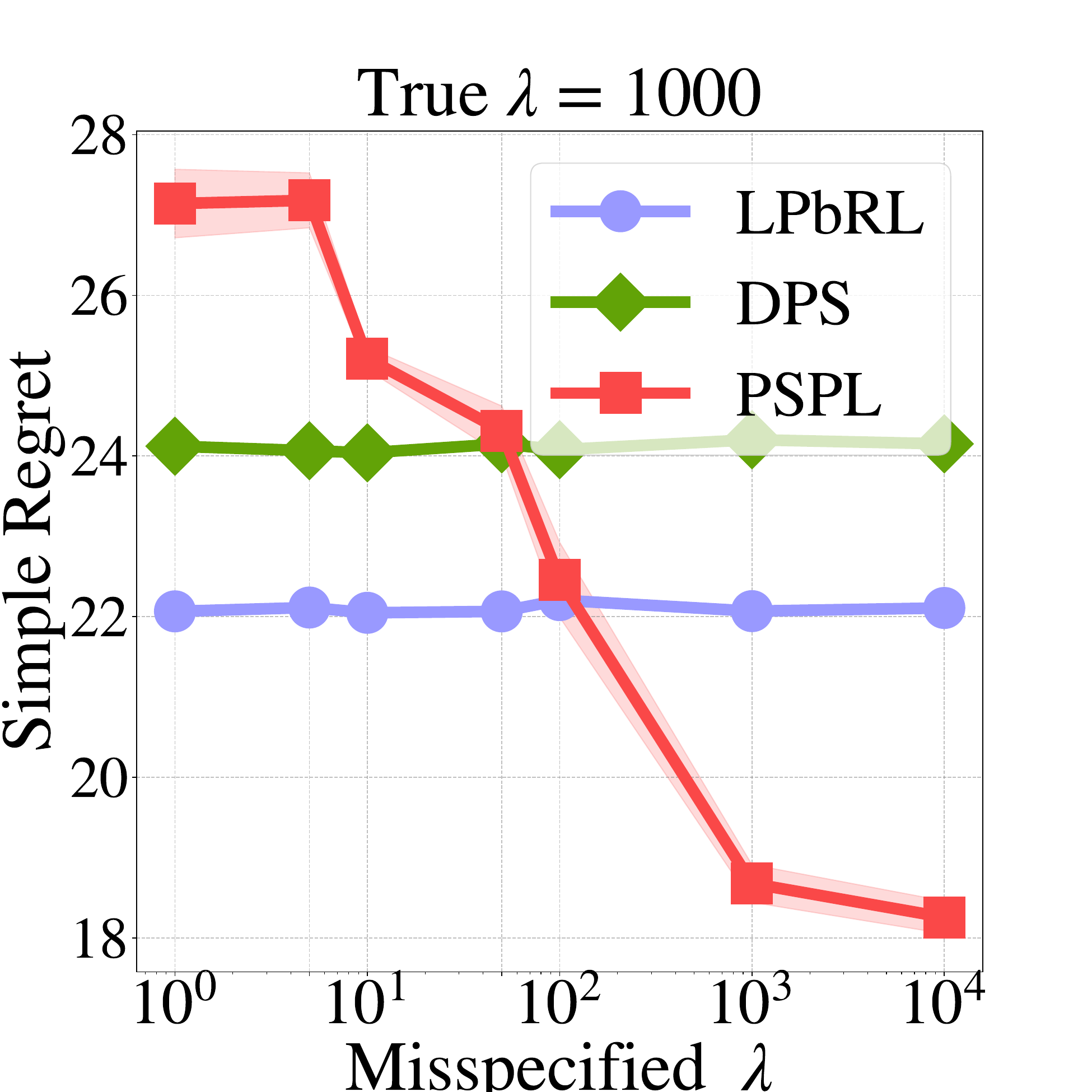}
    } 
    {
        \includegraphics[height=0.35\textwidth, width=0.4\textwidth]{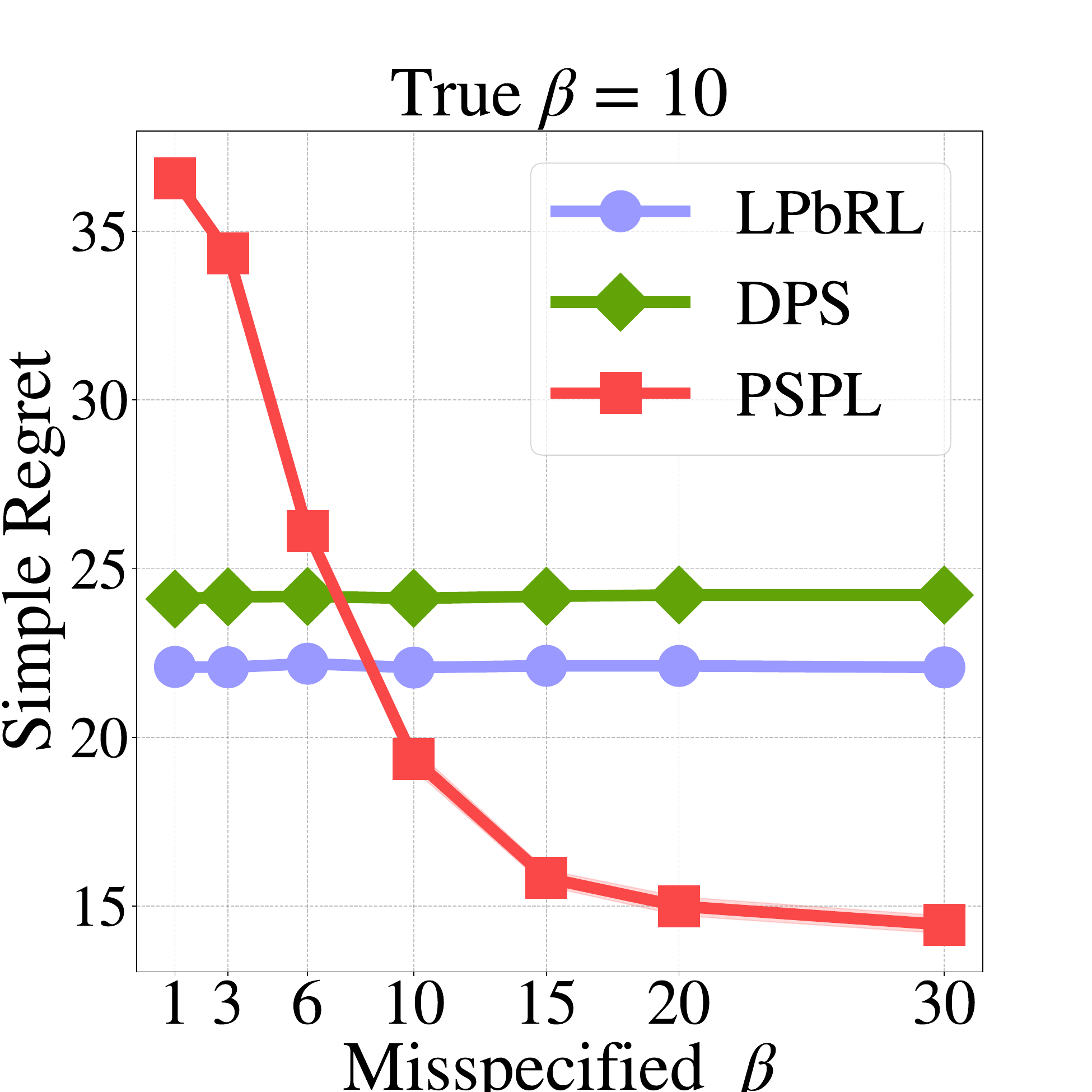}
    } 
\end{tcolorbox} 
\vspace{-0.7cm}
\caption{Sensitivity to flawed expert policy with $\lambda = \{10, 10^{3}\}$, and misspecified 
 competence.}
\vspace{-0.5cm}
\label{fig:more_ablation}
\end{minipage}
\end{figure*}

\subsubsection{Experiments on Image Generation Tasks (cont.)}
\label{appendix:pickapic}

We instantiate our framework on the Pick-a-Pic dataset of human preferences for text–to–image generation \cite{kirstain2023pick}. Overall, the dataset contains over 500,000 examples and 35,000 distinct prompts. Each example contains a prompt, sequence generations of two images, and a label for which image is preferred. We let each generation be a trajectory, so the dataset contains trajectory preferences $\mathcal D_{0} = { (\tau_{i}^{+},\tau_{i}^{-},y_{i}) }_{i=1}^{N}$ with $y_i = 1$ iff $\tau_{i}^{+} \succ \tau_{i}^{-}$. Each trajectory $\tau =(p,z_{0 : T})$ is the entire latent denoising chain of length $T$ for prompt $p$ sampled from some prompt distribution.

Following \cite{blacktraining, zhang2024flow}, text--to--image generation is a finite--horizon MDP $\mathcal M=(\mathcal S,\mathcal A,P)$: $s_t=(p,z_t),a_t=\epsilon_t,z_{t+1}=f_{\rho}(z_t,\epsilon_t),$ where $z_t \in \mathbb R^{d}$ is the latent, $a_t$ is the noise $\epsilon_{t}$ predicted by the policy $\pi_\theta(a_t \mid s_t)$, and $f_{\boldsymbol\alpha}$ is the deterministic DDPM transition with frozen scheduler $\rho$. Please see \cite{blacktraining} for a comprehensive and detailed discussion on modeling diffusion as a MDP; we follow the same approach. The episode horizon is $H=T=50$. For each trajectory, we adopt the additive embedding as discussed in the paper: $ \phi(\tau)=\sum_{t=1}^T \phi(s_t,a_t)$ with $\phi(s_t,a_t)=\bigl[ \text{CLIP}(p,\text{Dec}(z_T)), \lVert\epsilon_t\rVert_2 , t/T \bigr],$ where $\text{Dec}(\cdot)$ is the standard VAE decoder shipped with Stable‐Diffusion \cite{rombach2022high}. As in the paper, the model assumes that rater of competence $\lambda, \beta$ follows the Bradley-Terry model i.e. $\Pr(Y = 1\mid\tau^+,\tau^- , \vartheta,\beta) =\sigma \bigl(\beta\langle\vartheta,\phi(\tau^+) - \phi(\tau^-)\rangle\bigr)$, where $\sigma(\cdot)$ is the sigmoid link function.

To model the reward parameters, we let the uninformed prior be $\nu_0=\mathcal N(\mu_0,\Sigma_0)$, $\mu_0=0$, and diagonal $\Sigma_0= \mathrm{diag}(\sigma_{\text{clip}}^2,\sigma_{\text{noise}}^2, \sigma_{\text{time}}^2)$, where $\sigma_{\text{clip}}^2,\sigma_{\text{noise}}^2$ are empirical variances from $10$k random chains and $\sigma_{\text{time}}^2 =1/12$. Regarding modeling of transitions, the physical scheduler $f_{\rho}$ is known; uncertainty remains only in the $\ell_2$-norm of $\epsilon_t$. We discretize this norm into $C{=}10$ bins as $b_i$ for $i \in [C]$ and model $P(b_i\mid s_t) \sim \mathrm{Dir}(\boldsymbol\alpha_0), \;\; \boldsymbol\alpha_0=\mathbf 1_C.$ The resulting Dirichlet counts are updated from both $\mathcal{D}_0$ and online episodes.

Since we do not know the optimal generation sequence, minimizing simple regret is equivalent to maximizing expected reward in the final online episode. To evaluate this final generation, we conduct $L=10$ rollouts and compute the average reward with a weighted ensemble of automatic quality metrics, similar to \cite{clark2023directly, xu2023imagereward} i.e. $r_{\theta}(\cdot)= 0.7*\text{ImageReward--v2} + 0.3*\text{Aesthetic--LAION}$ where ImageReward--v2 is a reward model from \cite{xu2023imagereward}, and Aesthetic--LAION is a reward model hosted on HuggingFace \cite{laion_ev_2025}. For evaluation, we sweep $N_{\text{off}} = 50$k preference triplets for the prior, reserve $N_{\text{val}} = 15$k examples for evaluation, and create an exploration pool of $10$k unseen examples. For implementation, we use a 2 layer MLP, 512 GELU units, and outputs $(\mu_\theta)_{t}$ and $\log ((\sigma_\theta)_{t})$, where $(\mu_\theta)_{t}$ is the predicted noise vector the agent believes will best denoise $z_{t}$ and $(\sigma_\theta)_{t}$ controls exploratory perturbations around that prediction, enabling posterior sampling for PSPL. This is analogous to solving the unconstrained optimization Problem \eqref{eq:surrogate_perturbed_loss} in the function approximation setting. 

For the online phase, we use the above reward model $r_{\theta}(\cdot)$ for the BT preference model (see Equation \eqref{eq:pref_logistic_trajectories}), with an expert rater (i.e. $\lambda,\beta \to \infty$), similar to \cite{kirstain2023pick}. Since, we show that PSPL is robust to mis-specifications in rater competence (please see Appendix \ref{sec:appendix_ablation}), we use an expert rater for ease of comparison.

\textbf{Training details.} Following \cite{zhang2024flow}, we first cluster the prompts in the dataset to obtain a mapping from $\text{cluster}_{j} \to (\Dcal_{0})_{j}$, where $j \in [J]$ is the cluster index out of J clusters, and $(\Dcal_{0})_{j}$ is the dataset of trajectories and corresponding preference labels for prompts in prompt $\text{cluster}_{j}$ i.e. $(\Dcal_{0})_{j} = { (\tau_{j,i}^{+},\tau_{j,i}^{-},y_{j,i}) }_{i=1}^{N_{j}}$, where $\tau_{j,\cdot}^{+}$ and $\tau_{j,\cdot}^{-}$ are the \emph{winning} and \emph{losing} trajectory generations given a prompt from $\text{cluster}_{j}$ for all $j \in [J]$. For tractability, we compress the training images to be 128$\times$128 pixels, and optimize for $K=100k$ episodes for each cluster $j \in [J]$. Future direction of this work will consider incorporating prompt information as a prior to the MLP, resulting in prompt conditioned inference. However, that is beyond the current scope of the paper. Finally, all experiments are run on NVIDIA GeForce RTX 5080, GPU 16GB, and Memory DDR5 64 GB. Training times for all algorithms are given in Table \ref{tab:training_times} and comprehensive validation results are shown in Figure \ref{fig:more_image_gen}.

\begin{table}[ht]
\label{tab:training_times}
\caption{Average training times of baselines over 5 independent runs.}
\centering
\begin{tabular}{ccc}
\toprule
                  & DPS             & PSPL            \\ 
                  \midrule 
Training Time (h) & $3.52 \pm 0.11$ & $3.73 \pm 0.09$ \\
\bottomrule
\end{tabular}
\end{table}

\begin{figure}[ht]
\centering
\begin{tcolorbox}[width=.85\textwidth, nobeforeafter, coltitle = black, fonttitle=\fontfamily{lmss}\selectfont, title= Image Generation Tasks, halign title=flush center, colback=backg_blue!5, colframe=Pink!10, boxrule=2pt, grow to left by=-0.2mm, left=-15pt, right=-15pt, bottom=-1.5pt, top=-1pt]
    \centering
    {
        \includegraphics[scale=0.4]{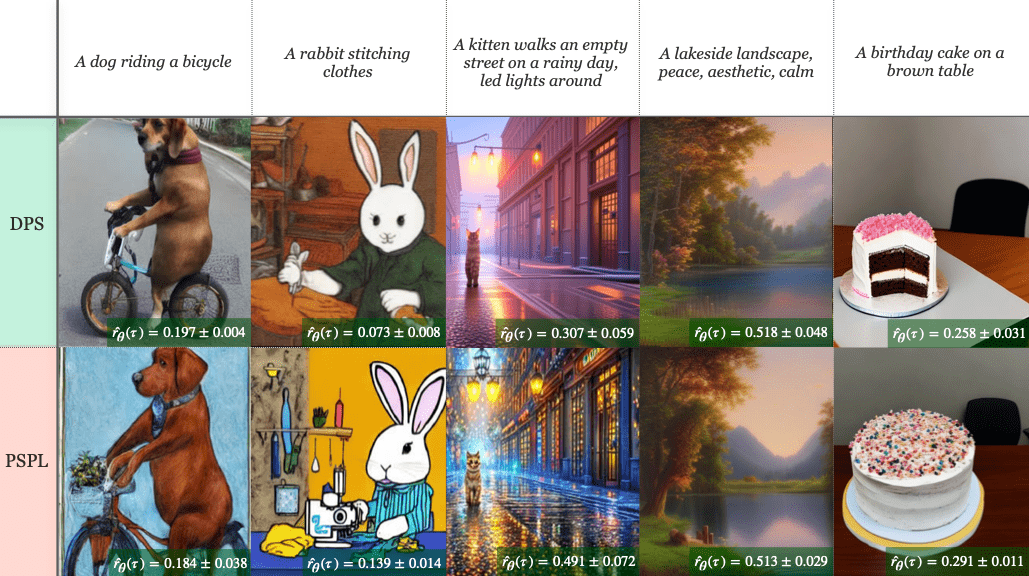}
    } 
\end{tcolorbox} 
\caption{\small Sample image generations with final image reward $\hat{r}_{\theta}(\cdot)$ over 5 independent runs. Images are enlarged for clarity.}
\label{fig:more_image_gen}
\end{figure}

\section{Auxiliary Definitions and Lemmas}

\begin{lemma}
\label{lem:binomial}
    Let $X$ be the sum of $L$ i.i.d. Bernoulli random variables with mean $p\in (0, 1)$. Let $q\in (0, 1)$, then 
    \begin{align*}
        \Pr(X \leq qL) &\leq \exp\left(-2L(q-p)^2\right),\qquad\text{if }q < p,\\
        \Pr(X \geq qL) &\leq \exp\left(-2L(q-p)^2\right),\qquad\text{if }q > p.
    \end{align*}
\end{lemma}

\begin{proof}
    Both inequalities can be obtained by applying Hoeffding's Inequality (see \ref{Hoeffding's inequality}).
\end{proof}

\begin{definition}
\textbf{ $\alpha$-dependence in \cite{russo2013eluder}}. For $\alpha>0$ and function class $\mathcal{Z}$ whose elements are with domain $\mathcal{X}$, an element $x \in \mathcal{X}$ is $\alpha$-dependent on the set $\mathcal{X}_n:=\left\{x_1, x_2, \cdots, x_n\right\} \subset \mathcal{X}$ with respect to $\mathcal{Z}$, if any pair of functions $z, z^{\prime} \in \mathcal{Z}$ with $\sqrt{\sum_{i=1}^n\left(z\left(x_i\right)-z^{\prime}\left(x_i\right)\right)^2} \leqslant$ $\alpha$ satisfies $z(x)-z^{\prime}(x) \leqslant \alpha$. Otherwise, $x$ is $\alpha$-independent on $\mathcal{X}_n$ if it does not satisfy the condition.
\end{definition}

\begin{definition}
\textbf{ Eluder dimension  in \cite{russo2013eluder}}. For $\alpha>0$ and function class $\mathcal{Z}$ whose elements are with domain $\mathcal{X}$, the Eluder dimension $\operatorname{dim}_E(\mathcal{Z}, \alpha)$, is defined as the length of the longest possible sequence of elements in $\mathcal{X}$ such that for some $\alpha^{\prime} \geqslant \alpha$, every element is $\alpha^{\prime}$ independent of its predecessors.
\end{definition}

\begin{definition}
    \textbf{Covering number.} Given two functions $l$ and $u$, the bracket $[l, u]$ is the set of all functions $f$ satisfying $l \leq f \leq u$. An $\alpha$-bracket is a bracket $[l, u]$ with $\|u-l\|<$ $\alpha$. The covering number $N_{[\cdot]}(\mathcal{F}, \alpha,\|\cdot\|)$ is the minimum number of $\alpha$-brackets needed to cover $\mathcal{F}$.
\end{definition}

\begin{lemma}
   \label{linear-dimension-covering}
   \textbf{(Linear Preference Models Eluder dimension and Covering number).} For the case of $d$-dimensional generalized trajectory linear feature models  $r_{\xi}\left(\xi_H\right):=\left\langle\phi\left(\xi_H\right), \mathbf{w}_r\right\rangle$, where $\phi:$ Traj  $\rightarrow \mathbb{R}^{dim_{\mathbb{T}}}$ is a known $dim_{\mathbb{T}}$ dimension feature map satisfying $\left\|\psi\left(\xi_H\right)\right\|_2 \leq B$ and $\theta \in \mathbb{R}^d$ is an unknown parameter with $\|\mathbf{w}_r\|_2 \leq \kappa_w$. Then the $\alpha$-Eluder dimension of $r_{\xi}(\xi_H)$ is at most $\mathcal{O}(dim_{\mathbb{T}} \log (B \kappa_w / \alpha))$. The $\alpha$ - covering number  is upper bounded by $\left(\frac{1+2 B\kappa_w}{\alpha}\right)^{dim_{\mathbb{T}}}$.
\end{lemma}

Let $\left(X_p, Y_p\right)_{p=1,2, \ldots}$ be a sequence of random elements, $X_p \in X$ for some measurable set $X$ and $Y_p \in \mathbb{R}$. Let $\mathcal{F}$ be a subset of the set of real-valued measurable functions with domain $X$. Let $\mathbb{F}=\left(\mathbb{F}_p\right)_{p=0,1, \cdots}$ be a filtration such that for all $p \geq 1,\left(X_1, Y_1, \cdots, X_{p-1}, Y_{p-1}, X_p\right)$ is $\mathbb{F}_{p-1}$ measurable and such that there exists some function $f_{\star} \in \mathcal{F}$ such that $\mathbb{E}\left[Y_p \mid \mathbb{F}_{p-1}\right]=f_*\left(X_p\right)$ holds for all $p \geq 1$. The (nonlinear) least square predictor given $\left(X_1, Y_1, \cdots, X_t, Y_t\right)$ is $\hat{f}_t=\operatorname{argmin}_{f \in \mathcal{F}} \sum_{p=1}^t\left(f\left(X_p\right)-Y_p\right)^2$. We say that $Z$ is conditionally $\kappa$-subgaussion given the $\sigma$-algebra $\mathbb{F}$ is for all $\lambda \in \mathbb{R}, \log \mathbb{E}[\exp (\lambda Z) \mid \mathbb{F}] \leq \frac{1}{2} \lambda^2 \kappa^2$. For $\alpha>0$, let $N_\alpha$ be the $\|\cdot\|_{\infty}$-covering number of $\mathcal{F}$ at scale $\alpha$. For $\beta > 0$, define
\begin{equation}
\mathcal{F}_t(\beta)=\left\{f \in \mathcal{F}: \sum_{p=1}^t\left(f\left(X_p\right)-\hat{f}_t\left(X_p\right)\right)^2 \leq \beta\right\} .
\end{equation}
\begin{lemma}
    \label{transition limit}
    (Theorem 5 of \cite{ayoub2020model}).
    Let $\mathbb{F}$ be the filtration defined above and assume that the functions in $\mathcal{F}$ are bounded by the positive constant $C>0$. Assume that for each $s \geq 1,\left(Y_p-f_*\left(X_p\right)\right)$ is conditionally $\sigma$-subgaussian given $\mathbb{F}_{p-1}$. Then, for any $\alpha>0$, with probability $1-\delta$, for all $t \geq 1, f_* \in \mathcal{F}_t\left(\beta_t(\delta, \alpha)\right)$, where
$$
\beta_t(\delta, \alpha)=8 \sigma^2 \log \left(2 N_\alpha / \delta\right)+4 t \alpha\left(C+\sqrt{\sigma^2 \log (4 t(t+1) / \delta)}\right) .
$$
\end{lemma}

\begin{lemma}

    (Lemma 5 of \cite{russo2013eluder}). Let $\mathcal{F} \in B_{\infty}(X, C)$ be a set of functions bounded by $C>0$, $\left(\mathcal{F}_t\right)_{t \geq 1}$ and $\left(x_t\right)_{t \geq 1}$ be sequences such that $\mathcal{F}_t \subset \mathcal{F}$ and $x_t \in \mathcal{X}$ hold for $t \geq 1$. Let $\left.\mathcal{F}\right|_{x_{1: t}}=\left\{\left(f\left(x_1\right), \ldots, f\left(x_t\right)\right): f \in \mathcal{F}\right\}\left(\subset \mathbb{R}^t\right)$ and for $S \subset \mathbb{R}^t$, let $\operatorname{diam}(S)=\sup _{u, v \in S}\|u-v\|_2$ be the diameter of $S$. Then, for any $T \geq 1$ and $\alpha>0$ it, holds that
$$
\sum_{t=1}^T \operatorname{diam}\left(\left.\mathcal{F}_t\right|_{x_t}\right) \leq \alpha+C(d \wedge T)+2 \delta_T \sqrt{d T},
$$
where $\delta_T=\max _{1 \leq t \leq T} \operatorname{diam}\left(\left.\mathcal{F}_t\right|_{x_{1: t}}\right)$ and $d=\operatorname{dim}_{\mathcal{E}}(\mathcal{F}, \alpha)$.
\label{upper bound}
\end{lemma}

\begin{lemma}
    If $\left(\beta_t \geq 0 \mid t \in \mathbb{N}\right)$ is a nondecreasing sequence and $\mathcal{F}_t:=\left\{f \in \mathcal{F}:\left\|f-\hat{f}_t^{L S}\right\|_{2, E_t} \leq \sqrt{\beta_t}\right\}$, where $\hat{f}_t^{L S} \in \arg \min _{f \in \mathcal{F}} L_{2, t}(f)$ and $L_{2, t}(f)=\sum_1^{t-1}\left(f\left(A_t\right)-R_t\right)^2$, then for all $T \in \mathbb{N}$ and $\epsilon>0$,
$$
\sum_{t=1}^T \mathbf{1}\left(w_{\mathcal{F}_t}\left(A_t\right)>\epsilon\right) \leq\left(\frac{4 \beta_T}{\epsilon^2}+1\right) \operatorname{dim}_E(\mathcal{F}, \epsilon)
$$
 where $w_{\mathcal{F}}(a):=\sup _{f \in \mathcal{F}} f(a)-\inf _{f \in \mathcal{F}} f(a)$ denotes confidence interval widths.
\end{lemma}

\begin{theorem}
\label{Hoeffding's inequality}
\textbf{Hoeffding's inequality}\citep{hoeffding1994probability}.
Let \(X_1, X_2, \ldots, X_n\) be independent random variables that are sub-Gaussian with parameter \( \sigma \). Define \( S_n = \sum_{i=1}^n X_i \). Then, for any \( t > 0 \), Hoeffding's inequality provides an upper bound on the tail probabilities of \( S_n \), which is given by:
\[
\Pr\left( |S_n - \mathbb{E}[S_n]| \geq t \right) \leq 2 \exp\left(-\frac{t^2}{2n\sigma^2}\right).
\]
This result emphasizes the robustness of the sum \( S_n \) against deviations from its expected value, particularly useful in applications requiring high confidence in estimations from independent sub-Gaussian observations.
\end{theorem}

\begin{lemma}
\label{visitation}
    (Lemma F.4. in \cite{dann2017unifying}) Let $\mathcal{F}_i$ for $i=1 \ldots$ be a filtration and $X_1, \ldots X_n$ be a sequence of Bernoulli random variables with $\mathbb{P}\left(X_i=1 \mid \mathcal{F}_{i-1}\right)=P_i$ with $P_i$ being $\mathcal{F}_{i-1}$-measurable and $X_i$ being $\mathcal{F}_i$ measurable. It holds that
$$
\mathbb{P}\left(\exists n: \sum_{t=1}^n X_t<\sum_{t=1}^n P_t / 2-W\right) \leq e^{-W}
$$
\end{lemma}